\documentclass[a4paper,USenglish,cleveref, autoref, thm-restate]{lipics-v2021}

\usepackage[T1]{fontenc}
\usepackage{enumerate}
\usepackage{graphicx}
\usepackage{subcaption}
\usepackage{hyperref}
\usepackage{color}

\usepackage{xcolor}
\definecolor{lightgray}{gray}{0.9}

\usepackage{algorithm}
\usepackage{algpseudocodex}
\usepackage{amsmath}
\usepackage{amssymb}
\usepackage{mathtools}

\usepackage{tikz}
\usetikzlibrary{arrows.meta, positioning, matrix}
\usepackage{lmodern}

\usepackage{url}

\usepackage{booktabs}
\usepackage{adjustbox}
\usepackage{siunitx}

\DeclareMathOperator*{\argmax}{\arg\!\max}

\newcommand{\defined}{\stackrel{\text{def}}{=}}

\title{Faster Verified Explanations for Neural Networks} 

\author{Alessandro {De Palma}\footnote{\hspace{-3.5pt}Work partly carried out at Inria $\&$ ENS $\mid$ PSL, France.}}{London School of Economics and Political Science, UK}{a.de-palma@lse.ac.uk}{https://orcid.org/0009-0002-7229-7723}{}

\author{Greta Dolcetti}{Ca' Foscari University of Venice, Italy}{greta.dolcetti@unive.it}{https://orcid.org/0000-0002-2983-9251}{}

\author{Caterina Urban}{Inria $\&$ ENS $\mid$ PSL, France}{caterina.urban@inria.fr}{https://orcid.org/0000-0002-8127-9642}{}

\authorrunning{A. De Palma and G. Dolcetti and C. Urban} 

\Copyright{Alessandro De Palma and Greta Dolcetti and Caterina Urban} 


\relatedversion{}
\relatedversiondetails{Full Version}{https://arxiv.org/abs/2512.00164}


\category{} 

\relatedversion{} 

\supplement{Our implementation is publicly available here:}
\supplementdetails[subcategory={Source Code}]{Software}{https://github.com/alessandrodepalma/favex} 

\funding{This work was partially supported by the \textsc{SAIF} project, funded by the ``France 2030'' government investment plan managed by the French National Research Agency, under the reference \textsc{ANR-23-PEIA-0006}, as well as the \textsc{ForML} project, funded by the French National Research Agency, under the reference \textsc{ANR-23-CE25-0009}.}

\acknowledgements{The authors are grateful to the CLEPS infrastructure from Inria Paris for providing resources and support.}

\nolinenumbers 

\EventEditors{Robbert Krebbers and Alexandra Silva}
\EventNoEds{2}
\EventLongTitle{40th European Conference on Object-Oriented Programming (ECOOP 2026)}
\EventShortTitle{ECOOP 2026}
\EventAcronym{ECOOP}
\EventYear{2026}
\EventDate{June 29--July 3, 2026}
\EventLocation{Brussels, Belgium}
\EventLogo{}
\SeriesVolume{372}
\ArticleNo{9}

\begin{document}

\maketitle              

\begin{abstract}

	Verified explanations are a 
	principled way to explain the decisions taken by neural networks, which are otherwise black-box in nature.
	However, these techniques face significant scalability challenges, as they require multiple calls to neural network verifiers, each of them with an exponential worst-case complexity.
	We present \textsc{FaVeX}, a novel algorithm to compute verified explanations. 
	\textsc{FaVeX} accelerates the computation by dynamically combining batch and sequential processing of input features, and by reusing information from previous queries, both when proving invariances with respect to certain input features, and when searching for feature assignments altering the prediction.
	Furthermore, we present a novel and hierarchical definition of verified explanations, termed verifier-optimal robust explanations, that explicitly factors the incompleteness of network verifiers within the explanation.
	Our comprehensive experimental evaluation demonstrates the superior scalability of both \textsc{FaVeX}, and of verifier-optimal robust explanations, which together can produce meaningful formal explanation on networks with hundreds of thousands of non-linear activations.
\end{abstract}

\section{Introduction}\label{sec:intro}

%

\begin{figure}[t]
    \centering

    \begin{subfigure}[b]{0.4\textwidth}
        \centering
        \includegraphics[width=0.46\textwidth]{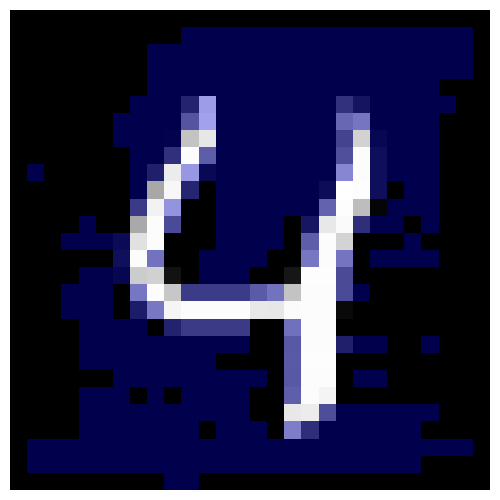}
        \caption{\centering Robust explanation.}
        \label{fig:sub1}
    \end{subfigure}
    \begin{subfigure}[b]{0.4\textwidth}
        \centering
        \includegraphics[width=0.46\textwidth]{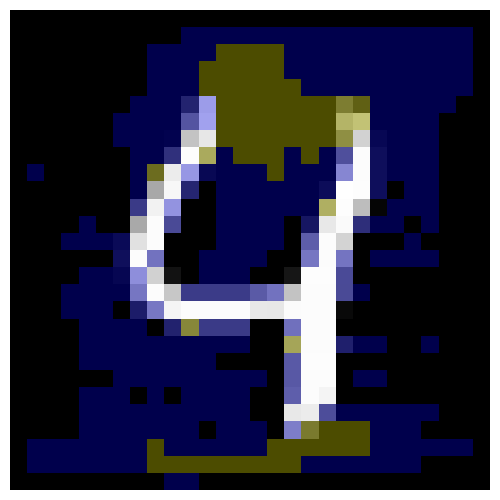}
        \caption{\centering Verifier-optimal robust explanation.}
        \label{fig:sub2}
    \end{subfigure}
    
    \caption{Standard~(a) and verifier-optimal~(b) robust explanations ($\epsilon = 0.25$) of the fifth image of the MNIST test set for a state-of-the-art network from the certified training literature~\cite{DePalma24}. \emph{Counterfactuals} are highlighted in \emph{yellow}, \emph{unknowns} in \emph{blue}, and \emph{invariants} are not highlighted. While the robust explanation is too large to be meaningful, the counterfactuals in~(b) point to an interpretable digit transformation. \label{fig:intro}}
\end{figure}

Machine learning models have demonstrated remarkable performance across a wide range of tasks, from image classification to natural language processing. Yet, as these models are increasingly deployed in safety-critical domains -- such as finance, transportation, and even healthcare -- concerns about trust, accountability, and robustness have become paramount.
In such contexts, explainability is essential: users and auditors must understand \emph{why} a model produces a particular prediction, and whether that prediction remains stable under small, semantically meaningful perturbations of the input~\cite{GuidottiMRTGP19, doshi2017towards}.

A promising line of research in eXplainable Artificial
Intelligence addresses this need through \emph{verified explanations}, e.g.,~\cite{MarquesSilva22b,MarquesSilva24}, among others, which aim to formally certify which input features are responsible for a model prediction by leveraging the theory and tools of machine learning verification~\cite{Albarghouthi21,Urban21}. Thus, verified explanations aim to provide rigorous guarantees, in contrast to heuristic or gradient-based methods, e.g.,~\cite{Hechtlinger16,Lundberg17,Ribeiro16}.
Among verified explanations, \emph{(optimal) robust explanations}~\cite{Ignatiev19b,LaMalfa21,Wu23,Huang24} have emerged as principled and mathematically grounded notions linking local robustness to the minimality of feature sets preserving a model prediction: they identify subsets of input features that, if left unchanged, cannot alter the predicted outcome. Formally, this corresponds to establishing local robustness of the model with respect to perturbations on the remaining input features. 

However, despite their theoretical appeal, optimal robust explanations are difficult to scale beyond small neural networks with tens of non-linear activations or low-dimensional input spaces~\cite{Wu23, Wu24}. Indeed, the computation of optimal robust explanations requires repeatedly querying neural network verifiers, which solve NP-hard problems~\cite{Katz17} and often exhibit their exponential worst-case complexity. In practice, even finding a single counterfactual -- a necessary step to ensure the \emph{optimality} of an explanation -- can quickly become infeasible for state-of-the-art neural networks with hundreds of thousands of non-linear activations (we demonstrate and discuss this in detail in Section~\ref{sec:exp-definitions}). 

Crucially, the definition of optimal robust explanations assumes that verification is complete, i.e., always either proving robustness or producing a counterfactual. This assumption is unrealistic in practice because, due to their exponential worst-case complexity, modern verifiers have to routinely enforce timeouts to handle large models efficiently, de facto making compromises about completeness to gain better performance~\cite{Bunel20, Katz17, TjengXT19}. As a consequence, although optimal explanations are a useful semantic ideal, they are rarely practically achievable for modern architectures. For instance, Figure~\ref{fig:sub1} shows the robust (but not optimal) explanation that can be computed in practice
for the classification of the fifth image in the MNIST testing data set by a state-of-the-art convolutional neural network (with roughly $230k$ ReLUs) from the certified training literature~\cite{DePalma24}. Notably, the explanation (i.e., the pixels that are highlighted in blue) covers a large portion of the image, making it too coarse to provide meaningful insights in practical settings. On the other hand, the counterfactuals (highlighted in yellow) directly point to a highly-interpretable digit transformation.


\subparagraph*{Contributions.}
We address the aforementioned challenges through three main contributions:
\begin{enumerate}
    \item \textbf{Verifier-Optimal Robust Explanations (Section~\ref{sec:v-optimal-x})}.
     We introduce a new and practically achievable counterpart of optimal robust explanations -- \emph{verifier-optimal robust explanations} -- that explicitly incorporates the behavior of a given verifier and its de facto incompleteness. Our notion partitions the input features into: (i) \emph{invariants}, for which the verifier proves robustness, (ii) \emph{counterfactuals}, for which the verifier identifies perturbations causing misclassification, and (iii) \emph{unknowns}, where verification remains inconclusive. This viewpoint better reflects what can be guaranteed in practice and enables the discovery of counterfactuals at a scale unattainable with prior definitions. 
     Figure~\ref{fig:sub2} shows the verifier-optimal robust explanation computed by our approach for the same MNIST image and convolutional model of Figure~\ref{fig:sub1}. In this case, the explanation is \emph{hierarchical}, with the most informative pixels (i.e., the counterfactuals, highlighted in yellow) covering only a small portion of the image.
     Importantly, the definition of verifier-optimal robust explanation is a generalization of optimal robust explanations:
     when the verifier is complete, the set of unknowns is empty and verifier-optimal robust explanations coincide exactly with optimal robust explanations (Lemma~\ref{lem}).     

    \item \textbf{\textsc{FaVeX}: A Fast Algorithm for Computing Verified Explanations (Section~\ref{sec:favex})}.
    We propose a new algorithm that substantially accelerates the computation of both standard and verifier-optimal explanations. \textsc{FaVeX} incorporates three key innovations:
    \begin{enumerate}[(a)]
    \item a hybrid query processing strategy that dynamically combines batch-based processing with binary search and sequential feature evaluation (Section~\ref{subsec:favex});
    \item branch reuse within branch-and-bound verification to leverage information from previous verification queries (Section~\ref{subsec:reuse}); and
    \item a restricted-space adversarial attack to accelerate counterfactual search (Section~\ref{subsec:attack}).
    \end{enumerate}
    We also introduce a traversal strategy aligned with the verifier's logit-difference objective (Algorithm \ref{alg:traversal}).

    \item \textbf{Implementation and Experimental Evaluation (Sections~\ref{sec:impl}~and~\ref{sec:evaluation})}.
    We implemented \textsc{FaVeX} using the \textsc{PyTorch} deep learning library, employing the \textsc{OVAL} branch-and-bound framework~\cite{Bunel2018,Bunel2020,DePalma21,sparsealgosDePalma2024} and its $\alpha$-$\beta$-CROWN implementation~\cite{Xu21,Wang21} as the backbone for our verifier. We evaluate it on both small networks similar to those employed in the verified explanation literature~\cite{Wu23,Wu24}, and on significantly larger state-of-the-art certifiably-robust neural networks.
The results demonstrate that:
\begin{enumerate}[(i)]
\item \textsc{FaVeX} significantly reduces the time required to compute standard robust explanations on small fully-connected networks;
\item verifier-optimal robust explanations can be computed efficiently even on larger convolutional networks; and
\item \textsc{FaVeX} finds counterfactuals even for networks with hundreds of thousands of 
activations, making formal explanations practical and meaningful at scale.
\end{enumerate}
\end{enumerate}

%

\section{Standard vs. Verifier-Optimal Robust Explanations}\label{sec:overview}

In this section, through a simple 
example\footnote{A self-contained script to reproduce this example is available in our source code repository at: \texttt{https://github.com/alessandrodepalma/favex/blob/main/ecoop2026.py}}, we provide an informal overview of the difference between our verifier-optimal robust explanations and standard robust explanations.

\begin{figure}[t]
\centering
\resizebox{200pt}{!}{%
\begin{tikzpicture}[x=4.3cm, y=0.6cm, >=stealth,
  neuron/.style={circle, draw, minimum size=0.45cm, inner sep=0pt, font=\tiny},
  edge/.style={-, gray, line width=0.1pt},
  lbl/.style={sloped, font=\fontsize{4.5pt}{4.5pt}\selectfont, fill=white, inner sep=0.5pt}]

\node[neuron] (x00) at (0,0)  {$x_{1}$};
\node[neuron] (x01) at (0,-1) {$x_{2}$};
\node[neuron] (x02) at (0,-2) {$x_{3}$};
\node[neuron] (x03) at (0,-3) {$x_{4}$};
\node[neuron] (x04) at (0,-4) {$x_{5}$};
\node[neuron] (x05) at (0,-5) {$x_{6}$};
\node[neuron] (x06) at (0,-6) {$x_{7}$};
\node[neuron] (x07) at (0,-7) {$x_{8}$};
\node[neuron] (x08) at (0,-8) {$x_{9}$};

\node[neuron] (x10) at (1,-1.75) {$x_{10}$};
\node[neuron] (x11) at (1,-3.25) {$x_{11}$};
\node[neuron] (x12) at (1,-4.75) {$x_{12}$};
\node[neuron] (x13) at (1,-6.25) {$x_{13}$};

\node[neuron] (x20) at (2,-3.4) {$x_{14}$};
\node[neuron] (x21) at (2,-4.6) {$x_{15}$};

\foreach \i/\w in {
  0/-3.6, 1/-0.4, 2/3.1, 3/8.0,
  4/-6.2, 5/-1.4, 6/0.8, 7/3.0, 8/-4.0%
}
  \draw[edge] (x0\i) -- node[lbl, pos=0.10] {\w} (x10);

\foreach \i/\w in {
  0/-6.7, 1/-8.2, 2/-3.8, 3/5.3,
  4/-2.3, 5/-4.1, 6/0.6, 7/-0.9, 8/0.9%
}
  \draw[edge] (x0\i) -- node[lbl, pos=0.15] {\w} (x11);

\foreach \i/\w in {
  0/3.9, 1/4.2, 2/2.6, 3/6.2,
  4/2.3, 5/9.1, 6/9.2, 7/3.4, 8/6.2%
}
  \draw[edge] (x0\i) -- node[lbl, pos=0.20] {\w} (x12);

\foreach \i/\w in {
  0/1.4, 1/-8.6, 2/-4.6, 3/-3.2,
  4/0.5, 5/-6.2, 6/-2.7, 7/-9.9, 8/-4.9%
}
  \draw[edge] (x0\i) -- node[lbl, pos=0.25] {\w} (x13);

\foreach \i/\w in {10/-11.0, 11/12.4, 12/-12.7, 13/12.9}
  \draw[edge] (x\i) -- node[lbl, pos=0.15] {\w} (x20);

\foreach \i/\w in {10/15.2, 11/-19.6, 12/11.1, 13/-13.5}
  \draw[edge] (x\i) -- node[lbl, pos=0.15] {\w} (x21);

\end{tikzpicture}
}
\caption{Fully connected feed-forward neural network classifier with ReLU activations, trained on the original Wisconsin Breast Cancer Database\footnotemark.}
\label{fig:bcw}
\end{figure}
\footnotetext{\url{https://archive.ics.uci.edu/dataset/15/breast+cancer+wisconsin+original}}

Let us consider the classifier with $9$-dimensional input in Figure~\ref{fig:bcw} and let $\mathbf{x} = (x_1, \dots, x_9) = (1.0, 0.7, 0.7, 0.2, 0.8, 0.4, 0.7, 0.3, 0.2)$ be the input vector for which we want to compute an explanation. For simplicity, we employ a basic deletion-based algorithm~\cite{Wu23,Huang24} that processes input features indexes in sequential order from $1$ to $9$. Each step introduces $\epsilon$-bounded $\ell_\infty$ perturbations (i.e., such that the maximum absolute value difference for any feature is at most $\epsilon$) to the input feature under analysis and selected previous features (depending on the explanation definition), while all remaining input features remain unchanged. Let us fix $\epsilon = 0.6$ for this example.
A verifier is then queried to determine whether this perturbation is robust (i.e., the model's prediction remains the same), a counterexample is found (i.e., applying the perturbation causes the model to change its prediction), or a timeout is reached (i.e., the query cannot be resolved within the time limit).

\subparagraph*{Standard Robust Explanations.}
%
\begin{figure}[b]
\centering
\begin{tikzpicture}[
    feat/.style={draw, minimum width=0.8cm, minimum height=0.2cm},
    irrel/.style={feat, fill=gray!20},
    analyzed/.style={feat, fill=gray!60},
    relevant/.style={feat, fill=yellow!70!orange}
]

\foreach \i in {1,...,9}
    \node[left] at (0,{-\i*0.35}) {$x_{\i}$};

\foreach \i in {1,...,9}{
    \node[irrel] at (1,{-\i*0.35}) {};
}
\node[circle, fill=black, inner sep=1pt] at (1,{-1*0.35}) {};

\foreach \i in {2,...,9}{
    \node[irrel] at (2,{-\i*0.35}) {};
}
\node[analyzed] at (2,{-1*0.35}) {};
\node at (2,{-1*0.35}) {\scriptsize \checkmark};
\node[circle, fill=black, inner sep=1pt] at (2,{-2*0.35}) {};

\foreach \i in {3,...,9}{
    \node[irrel] at (3,{-\i*0.35}) {};
}
\foreach \i in {1,2}{
    \node[analyzed] at (3,{-\i*0.35}) {};
    \node at (3,{-\i*0.35}) {\scriptsize \checkmark};
}
\node[circle, fill=black, inner sep=1pt] at (3,{-3*0.35}) {};

\foreach \i in {4,...,9}{
    \node[irrel] at (4,{-\i*0.35}) {};
}
\foreach \i in {1,2,3}{
    \node[analyzed] at (4,{-\i*0.35}) {};
    \node at (4,{-\i*0.35}) {\scriptsize \checkmark};
}
\node[circle, fill=black, inner sep=1pt] at (4,{-4*0.35}) {};

\foreach \i in {5,...,9}{
    \node[irrel] at (5,{-\i*0.35}) {};
}
\foreach \i in {1,2,3,4}{
    \node[analyzed] at (5,{-\i*0.35}) {};
    \node at (5,{-\i*0.35}) {\scriptsize \checkmark};
}
\node[circle, fill=black, inner sep=1pt] at (5,{-5*0.35}) {};

\foreach \i in {6,...,9}{
    \node[irrel] at (6,{-\i*0.35}) {};
}
\foreach \i in {1,2,3,4,5}{
    \node[analyzed] at (6,{-\i*0.35}) {};
    \node at (6,{-\i*0.35}) {\scriptsize \checkmark};
}
\node[circle, fill=black, inner sep=1pt] at (6,{-6*0.35}) {};

\foreach \i in {7,...,9}{
    \node[irrel] at (7,{-\i*0.35}) {};
}
\foreach \i in {1,2,3,4,5}{
    \node[analyzed] at (7,{-\i*0.35}) {};
    \node at (7,{-\i*0.35}) {\scriptsize \checkmark};
}
\foreach \i in {6}{
    \node[relevant] at (7,{-\i*0.35}) {};
    \node at (7,{-\i*0.35}) {\scriptsize $\times$};
}
\node[circle, fill=black, inner sep=1pt] at (7,{-7*0.35}) {};

\foreach \i in {8,9}{
    \node[irrel] at (8,{-\i*0.35}) {};
}
\foreach \i in {1,2,3,4,5}{
    \node[analyzed] at (8,{-\i*0.35}) {};
    \node at (8,{-\i*0.35}) {\scriptsize \checkmark};
}
\node[relevant] at (8,{-7*0.35}) {};
    \node at (8,{-7*0.35}) {\scriptsize ?};
\foreach \i in {6}{
    \node[relevant] at (8,{-\i*0.35}) {};
    \node at (8,{-\i*0.35}) {\scriptsize $\times$};
}
\node[circle, fill=black, inner sep=1pt] at (8,{-8*0.35}) {};

\node[irrel] at (9,{-9*0.35}) {};
\foreach \i in {1,2,3,4,5,8}{
    \node[analyzed] at (9,{-\i*0.35}) {};
    \node at (9,{-\i*0.35}) {\scriptsize \checkmark};
}
\node[relevant] at (9,{-7*0.35}) {};
\node at (9,{-7*0.35}) {\scriptsize ?};
\foreach \i in {6}{
    \node[relevant] at (9,{-\i*0.35}) {};
    \node at (9,{-\i*0.35}) {\scriptsize $\times$};
}
\node[circle, fill=black, inner sep=1pt] at (9,{-9*0.35}) {};

\foreach \i in {1,2,3,4,5,8}{
    \node[analyzed] at (10,{-\i*0.35}) {};
    \node at (10,{-\i*0.35}) {\scriptsize \checkmark};
}
\node[relevant] at (10,{-9*0.35}) {};
\node at (10,{-9*0.35}) {\scriptsize ?};
\node[relevant] at (10,{-7*0.35}) {};
\node at (10,{-7*0.35}) {\scriptsize ?};
\foreach \i in {6}{
    \node[relevant] at (10,{-\i*0.35}) {};
    \node at (10,{-\i*0.35}) {\scriptsize $\times$};
}
\end{tikzpicture}
\caption{Example of computing a standard robust explanation on 9 features. 
The verification steps proceed from left to right; the dotted box indicates the feature under analysis, dark grey shows the invariants (\tikz[baseline=-0.6ex]{\node[draw, fill=gray!60, minimum width=0.34cm, minimum height=0.18cm] {\scriptsize $\checkmark$};}). Yellow denotes the explanations, including both features for which a counterexample has been found (\tikz[baseline=-0.6ex]{\node[draw, fill=yellow!70!orange, minimum width=0.34cm, minimum height=0.18cm] {\scriptsize $\times$};}), and those for which a timeout has been encountered (\tikz[baseline=-0.6ex]{\node[draw, fill=yellow!70!orange, minimum width=0.34cm, minimum height=0.18cm] {\scriptsize ?};}). The final results of the analysis can be seen in the rightmost column.}
\label{fig:overview-standard}
\end{figure}
%
Figure~\ref{fig:overview-standard} shows how the analysis proceeds to compute a standard robust (but, in practice, not optimal) explanation using branch-and-bound verification with \textsc{CROWN} / \textsc{DeepPoly}~\cite{Zhang18,Singh19a} 
(cf. Section~\ref{sec:verification}).
First, the  verifier is queried with the $\ell_\infty$ perturbation applied only to $x_1$. In this case, robustness is verified: the feature index is added to the invariants set $\mathcal{R}_\mathbf{x}$ (initially empty). Next, the perturbation is applied to the features indexed by $\mathcal{R}_\mathbf{x}$ (currently just $x_1$) and to the new feature under analysis, $x_2$. Robustness is again verified:  this feature index is added to $\mathcal{R}_\mathbf{x}$. The analysis then proceeds with $x_3$, $x_4$, $x_5$, finding all of them to be invariants.
Next, for $x_6$, the verifier finds a counterexample. This feature is therefore considered part of the explanation and will not be perturbed in future steps. The analysis continues with $x_7$. 
When $x_7$ is perturbed (along with the invariants $x_1$, $x_2$, $x_3$, $x_4$, and $x_5$), the verifier times out. Since the definition of optimal robust explanation does not account for the incompleteness caused by timeouts, the feature is considered part of the explanation from this point on. The computation continues until all input features have been processed: another timeout is encountered for $x_9$, while $x_8$ is invariant. The explanation indiscriminately comprises features $x_6$, $x_7$, and $x_9$.
However, given the verifier timeouts, $x_7$ and $x_9$ may actually be invariants.

\subparagraph*{Verifier-Optimal Robust Explanations.}
%
\begin{figure}
\centering
\begin{tikzpicture}[
    feat/.style={draw, minimum width=0.8cm, minimum height=0.2cm},
    irrel/.style={feat, fill=gray!20},
    analyzed/.style={feat, fill=gray!60},
    relevant/.style={feat, fill=yellow!70!orange},
    blueish/.style={feat, fill=blue!40}
]

\foreach \i in {1,...,9}
    \node[left] at (0,{-\i*0.35}) {$x_{\i}$};

\foreach \i in {1,...,9}{
    \node[irrel] at (1,{-\i*0.35}) {};
}
\node[circle, fill=black, inner sep=1pt] at (1,{-1*0.35}) {};

\foreach \i in {2,...,9}{
    \node[irrel] at (2,{-\i*0.35}) {};
}
\node[analyzed] at (2,{-1*0.35}) {};
\node at (2,{-1*0.35}) {\scriptsize \checkmark};
\node[circle, fill=black, inner sep=1pt] at (2,{-2*0.35}) {};

\foreach \i in {3,...,9}{
    \node[irrel] at (3,{-\i*0.35}) {};
}
\foreach \i in {1,2}{
    \node[analyzed] at (3,{-\i*0.35}) {};
    \node at (3,{-\i*0.35}) {\scriptsize \checkmark};
}
\node[circle, fill=black, inner sep=1pt] at (3,{-3*0.35}) {};

\foreach \i in {4,...,9}{
    \node[irrel] at (4,{-\i*0.35}) {};
}
\foreach \i in {1,2,3}{
    \node[analyzed] at (4,{-\i*0.35}) {};
    \node at (4,{-\i*0.35}) {\scriptsize \checkmark};
}
\node[circle, fill=black, inner sep=1pt] at (4,{-4*0.35}) {};

\foreach \i in {5,...,9}{
    \node[irrel] at (5,{-\i*0.35}) {};
}
\foreach \i in {1,2,3,4}{
    \node[analyzed] at (5,{-\i*0.35}) {};
    \node at (5,{-\i*0.35}) {\scriptsize \checkmark};
}
\node[circle, fill=black, inner sep=1pt] at (5,{-5*0.35}) {};

\foreach \i in {6,...,9}{
    \node[irrel] at (6,{-\i*0.35}) {};
}
\foreach \i in {1,2,3,4,5}{
    \node[analyzed] at (6,{-\i*0.35}) {};
    \node at (6,{-\i*0.35}) {\scriptsize \checkmark};
}
\node[circle, fill=black, inner sep=1pt] at (6,{-6*0.35}) {};

\foreach \i in {7,...,9}{
    \node[irrel] at (7,{-\i*0.35}) {};
}
\foreach \i in {1,2,3,4,5}{
    \node[analyzed] at (7,{-\i*0.35}) {};
    \node at (7,{-\i*0.35}) {\scriptsize \checkmark};
}
\foreach \i in {6}{
    \node[relevant] at (7,{-\i*0.35}) {};
    \node at (7,{-\i*0.35}) {\scriptsize $\times$};
}
\node[circle, fill=black, inner sep=1pt] at (7,{-7*0.35}) {};

\foreach \i in {8,9}{
    \node[irrel] at (8,{-\i*0.35}) {};
}
\foreach \i in {1,2,3,4,5}{
    \node[analyzed] at (8,{-\i*0.35}) {};
    \node at (8,{-\i*0.35}) {\scriptsize \checkmark};
}
\node[blueish] at (8,{-7*0.35}) {};
\node at (8,{-7*0.35}) {\scriptsize U};
\foreach \i in {6}{
    \node[relevant] at (8,{-\i*0.35}) {};
    \node at (8,{-\i*0.35}) {\scriptsize $\times$};
}
\node[circle, fill=black, inner sep=1pt] at (8,{-8*0.35}) {};

\node[irrel] at (9,{-9*0.35}) {};
\foreach \i in {1,2,3,4,5}{
    \node[analyzed] at (9,{-\i*0.35}) {};
    \node at (9,{-\i*0.35}) {\scriptsize \checkmark};
}
\node[blueish] at (9,{-7*0.35}) {};
\node at (9,{-7*0.35}) {\scriptsize U};
\foreach \i in {6,8}{
    \node[relevant] at (9,{-\i*0.35}) {};
    \node at (9,{-\i*0.35}) {\scriptsize $\times$};
}
\node[circle, fill=black, inner sep=1pt] at (9,{-9*0.35}) {};

\foreach \i in {1,2,3,4,5}{
    \node[analyzed] at (10,{-\i*0.35}) {};
    \node at (10,{-\i*0.35}) {\scriptsize \checkmark};
}
\node[blueish] at (10,{-7*0.35}) {};
\node at (10,{-7*0.35}) {\scriptsize U};
\foreach \i in {6,8,9}{
    \node[relevant] at (10,{-\i*0.35}) {};
    \node at (10,{-\i*0.35}) {\scriptsize $\times$};
}
\end{tikzpicture}
\caption{Example of computing a verifier-optimal robust explanation on 9 features. The verification steps proceed from left to right; the dotted box indicates the feature under analysis, dark grey shows the invariants (\tikz[baseline=-0.6ex]{\node[draw, fill=gray!60, minimum width=0.34cm, minimum height=0.18cm] {\scriptsize $\checkmark$};}), blue shows the unknowns (\tikz[baseline=-0.6ex]{\node[draw, fill=blue!40, minimum width=0.34cm, minimum height=0.18cm] {\scriptsize U};}), yellow shows the explanations for which a counterexample has been found (\tikz[baseline=-0.6ex]{\node[draw, fill=yellow!70!orange, minimum width=0.34cm, minimum height=0.18cm] {\scriptsize $\times$};}). The final results of the analysis can be seen in the rightmost column.}
\label{fig:overview-favex}
\end{figure}
Figure~\ref{fig:overview-favex} instead illustrates the computation of a verifier-optimal robust explanation. In this case the analysis explicitly accounts for verifier limitations due to timeouts. Thus, rather than considering $x_7$ part of the explanation, the analysis adds the feature index to the unknowns set $\mathcal{U}_\mathbf{x}$ (initially empty). The analysis then proceeds by allowing {perturbations} also to the features indexed by $\mathcal{U}_\mathbf{x}$. 
This way, a counterexample is found for both $x_9$ and $x_8$ (which was instead invariant in Figure~\ref{fig:overview-standard}). This can occur because the considered perturbation region is effectively larger, as it also includes the features indexed by $\mathcal{U}_\mathbf{x}$. (This effect can also be observed in practice beyond this illustrative example: for instance, Figure~\ref{fig:sub2} has $4$ more invariant pixels than Figure~\ref{fig:sub1}.) The explanation, containing the features $x_6$, $x_7$, $x_8$, $x_9$, is larger than the one in Figure~\ref{fig:overview-standard} but it is \emph{hierarchical}: features $x_6$, $x_8$, $x_9$  (counterfactuals) are more important to the classification outcome, while $x_7$ (unknown) is less important.

%

\section{Background}\label{sec:background}

In the following, we will denote vectors by boldface lowercase letters ($\mathbf{x} \in \mathbb{R}^d$), and use subscripts to index their entries (for instance, $\mathbf{x}_i$ denotes the $i$-th entry). Furthermore, we will denote sets by calligraphic uppercase letters (e.g, $\mathcal{A}$), and use $\mathbf{x}_{\mathcal{A}}$ to denote the vector containing the entries of $\mathbf{x}$ indexed by the elements of $\mathcal{A}$.

\subsection{Neural Network Verification} \label{sec:verification}

\subparagraph*{Neural Networks.}
A \emph{neural network} classifier is a function $f : \mathbb{R}^{d} \rightarrow \mathbb{R}^{k}$ mapping a $d$-dimensional input vector of \emph{features} $\mathbf{x}$ to a $k$-dimensional output vector $f(\mathbf{x})$ of \emph{logits}, which determine the network's classification output $y_f  : \mathbb{R}^{k} \rightarrow \{1, \dots, k\}$, defined as $y_f (\mathbf{x}) \defined \argmax_i f (\mathbf{x})_i$.
We here focus on neural networks with rectified linear unit (ReLU) activations, which consist of a sequential composition $f = f_l \circ \dots \circ f_1$ of $l$ layers, where each layer $f_i\colon \mathbb{R}^m \rightarrow \mathbb{R}^n$, is either an affine transformation $f_i(\mathbf{x}) = \mathbf{A}\mathbf{x} + \mathbf{b}$, with $\mathbf{A} \in \mathbb{R}^{n\times m}$, $\mathbf{b} \in \mathbb{R}^n$, or an entry-wise application of the ReLU activation function, $\mathrm{ReLU}(x) \defined \max(0, x)$; the final layer $f_l$ is an affine transformation.

\subparagraph*{Local Robustness.} Neural network verification aims to provide formal guarantees about neural network behavior. In particular, verifying neural network \emph{safety}, involves proving that all network outputs corresponding to inputs satisfying a given input property also satisfy a specified output property. A widely studied safety property is \emph{local robustness}, which ensures that small (up to a chosen threshold) perturbations of an input do not change the classification output of a neural network. Formally, a neural network classifier $f : \mathbb{R}^{d} \rightarrow \mathbb{R}^{k}$ is said to be locally robust on input $\mathbf{x} \in \mathbb{R}^d$ if, given a set of allowed perturbations $C(\mathbf{x}) \subseteq \mathbb{R}^{d}$, for all $\mathbf{x}' \in \mathbb{R}^{d}$ we have:
\[
\mathbf{x}' \in C(\mathbf{x}) \Rightarrow y_f (\mathbf{x}') = y_f (\mathbf{x})
\]
that is, the classification output of the neural network is invariant to perturbations in $C(\mathbf{x})$. In the following, we focus on perturbations within the $\ell_\infty$-ball with radius $\epsilon > 0$ centered at $\mathbf{x}${, i.e., 
$C(\mathbf{x}) = B^\epsilon(\mathbf{x}) \defined \{ \mathbf{x}' \in \mathbb{R}^{d} \mid \left\lVert  \mathbf{x}' -  \mathbf{x}  \right\rVert_{\infty} \leq \epsilon \}$, 
where $\left\lVert
  \mathbf{v} \right\rVert_{\infty} \defined \max_i |v_i|$ denotes the $\ell_\infty$ norm of a vector
  $\mathbf{v}$.
In some cases, we are interested in perturbations that affect only a subset of the input features. Let $\mathcal{A} \subseteq \{1, \dots, d\}$ denote the indices of input features that may be perturbed. We will write $\overline{\mathcal{A}} := \{1, \dots, d \} \setminus \mathcal{A}$ for the complement of $\mathcal{A}$ under $\{1, \dots, d\}$. We define the set of allowed perturbations restricted to $\mathcal{A}$ as $B^\epsilon_\mathcal{A}(\mathbf{x}) \defined \{ \mathbf{x}' \in \mathbb{R}^{d} \mid \mathbf{x}'_{\overline{\mathcal{A}}} = \mathbf{x}^{}_{\overline{\mathcal{A}}} \land \left\lVert  \mathbf{x}'_\mathcal{A} -  \mathbf{x}^{}_\mathcal{A}  \right\rVert_{\infty} \leq \epsilon \}$.

We define a \emph{robustness query} as a tuple $(\mathbf{x}, \mathcal{A}, \epsilon, f)$, where $\mathbf{x}$ denotes the input vector, $\mathcal{A}$ the set of the indices of the perturbed input features, $\epsilon$ the perturbation radius, and $f$ the employed network. Specifically, $(\mathbf{x}, \mathcal{A}, \epsilon, f)$ amounts to assessing local robustness of $f$ on $\mathbf{x}$ given the set of allowed perturbations $B^\epsilon_\mathcal{A}(\mathbf{x})$. Let $\mathcal{Q}$ be the set of all robustness queries. A \emph{neural network verifier} $v : \mathcal{Q} \rightarrow \{-1, 0, 1\}$ is a function taking a robustness query as input, and returning: $1$ if local robustness is proved; $-1$ if a \emph{counterfactual} is found, that is, 
an input $\mathbf{x}' \in B^\epsilon_\mathcal{A}(\mathbf{x})$ such that $y_f(\mathbf{x}') \neq y_f(\mathbf{x})$; 
and $0$ in case the verification is inconclusive. A \emph{complete} verifier $\overline{v} : \mathcal{Q} \rightarrow \{-1, 1\}$ never returns $0$, it always either verifies the query if it holds (returns $1$), or finds a counterfactual (returns $-1$).

Internally, a neural network verifier $v$ leverages an \emph{analyzer} $a: \mathcal{Q} \rightarrow \mathbb{R}$ which, given a robustness query $(\mathbf{x}, \mathcal{A}, \epsilon, f)$, computes lower and upper bound values of each  logit in $f(\mathbf{x})$ for any input vector $\mathbf{x}' \in B^\epsilon_\mathcal{A}(\mathbf{x})$. 
To do so it can leverage different abstractions such as \textsc{IBP}~\cite{Gehr18,Gowal19}, \textsc{CROWN} / \textsc{DeepPoly}~\cite{Zhang18,Singh19a}, or $\alpha$-\textsc{CROWN}~\cite{Xu21}, among others~\cite[etc.]{Li19,Muller22,Singh18,Singh19b,Wang18}.
Based on these bounds, $a$ returns the lower bound value of the \emph{worst-case logit difference}, the minimum difference between the logit corresponding to the true class $f(\mathbf{x})_{y_f(\mathbf{x})}$ and the logit corresponding to any other class: $\text{min}_{i \not= y_f(\mathbf{x})} (f(\mathbf{x})_{y_f(\mathbf{x})} - f(\mathbf{x})_{i})$. This lower bound is central to verification: if strictly positive, 
the logit corresponding to the true class remains larger than all others for every perturbed input vector in $B^\epsilon_\mathcal{A}(\mathbf{x})$, thus local robustness of $f$ on $x$ is verified (the verifier $v$ can return $1$).
%
%

\subparagraph*{Branch-and-Bound.} Complete neural network verifiers can be interpreted under the lens of the \emph{branch-and-bound} paradigm~\cite{Bunel2018}, explicitly adopted by state-of-the-art complete neural network verifiers (e.g.,~\cite{Wang21,Zhou24}), which recursively partitions the verification problem into more tractable subproblems.

\begin{algorithm}[t]
\caption{Branch-and-Bound}
\label{alg:bab}
\begin{algorithmic}[1]
\Function{BaB}{$a, (\mathbf{x}, \mathcal{A}, \epsilon, f)$} 
\Comment{$a\colon \hat{\mathcal{Q}} \rightarrow \mathbb{R}, (\mathbf{x}, \mathcal{A}, \epsilon, f)  \in \mathcal{Q}$}
    \State unresolved $\gets \{ (\mathbf{x}, \mathcal{A}, \epsilon, f, \emptyset)  \}$ \Comment{$(\mathbf{x}, \mathcal{A}, \epsilon, f,\emptyset)  \in \hat{\mathcal{Q}}$}\label{bab:init}
    \For{$Q \in$ unresolved}
        \If{$\Call{cex}{\mathbf{x}, \mathcal{A}, \epsilon, f}$}
	\Return $-1$\label{bab:cex}
	\EndIf
    	\State unresolved $\gets$ unresolved $\setminus\, \{ Q \}$
        \State result $\gets a(Q)$ \label{bab:bounding}\Comment{Bounding step}
	\If{result $\leq 0$} \label{bab:branching1}\Comment{Branching step}
	\State $Q_1, Q_2 \gets \Call{split}{Q}$\label{bab:split}
	\State unresolved $\gets$ unresolved $\cup\, \{Q_1, Q_2 \}$\label{bab:branching2}
	\EndIf
    \EndFor
    \State \Return $1$\label{bab:verified}
\EndFunction
\end{algorithmic}
\end{algorithm}

Let $\hat{\mathcal{Q}} \defined \{ (\mathbf{x}, \mathcal{A}, \epsilon, f, C) \mid (\mathbf{x}, \mathcal{A}, \epsilon, f) \in \mathcal{Q}, C \in \mathcal{P}(\mathcal{C})\}$ be an extension of $\mathcal{Q}$ where each robustness query $(\mathbf{x}, \mathcal{A}, \epsilon, f) \in \mathcal{Q}$ is augmented by a set of splitting constraints $C \in \mathcal{P}(\mathcal{C})$, where $\mathcal{C}$ denotes the set of all possible splitting constraints and $\mathcal{P}(\mathcal{C})$ denotes its powerset. 
We refer to elements of $\hat{\mathcal{Q}}$ as \emph{robustness subproblems}. 
We assume that the analyzer $a\colon Q \rightarrow \mathbb{R}$ naturally accommodates splitting constraints. Therefore, from now on, we assume that analyzers directly operate on subproblems: $a\colon \hat{Q} \rightarrow \mathbb{R}$.
For a subproblem $Q = (\mathbf{x}, \mathcal{A}, \epsilon, f, C) \in \hat{\mathcal{Q}}$, we write $Q_C$ to denote its associated constraint set $C$.
Branch-and-bound partitions a given robustness subproblem $Q \in \hat{\mathcal{Q}}$ by growing its set of constraints $Q_C$. 
This is typically done by either partitioning the set of allowed perturbations (input splitting~\cite{Anderson20,Wang18}), or activations: in the following, we focus on the latter, termed \emph{ReLU splitting}~\cite{Bunel20,DePalma21,Henriksen21} for the networks under consideration, owing to its higher efficacy on high-dimensional input feature spaces. 
That is, we assume that constraints in $\mathcal{C}$ determine the sign of the pre-activation value of a ReLU: let $x_{i,j} = \max(0, \hat{x}_{i,j})$ denote a ReLU at the $i$-th layer and $j$-index of a neural network classifier, where $\hat{x}_{i,j}$ represents the pre-activation value input to the ReLU; the constraint $\hat{x}_{i,j} \geq 0$ determines that the ReLU behaves as the identify function, while the constraint $\hat{x}_{i,j} < 0$ determines that the ReLU behaves as the constant function equal to zero. 

Branch-and-bound verification is illustrated by Algorithm~\ref{alg:bab}. The procedure maintains a list of unresolved robustness subproblems, initially containing the given robustness query augmented with an empty set of constraints (cf. Line~\ref{bab:init}). 
Each unresolved subproblem $Q$ spawns a counterexample search (cf. Line~\ref{bab:cex}). If the search is successful, a counterfactual is found and the procedure terminates immediately.
Otherwise, the subproblem $Q$ is given to the analyzer $a$ in the \emph{bounding step} (cf. Line~\ref{bab:bounding}). 
If $v(Q) \leq 0$, the verification is inconclusive and the algorithm performs the \emph{branching step} (cf. Line~\ref{bab:branching1}): the unresolved subproblem $Q$ is split into further subproblems $Q_1$ and $Q_2$ (cf. Line~\ref{bab:split}) which are added back to the list of unresolved subproblems (cf. Line~\ref{bab:branching2}). The splitting is done according to the partitioning strategy; in case of ReLU splitting, the branching refines the subproblem $Q = (\mathbf{x}, \mathcal{A}, \epsilon, f, C)$ by selecting a ReLU whose pre-activation value $\hat{x}_{i,j}$ is not yet constrained by $C$ and partitioning the search space according to its possible sign status, yielding $Q_1 = (\mathbf{x}, \mathcal{A}, \epsilon, f, C \cup \{ \hat{x}_{i,j} \geq 0 \})$ and $Q_2 = (\mathbf{x}, \mathcal{A}, \epsilon, f, C \cup \{ \hat{x}_{i,j} < 0 \})$. The branch-and-bound process continues until all unresolved subproblems have been resolved, in which case robustness is verified (cf. Line~\ref{bab:verified}).

\begin{algorithm}[t]
\caption{Branch-and-Bound with Timeout}
\label{alg:bab-timeout}
\begin{algorithmic}[1]
\Function{BaB}{$a, (\mathbf{x}, \mathcal{A}, \epsilon, f), \BoxedString{T}$} 
\Comment{$a\colon \hat{\mathcal{Q}} \rightarrow \mathbb{R}, (\mathbf{x}, \mathcal{A}, \epsilon, f)  \in \mathcal{Q}, T \in \mathbb{N}$}
\BeginBox
	\State  $t_1 \gets \Call{time}$ 
	\EndBox
    \State unresolved $\gets \{ (\mathbf{x}, \mathcal{A}, \epsilon, f,\emptyset)  \}$ \Comment{$(\mathbf{x}, \mathcal{A}, \epsilon, f, \emptyset)  \in \hat{\mathcal{Q}}$}
    \For{$Q \in$ unresolved}
            \If{$\Call{cex}{\mathbf{x}, \mathcal{A}, \epsilon, f}$}
	\Return $-1$
	\EndIf
    	\State unresolved $\gets$ unresolved $\setminus\, \{ Q \}$
        \State result $\gets a(Q)$ \Comment{Bounding step}
	\If{result $\leq 0$} \Comment{Branching step}
	\BeginBox
	\State $t_2 \gets \Call{time}$
	\If{$t_2 - t_1 < T$}
	\EndBox
		\State $Q_1, Q_2 \gets \Call{split}{Q}$
		\State unresolved $\gets$ unresolved $\cup\, \{Q_1, Q_2 \}$
		\BeginBox
	\Else
	~\Return $0$
	\EndBox
	\EndIf
	\EndIf
    \EndFor
    \State \Return $1$
\EndFunction
\end{algorithmic}
\end{algorithm}

Due to the NP-hardness of neural network verification~\cite{Katz17}, the number of subproblems generated by branch-and-bound can grow exponentially in the worst case. Therefore, practical implementations typically enforce a timeout, terminating the search and returning an inconclusive answer once the allocated verification time has elapsed (see Algorithm~\ref{alg:bab-timeout}, where the differences with respect to Algorithm~\ref{alg:bab} are highlighted in boxes). An alternative strategy is to cap the number of branching steps or the total processed subproblems~\cite{DePalma21}.

\subsection{Verified Explanations} \label{sec:verix}

\subparagraph*{Optimal Robust Explanations.}
Recent work~\cite{Ignatiev19b,MarquesSilva22b,MarquesSilva24} has established a tight connection between local robustness and explainability of machine learning classifiers. In particular, \emph{abductive explanations}~\cite{Ignatiev19a,MarquesSilva22a}, \emph{prime implicants}~\cite{Shih18}, and \emph{sufficient reasons}~\cite{Darwiche20,Darwiche23} characterize the minimal subset of input features that are responsible for a classifier prediction, in the sense that any perturbation on the rest of the input features will never change the classification output. Restricting perturbations to a set of allowed perturbations leads to \emph{distance-restricted} or \emph{robust explanations}~\cite{LaMalfa21,Wu23,Huang24}:

\begin{definition}[Robust Explanation]
	\label{def:robust-x}
	Given a classifier $f\colon \mathbb{R}^{d} \rightarrow \mathbb{R}^{k}$, an input vector $\mathbf{x} \in \mathbb{R}^d$, and a perturbation radius $\epsilon > 0$, a \emph{robust explanation} is a subset of the indexes of the input features $\mathcal{E}_{\mathbf{x}} \subseteq \{1, \dots, d\}$ such that $f$ is locally robust on $\mathbf{x}$ to perturbations in $B^\epsilon_{\overline{\mathcal{E}}_\mathbf{x}}(\mathbf{x})$:
\[
\forall \mathbf{x}'\in B^\epsilon_{\overline{\mathcal{E}}_\mathbf{x}}(\mathbf{x})\colon y_f (\mathbf{x}') = y_f (\mathbf{x})
\]
\end{definition}
Owing to Definition~\ref{def:robust-x}, we will call $\overline{\mathcal{E}}_{\mathbf{x}}$ the \emph{invariants} for $\mathbf{x}$.
The right-most column of Figure~\ref{fig:overview-standard} shows a robust explanation, which consists of the features for which a counterexample has been found (\tikz[baseline=-0.6ex]{\node[draw, fill=yellow!70!orange, minimum width=0.34cm, minimum height=0.18cm] {\scriptsize $\times$};}) or the verifier timed out (\tikz[baseline=-0.6ex]{\node[draw, fill=yellow!70!orange, minimum width=0.34cm, minimum height=0.18cm] {\scriptsize $?$};}), while the other features are the invariants (\tikz[baseline=-0.6ex]{\node[draw, fill=gray!60, minimum width=0.34cm, minimum height=0.18cm] {\scriptsize $\checkmark$};}).
Note that, the set of all input features indexes $\{ 1, \dots, d\}$ is a trivially robust explanation. More generally, many subsets of $\{ 1, \dots, d\}$ can constitute a robust explanation. Of particular interest is an explanation containing no superfluous features~\cite{Ignatiev19b,Wu23}: 

\begin{definition}[Optimal Explanation]
	\label{def:optimal-x}
	Given a classifier $f\colon \mathbb{R}^{d} \rightarrow \mathbb{R}^{k}$, an input vector $\mathbf{x} \in \mathbb{R}^d$, and a perturbation radius $\epsilon > 0$, a robust explanation $\mathcal{E}_{\mathbf{x}}$ is \emph{optimal} if removing an index from $\mathcal{E}_{\mathbf{x}}$ would break local robustness:
	\begin{equation}
		\label{eq:optimal-x}
		\forall\ i \in \mathcal{E}_{\mathbf{x}}\colon \exists\ \mathbf{x}^* \in B^\epsilon_{\overline{\mathcal{E}}_\mathbf{x} \cup \{ i\}}(\mathbf{x})\colon y_f (\mathbf{x}^*) \not= y_f (\mathbf{x})
	\end{equation}
\end{definition}

Note that, based on Definition~\ref{def:optimal-x}, establishing the optimality of a robust explanation $\mathcal{E}_{\mathbf{x}}$ entails obtaining a \emph{counterfactual} $\mathbf{x}^*$ witnessing Equation~\eqref{eq:optimal-x} 
for each feature in $\mathcal{E}_{\mathbf{x}}$.

\subparagraph*{Computing Robust Explanations.}
Existing algorithms for computing (optimal) robust explanations exploit a local robustness verifier as an \emph{oracle} for identifying invariants or finding counterfactuals. The simplest algorithms~\cite{Wu23,Huang24} operate sequentially, mimicking the deletion-based algorithm~\cite{Chinneck91} for computing minimal unsatisfiable subsets of logic formulas. {Others}~\cite{Izza24,Wu24} implement dichotomic search~\cite{Hemery06} or adapt the quickXplain algorithm~\cite{Junker04}.

Despite these advances, scalability remains a challenge: computing optimal explanation is possible for neural network classifiers with $20$-$60$ ReLUs~\cite{Wu23,Wu24}, while only robust but not necessarily optimal explanations can be computed for larger machine learning models with a large number of input features~\cite{Wu23,Izza24}. 
In particular, in our experiments, we found that finding counterfactuals rapidly becomes infeasible. 
For instance, 
computing robust explanations with a $60$-second timeout per robustness query on state-of-the-art models (with \textasciitilde$230k$ ReLUs) from the certified training literature~\cite{DePalma24} finds no counterexamples despite multiple hours of computation per image (cf. Table~\ref{table:definition-cnn7} in Section~\ref{sec:evaluation}).

\medskip

To bridge this gap between the \emph{ideal} notion of optimal robust explanations (Definition~\ref{def:optimal-x}) and its \emph{practical realizability}, we introduce in Section~\ref{sec:v-optimal-x} the concept of \emph{verifier-optimal robust explanations}, which explicitly takes into account the underlying verifier, thereby capturing what can be provably achieved in practice. Importantly, this definition enables finding counterfactuals even at larger scale.
Building on this notion, Section~\ref{sec:favex} presents an efficient algorithm for computing (verifier-optimal) robust explanations.

\section{Verifier-Optimal Robust Explanations}\label{sec:v-optimal-x}

The definition of optimal robust explanation (Definition~\ref{def:optimal-x}) implicitly requires that the underlying verifier is complete, i.e., always able to either verify invariance or find a counterfactual for a given robustness query. This makes the notion of optimal robust explanation inherently semantic: it characterizes what explanations \emph{should} be in principle, regardless of the limitations of practical verifiers.

\begin{figure}[b]
\begin{subfigure}{0.22\textwidth}
	\centering
	\includegraphics[width=.9\textwidth]{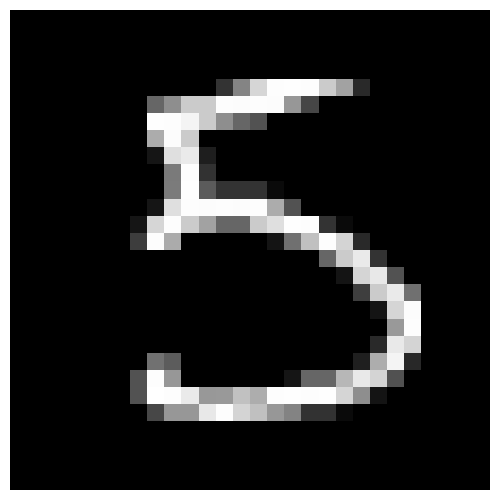}
	\captionsetup{justification=raggedright, labelsep=space, margin=6pt}
	\caption{\label{fig:MNIST-example} Sixteenth image of the MNIST~test~set.}
\end{subfigure}\hfill
\begin{subfigure}{0.22\textwidth}
	\centering
	\includegraphics[width=.9\textwidth]{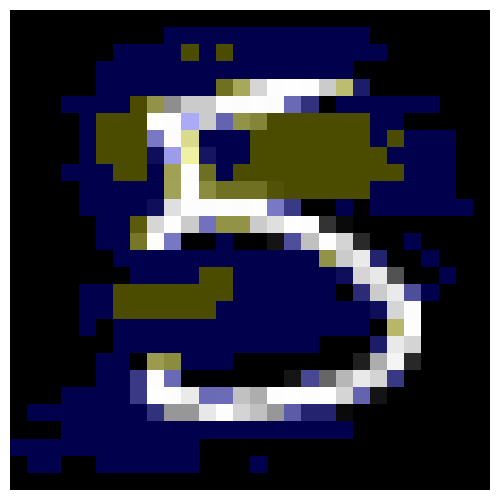}
	\captionsetup{justification=raggedright, labelsep=space, margin=6pt}
	\caption{\label{fig:MNIST-explanation} Explanation for image~(\subref{fig:MNIST-example}) at $\epsilon=0.25$.}
\end{subfigure}\hfill
\begin{subfigure}{0.22\textwidth}
	\centering
	\includegraphics[width=.9\textwidth]{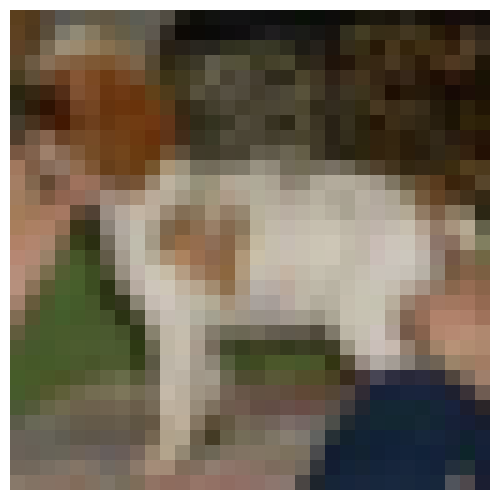}
	\captionsetup{justification=raggedright, labelsep=space, margin=6pt}
	\caption{\label{fig:CIFAR10-example} Thirteenth image of the CIFAR-10 test set.}
\end{subfigure}\hfill
\begin{subfigure}{0.22\textwidth}
	\centering
	\includegraphics[width=.9\textwidth]{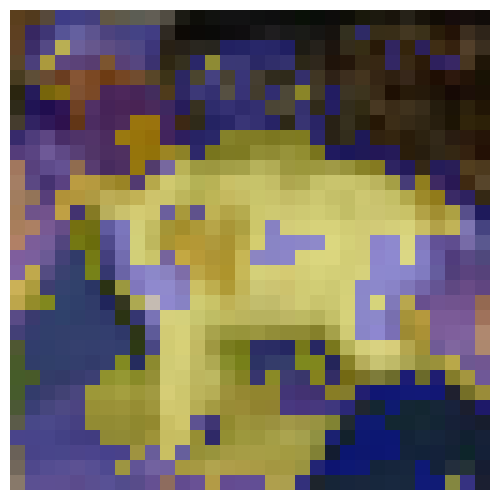}
	\captionsetup{justification=raggedright, labelsep=space, margin=6pt}
	\caption{\label{fig:CIFAR10-explanation} Explanation for image~(\subref{fig:CIFAR10-example}) at $\epsilon=\frac{16}{255}$.}
\end{subfigure}\hfill
\caption{Examples of $v$-optimal robust explanation computed on state-of-the-art networks from the certified training literature~\cite{DePalma24}.}
\label{fig:v-optimal-x}
\end{figure}
While this semantic notion provides the ideal target, in practice verifiers are incomplete and may fail to decide a robustness query. To account for such uncertainty, we introduce a \emph{practically achievable} counterpart of optimal robust explanations, that captures what can be established using a given verifier $v$. Specifically, our verifier-aware definition partitions the input feature indexes into three disjoint sets: 
\begin{enumerate}
\item the invariants $\mathcal{R}_{\mathbf{x}}$ (not highlighted in Figure~\ref{fig:v-optimal-x}): the feature indexes that are surely not part of the explanation, for which the verifier $v$ proves robustness (returns $1$);
\item the counterfactuals  $\mathcal{C}_{\mathbf{x}}$ (highlighted in yellow in Figure~\ref{fig:v-optimal-x}): the feature indexes for which the verifier $v$ finds a counterfactual (returns $-1$);
\item the unknowns $\mathcal{U}_{\mathbf{x}}$ (highlighted in blue in Figure~\ref{fig:v-optimal-x}): the feature indexes for which the verifier $v$ remains inconclusive (returns $0$);
\end{enumerate}
Intuitively, $\mathcal{R}_{\mathbf{x}}$ corresponds to input features that do not affect the classification outcome, while the union $\mathcal{C}_{\mathbf{x}} \cup \mathcal{U}_{\mathbf{x}}$ of the counterfactuals and the unknowns yields a \emph{hierarchical explanation} based on the ability of the given verifier $v$ to find counterfactuals.

We formally define our verifier-optimal explanations below.

\begin{definition}[Verifier-Optimal Robust Explanation]
	\label{def:v-optimal-x}
	Given a classifier $f : \mathbb{R}^{d} \rightarrow \mathbb{R}^{k}$, an input vector $\mathbf{x} \in \mathbb{R}^d$, a perturbation radius $\epsilon > 0$, and a verifier $v : \mathcal{Q} \rightarrow \{-1, 0, 1\}$,  a \emph{$v$-optimal robust explanation} is the union $\mathcal{C}_{\mathbf{x}} \cup \mathcal{U}_{\mathbf{x}}$ of disjoint subsets of the input feature indexes $\mathcal{C}_{\mathbf{x}} \subseteq \{1, \dots, d\}$ and $\mathcal{U}_{\mathbf{x}} \subseteq \{1, \dots, d\}$ such that:
	\begin{enumerate}
	\item $\mathcal{R}_{\mathbf{x}} \defined \overline{\mathcal{C}_{\mathbf{x}} \cup \mathcal{U}_{\mathbf{x}}}$ is the inclusion-maximal set of feature indexes such that the verifier $v$ proves local robustness of $f$ on $\mathbf{x}$ to perturbations in $B^\epsilon_{{\mathcal{R}}_\mathbf{x}}(\mathbf{x})$, i.e., $v((\mathbf{x}, \mathcal{R}_{\mathbf{x}} , \epsilon, f)) = 1$ and, for all $j  \in \mathcal{C}_{\mathbf{x}} \cup \mathcal{U}_{\mathbf{x}}$, $v((\mathbf{x}, \mathcal{R}_{\mathbf{x}} \cup \{ j \} , \epsilon, f)) \not= 1$;
	\item $v$ remains inconclusive when proving local robustness of $f$ on $\mathbf{x}$ to perturbations in $B^\epsilon_{\mathcal{R}_\mathbf{x} \cup \mathcal{U}_\mathbf{x}}(\mathbf{x})$, 
	i.e., $v((\mathbf{x}, \mathcal{R}_{\mathbf{x}} \cup \mathcal{U}_{\mathbf{x}}, \epsilon, f)) = 0$;
	\item the verifier $v$ always finds counterfactuals when proving local robustness of $f$ on $\mathbf{x}$ when perturbations are allowed on 
	all feature indexed by $\mathcal{R}_{\mathbf{x}} \cup \mathcal{U}_{\mathbf{x}}$ and \emph{any} feature indexed by $\mathcal{C}_{\mathbf{x}}$, i.e., $\forall\ i \in \mathcal{C}_{\mathbf{x}}\colon v((\mathbf{x}, \mathcal{R}_{\mathbf{x}} \cup \mathcal{U}_{\mathbf{x}} \cup \{i\}, \epsilon, f)) = -1$.
	\end{enumerate}
\end{definition}

The right-most column of Figure~\ref{fig:overview-favex} shows a verifier-optimal robust explanations, which consists of the counterfactuals in $\mathcal{C}_{\mathbf{x}}$ (\tikz[baseline=-0.6ex]{\node[draw, fill=yellow!70!orange, minimum width=0.34cm, minimum height=0.18cm] {\scriptsize $\times$};}) and the unknowns in $\mathcal{U}_{\mathbf{x}}$ (\tikz[baseline=-0.6ex]{\node[draw, fill=blue!40, minimum width=0.34cm, minimum height=0.18cm] {\scriptsize U};}), while the other features are the invariants in $\mathcal{R}_{\mathbf{x}}$ (\tikz[baseline=-0.6ex]{\node[draw, fill=gray!60, minimum width=0.34cm, minimum height=0.18cm] {\scriptsize $\checkmark$};}).

The first condition of Definition~\ref{def:v-optimal-x} implies that $\mathcal{U}_{\mathbf{x}} \cup \mathcal{C}_{\mathbf{x}}$ is a robust, yet in general not optimal, explanation. However, differently from Definitions~\ref{def:robust-x} and~\ref{def:optimal-x}, where perturbations are restricted to input features indexed by the invariants, in the second and third condition of Definition~\ref{def:v-optimal-x} perturbations are also allowed to all input features indexed by $\mathcal{U}_{\mathbf{x}}$. 
This increases the size of the perturbation space considered by the verifier, which in turn produces explanations that are \emph{potentially larger} than those obtained when considering only perturbations on the invariants. On the other hand, we argue that these explanations are \emph{more informative} since the input features indexed by $\mathcal{U}_{\mathbf{x}}$ do not produce counterfactuals individually, but counterfactuals found in combinations with these unknowns highlight clearly decisive input features for the classification outcome (the features indexed by $\mathcal{C}_{\mathbf{x}}$). In our experiments (cf. Tables~\ref{table:definition-cnn3} and~\ref{table:definition-cnn7} in Section~\ref{sec:evaluation}), 
we found that also allowing perturbations on all input features indexed by $\mathcal{U}_{\mathbf{x}}$
is crucial to enable finding counterfactuals. 
Note that, Definition~\ref{def:optimal-x} is an instance of Definition~\ref{def:v-optimal-x} in which the verifier is complete $\overline{v}\colon \mathcal{Q} \rightarrow \{-1, 1\}$: 

\begin{lemma}\label{lem}
Given a classifier $f : \mathbb{R}^{d} \rightarrow \mathbb{R}^{k}$, an input vector $\mathbf{x} \in \mathbb{R}^d$, a perturbation radius $\epsilon > 0$, and a \emph{complete} verifier $\overline{v} : \mathcal{Q} \rightarrow \{-1, 1\}$,  a \emph{$v$-optimal robust explanation} (Definition~\ref{def:v-optimal-x}) is an optimal robust explanation (Definition~\ref{def:optimal-x}).
\end{lemma}

\begin{proof}
Since the verifier is \emph{complete}, it never returns $0$. Thus, the second condition of Definition~\ref{def:v-optimal-x} never occurs. Hence, $\mathcal{U}_{\mathbf{x}} = \emptyset$ and Definition~\ref{def:v-optimal-x} coincides with Definition~\ref{def:optimal-x}. 
\end{proof}

\section{FaVeX: Faster Verified Explanations}\label{sec:favex}

In this section, we introduce \textsc{FaVeX}, our efficient algorithm for computing (verifier-optimal) robust explanations. We first present the main algorithm (Section~\ref{subsec:favex}) and then describe two acceleration strategies: an approach to speed up verification by reusing prior branches in branch-and-bound verification (Section~\ref{subsec:reuse}), and a technique to speed up the search for counterfactuals by means of adversarial attacks in a restricted space (Section~\ref{subsec:attack}).

\subsection{Computing (Verifier-Optimal) Robust Explanations}\label{subsec:favex}

\begin{algorithm}[t]
\caption{\textsc{FaVeX}}
\label{alg:favex}
\begin{algorithmic}[1]
\Function{FaVeX}{$\textsc{v-opt}, f, \mathbf{x}, \epsilon, a, \vec{\mathcal{A}}, T$} \label{favex:input} \Comment{\textsc{v-opt} $\in \{\textsc{true}, \textsc{false}\}$}
\Statex \Comment{$f\colon \mathbb{R}^{d} \rightarrow \mathbb{R}^{k}$, $\mathbf{x} \in \mathbb{R}^d$, $\epsilon > 0$, $a\colon \hat{\mathcal{Q}} \rightarrow \mathbb{R}$, $\vec{\mathcal{A}} \in Sym(\{1, \dots, d\})$, $T \in \mathbb{N}$}
        \State $\mathcal{R}_{\mathbf{x}}, \mathcal{U}_{\mathbf{x}}, \mathcal{C}_{\mathbf{x}} \gets \emptyset, \emptyset, \emptyset$ \label{favex:init}
        \State $\text{fallback} \gets \textsc{false}$
        \State $\text{batches} \gets \{ \vec{\mathcal{A}} \}$ \label{favex:all}
    \For{$B \in$ batches}
    	\State batches $\gets$ batches $\setminus\, \{ B \}$
	\If{$| B | > 1$} \label{favex:batch}	\Comment{batch robustness query}
	\If{fallback} \label{favex:triggered} \Comment{fallback to sequential processing}
	\For{$i \in B$} \label{favex:single-for1}
		
		\BeginBox
		\If{\textsc{v-opt}}
		\State result $\gets \Call{BaB}{a, (\mathbf{x}, \mathcal{R}_{\mathbf{x}} \cup \mathcal{U}_{\mathbf{x}} \cup \{ i \}, \epsilon, f), T}$ \label{favex:single-for-bab}
		\Else ~result $\gets \Call{BaB}{a, (\mathbf{x}, \mathcal{R}_{\mathbf{x}} \cup \{ i \}, \epsilon, f), T}$ \label{favex:single-for-bab-std}
		\EndBox
		\EndIf
		
		\If{result $== 1$} $\mathcal{R}_{\mathbf{x}} \gets \mathcal{R}_{\mathbf{x}} \cup \{i\}$ \label{favex:r1}	
		\Else 
		~\textbf{if} result = -1 \textbf{then}  $\mathcal{C}_{\mathbf{x}} \gets \mathcal{C}_{\mathbf{x}} \cup \{i\}$ \textbf{else} $\mathcal{U}_{\mathbf{x}} \gets \mathcal{U}_{\mathbf{x}} \cup \{i\}$ \label{favex:cu1}
		\EndIf
	\EndFor \label{favex:single-for2}
	\Else \Comment{binary search-based batch processing}
	\BeginBox
	\If{\textsc{v-opt}}
	\State result $\gets \Call{BaB}{a, (\mathbf{x}, \mathcal{R}_{\mathbf{x}} \cup \mathcal{U}_{\mathbf{x}} \cup B, \epsilon, f), T/10}$ \label{favex:batch-bab}
	\Else ~result $\gets \Call{BaB}{a, (\mathbf{x}, \mathcal{R}_{\mathbf{x}} \cup B, \epsilon, f), T/10}$ \label{favex:batch-bab-std}
	\EndBox
	\EndIf
	\If{result $== 1$} $\mathcal{R}_{\mathbf{x}} \gets \mathcal{R}_{\mathbf{x}} \cup B$	\label{favex:r2} 
	\Else \label{favex:batch-fail}
	\State $\vec{\mathcal{A}}_1, \vec{\mathcal{A}}_2 \gets \Call{halve}{\vec{\mathcal{A}}}$ \label{favex:halve}
	\State batches $\gets$ batches $\cup \,\{ \vec{\mathcal{A}}_1, \vec{\mathcal{A}}_2 \}$ \label{favex:binary-search}
	\EndIf
	\EndIf
	\Else		\label{favex:single} \Comment{single robustness query}
	
			\BeginBox
			\If{\textsc{v-opt}}
	\State result $\gets \Call{BaB}{a, (\mathbf{x}, \mathcal{R}_{\mathbf{x}} \cup \mathcal{U}_{\mathbf{x}} \cup B, \epsilon, f), T}$ \label{favex:single-bab}
	\Else ~result $\gets \Call{BaB}{a, (\mathbf{x}, \mathcal{R}_{\mathbf{x}} \cup B, \epsilon, f), T}$ \label{favex:single-bab-std}
	\EndBox
	\EndIf
	\If{result $== 1$} $\mathcal{R}_{\mathbf{x}} \gets \mathcal{R}_{\mathbf{x}} \cup B$	\label{favex:r3} 
	\Else \label{favex:single-fail} 
	\State fallback $\gets$ \textsc{true} \label{favex:fallback}
	\State \textbf{if} result = -1 \textbf{then}  $\mathcal{C}_{\mathbf{x}} \gets \mathcal{C}_{\mathbf{x}} \cup B$ \textbf{else} $\mathcal{U}_{\mathbf{x}} \gets \mathcal{U}_{\mathbf{x}} \cup B$ \label{favex:cu2}
	\EndIf
	\EndIf
    \EndFor     
        \State \Return $\mathcal{U}_{\mathbf{x}}$, $\mathcal{C}_{\mathbf{x}}$ \label{favex:return}
\EndFunction
\end{algorithmic}
\end{algorithm}

\textsc{FaVeX}, shown in Algorithm~\ref{alg:favex}, computes $v$-optimal robust explanations if the \textsc{v-opt} flag (cf. Line~\ref{favex:input}) is \textsc{true}, otherwise it computes robust explanations (cf. the parts highlighted in boxes in Algorithm~\ref{alg:favex}). It maintains three sets $\mathcal{R}_{\mathbf{x}}$, $\mathcal{U}_{\mathbf{x}}$, and $\mathcal{C}_{\mathbf{x}}$, containing invariants, unknowns, and counterfactuals, all initially empty (cf. Line~\ref{favex:init}). It distinguishes between two kinds of robustness queries: \emph{single queries}, 
in which perturbations are allowed on 
only one additional input feature besides the invariants in $\mathcal{R}_{\mathbf{x}}$ (cf. Lines~\ref{favex:single-for-bab-std} and~\ref{favex:single-bab-std})
and, if the \textsc{v-opt} flag is \textsc{true}, the unknowns in $\mathcal{U}_{\mathbf{x}}$ (cf. Lines~\ref{favex:single-for-bab} and~\ref{favex:single-bab}); and \emph{batch queries}, 
in which perturbations are allowed on 
multiple additional features simultaneously (cf. Line~\ref{favex:batch-bab} and~\ref{favex:batch-bab-std}). 
For batch queries, the timeout $T$ enforced on the underlying chosen verifier $v$ (cf. Line~\ref{favex:input}) is reduced by a factor of $10$ (cf. Line~\ref{favex:batch-bab} and~\ref{favex:batch-bab-std}). 
The reduction factor can be a parameter of the algorithm but, in our experiments, we found this value to be effective in practice.

Similarly to~\cite{Wu24}, batch queries accelerate the identification of consecutive invariants by querying the verifier for multiple candidate input features at once, instead of making consecutive separate single queries for each of them. The search for batches of consecutive invariants is based on \emph{binary search}. \textsc{FaVeX} takes as input a \emph{traversal strategy} $\vec{\mathcal{A}}$, i.e., an ordered sequence of all the input feature indexes (cf. Line~\ref{favex:input}). The choice of traversal strategy matters for producing smaller explanations~\cite{Wu23,Wu24}. We present our choice of traversal strategy later in this section. The first batch query allows perturbations on all input features indexed by $\vec{\mathcal{A}}$ (cf. Line~\ref{favex:all}). This query typically fails to verify (cf. Line~\ref{favex:batch-fail}), so $\vec{\mathcal{A}}$ is halved into two batches $\vec{\mathcal{A}}_1$ and $\vec{\mathcal{A}}_2$ (cf. Line~\ref{favex:halve}) and binary search continues over them (cf. Line~\ref{favex:binary-search}).

When the batch size is down to one (cf. Line~\ref{favex:single}), batch queries become de facto single queries, and they trigger a fallback to sequential feature processing (cf. Line~\ref{favex:fallback}) if the verifier disproves local robustness of $f$ on $\mathbf{x}$ or times out (cf. Line~\ref{favex:single-fail}). Once this occurs, unlike~\cite{Wu24}, we do not resume binary search over the next feature batches.
%
In practice, we found that
once batch queries have shrunk to size one and the verifier fails at this granularity, the likelihood of encountering further batches of consecutive invariants drops substantially. At this stage, continuing binary search adds overhead without providing benefits: it repeatedly queries increasingly small batches that are unlikely to yield new invariants, thereby worsening overall performance (cf. Tables~\ref{table:time-fc-10x2}-\ref{table:time-cnn7} in Section~\ref{sec:evaluation}). For this reason, after the fallback is triggered (cf. Line~\ref{favex:triggered}), \textsc{FaVeX} switches entirely to sequential feature processing (cf. Lines~\ref{favex:single-for1}-\ref{favex:single-for2}).

The sets $\mathcal{R}_{\mathbf{x}}$, $\mathcal{U}_{\mathbf{x}}$, and $\mathcal{C}_{\mathbf{x}}$ are updated based on the result returned by the chosen verifier $v$: $\mathcal{R}_{\mathbf{x}}$ is grown if verification succeeds ($v$ returns $1$, cf. Lines~\ref{favex:r1},~\ref{favex:r2}, and~\ref{favex:r3}); otherwise, $\mathcal{C}_{\mathbf{x}}$ or $\mathcal{U}_{\mathbf{x}}$ are grown if counterfactuals are found ($v$ returns $-1$) or if $v$ times out ($v$ returns $0$), respectively (cf. Lines~\ref{favex:cu1} and~\ref{favex:cu2}).
Finally \textsc{FaVeX} returns $\mathcal{U}_{\mathbf{x}}$ and $\mathcal{C}_{\mathbf{x}}$ (cf. Line~\ref{favex:return}), whose union $\mathcal{U}_{\mathbf{x}} \cup \mathcal{C}_{\mathbf{x}}$ yields a $v$-optimal robust explanation  (cf. Definition~\ref{def:v-optimal-x}), if the \textsc{v-opt} flag is \textsc{true}, and a robust explanation (cf. Definition~\ref{def:robust-x}), otherwise.

\begin{theorem}
Given a classifier $f : \mathbb{R}^{d} \rightarrow \mathbb{R}^{k}$, an input vector $\mathbf{x} \in \mathbb{R}^d$, a perturbation radius $\epsilon > 0$, and a verifier $v : Q \to \{-1,0,1\}$, \textsc{FaVeX} (Algorithm~\ref{alg:favex}) always terminates and
returns disjoint sets $\mathcal{U}_\mathbf{x}, \mathcal{C}_\mathbf{x} \subseteq \{1,\dots,d\}$ such that $\mathcal{C}_\mathbf{x} \cup \mathcal{U}_\mathbf{x}$
is a $v$-optimal robust explanation of $\mathbf{x}$ (Definition~\ref{def:v-optimal-x}), if the \textsc{v-opt} flag is \textsc{true}, or a robust explanations of $\mathbf{x}$ (Definition~\ref{def:robust-x}), if the \textsc{v-opt} flag is \textsc{false}.
\end{theorem}

\begin{proof}
\emph{(Termination)}
Algorithm~\ref{alg:favex} processes a finite set of input feature indexes $\{1,\dots,d\}$ using a combination of
batch splitting and sequential queries. Each step strictly reduces the size of
unprocessed batches or classifies a feature index into $\mathcal{R}_\mathbf{x}$, $\mathcal{U}_\mathbf{x}$, or $\mathcal{C}_\mathbf{x}$. Since no feature index is revisited indefinitely and each verifier call terminates (in the worst case due to timeout), the algorithm performs finitely many queries and thus terminates.

\emph{(Correctness)}
Algorithm~\ref{alg:favex} maintains a partition of the features indexes $\{1,\dots,d\}$ into invariants $\mathcal{R}_\mathbf{x}$, unknowns $\mathcal{U}_\mathbf{x}$, and
counterfactuals $\mathcal{C}_\mathbf{x}$, updated according to the verifier outcome: $v=1$ adds to $\mathcal{R}_\mathbf{x}$, $v=0$ to $\mathcal{U}_\mathbf{x}$, and $v=-1$ to $\mathcal{C}_\mathbf{x}$, with robustness queries performed over $\mathcal{R}_\mathbf{x} \cup \mathcal{U}_\mathbf{x}$, when \textsc{v-opt} is \textsc{true}, or over $\mathcal{R}_\mathbf{x}$, \textsc{v-opt} is \textsc{false}.
By construction, the resulting partition
satisfies the three conditions of Definition~\ref{def:v-optimal-x}, when \textsc{v-opt} flag is \textsc{true}, or the condition of Definition~\ref{def:robust-x}, when \textsc{v-opt} is \textsc{false}.
\end{proof}

\subparagraph*{Traversal Strategy.} 
\begin{algorithm}[t]
\caption{\textsc{FaVeX} Traversal Strategy}
\label{alg:traversal}
\begin{algorithmic}[1]
\Function{FaVeX-traversal}{$f, \mathbf{x}, \epsilon, a$} \label{traversal:input}
\Comment{$f\colon \mathbb{R}^{d} \rightarrow \mathbb{R}^{k}$, $\mathbf{x} \in \mathbb{R}^d$, $\epsilon > 0$, $a\colon \mathcal{Q} \rightarrow \mathbb{R}$}
\State $S \gets (0, \dots, 0)$ \Comment{$S \in \mathbb{R}^d$}
\For{$i \in \{1, \dots, d\}$} \label{traversal:for}
	\State $S_i \gets a( (\mathbf{x}, \{ i \}, \epsilon, f) )$\label{traversal:score}
\EndFor
\State \Return $\Call{argsort}{S, \textsc{descending}}$ \label{traversal:return}
\EndFunction
\end{algorithmic}
\end{algorithm}
The traversal strategy used by \textsc{FaVeX} is shown in Algorithm~\ref{alg:traversal}. It leverages an analyzer $a$ (cf. Line~\ref{traversal:input}) to associate a \emph{score} to each input feature index (cf. Line~\ref{traversal:for}). The score is the lower bound of the worst-case logit difference computed by $a$, when the perturbations are restricted to the currently considered input feature index $i$ (cf. Line~\ref{traversal:score}). The input feature indexes are then sorted in descending order according to their score (cf. Line~\ref{traversal:return}). The underlying intuition is that input feature indexes associated with a larger lower bound on the worst-case logit difference are more likely to be found invariants for the explanation by \textsc{FaVeX}. Note that our traversal strategy is similar to the one used in \textsc{VeriX+}~\cite{Wu24}, with a subtle but important difference: the score associated to each input feature index by \textsc{VeriX+}~\cite{Wu24} is the lower bound on the logit of the true classification output instead of the worst-case logit difference. 
We argue that prioritizing indexes by their contribution to the worst-case logit difference is more faithful to the verification objective: it directly reflects the margin that must remain positive under perturbations, rather than focusing solely on the robustness of the true class logit without accounting for competing classes. This leads to a traversal strategy that better aligns with the decision criterion used by the verifier and therefore improves the likelihood of identifying invariants early.
%
%
An empirical comparison of different traversal strategies is shown in Tables \ref{table:traversal-fc-10x2}-\ref{table:traversal-cnn7} in Section~\ref{sec:evaluation}.


\subsection{Incremental Branch-and-Bound Verification}\label{subsec:reuse}

We can now remark that most often an invocation of the verifier $v$ within \textsc{FaVeX} concerns a robustness query $(\textbf{x}, \mathcal{A}, \epsilon, f)$ that differs only slightly from the previous one (typically $\mathcal{A}$ includes only one or, sometimes, few more input feature indexes). Instead of running $v$ from scratch, we propose a strategy that exploits this incremental query evolution to reuse previous verifier computations and, in turn, speed up the overall procedure.

\begin{figure}[b]
\centering
\begin{tikzpicture}[scale=0.8, sibling distance=19em, level distance=4.5em,
    every node/.style = {shape=rectangle, rounded corners,
        draw=blue!50, align=center, very thick, fill=blue!10},
    leaf/.style = {draw=green!75!black, rounded corners=0pt, fill=green!10}]
    \node (root) {$(\mathbf{x}, \mathcal{A}, 0.6, f,\emptyset)$}
        child { 
            node[leaf] (leftleaf) {$(\mathbf{x}, \mathcal{A}, 0.6, f,\{ \hat{x}_{13} \geq 0 \})$} 
            edge from parent[draw] node[left, draw=none, fill=none] {$\hat{x}_{13} \geq 0~$}
        }
        child { 
            node (rightnode) {$(\mathbf{x}, \mathcal{A}, 0.6, f,\{ \hat{x}_{13} < 0 \})$}
            child { 
                node[leaf] (bottomleft) {$(\mathbf{x}, \mathcal{A}, 0.6, f,\{ \hat{x}_{13} < 0 , \hat{x}_{11} \geq 0 \})$} 
                edge from parent[draw] node[left, draw=none, fill=none] {$\hat{x}_{11} \geq 0~$} 
            }
            child { 
                node[leaf] (bottomright) {$(\mathbf{x}, \mathcal{A}, 0.6, f,\{ \hat{x}_{13} < 0 , \hat{x}_{11} < 0 \})$} 
                edge from parent[draw] node[right, draw=none, fill=none] {$~\hat{x}_{11} < 0$} 
            }
            edge from parent[draw] node[right, draw=none, fill=none] {$~\hat{x}_{13} < 0$}
        };
    \node[above=0em of root, draw=none, fill=none] {\textcolor{gray}{$Q$}};
    \node[left=0em of leftleaf, draw=none, fill=none] {\textcolor{gray}{$Q_1$}};
    \node[below=0em of bottomleft, draw=none, fill=none] {\textcolor{gray}{$Q^1_2$}};
    \node[below=0em of bottomright, draw=none, fill=none] {\textcolor{gray}{$Q^2_2$}};
    \node[right=0em of rightnode, draw=none, fill=none] {\textcolor{gray}{$Q_2$}};
\end{tikzpicture}
\caption{Example of search tree explored by branch-and-bound during the analysis illustrated in Figure~\ref{fig:overview-standard}, where $f$ is the neural network classifier shown in Figure~\ref{fig:bcw}, and $\mathcal{A} = \{1, 2, 3, 4, 5, 8\}$.}\label{fig:tree}
\end{figure}

Specifically, we observed that branch-and-bound may perform similar sequences of branching steps for consecutive robustness queries. Thus, we propose to save and reuse the branching steps in later branch-and-bound calls. Concretely, inspired by \textsc{ivan}~\cite{Ugare23}, we cache the (constraints associated with the) robustness subproblems at the \emph{leaves} of the search tree explored by branch-and-bound, and reuse them as starting point for subsequent branch-and-bound calls. Figure~\ref{fig:tree} shows an example of branch-and-bound search tree 
(that occurs during the analysis in Figure~\ref{fig:overview-standard})
where the initial subproblem $Q$ (local robustness verification with $0.6$-bounded $\ell_\infty$ perturbation applied to the input features $x_1$, $x_2$, $x_3$, $x_4$, $x_5$, and $x_8$) is first split into subproblems $Q_1$ and $Q_2$ (adding the constraints $\hat{x}_{13} \geq 0$ and $\hat{x}_{13} < 0$, respectively), and then $Q_2$ is further split into $Q^1_2$ and $Q^2_2$ (adding $\hat{x}_{11} \geq 0$ and $\hat{x}_{11} < 0$, respectively). In this case, we cache the (constraints associated with) subproblems $Q_1$, $Q^1_2$ and $Q^2_2$. The subsequent invocation of the verifier (local robustness verification with $0.6$-bounded $\ell_\infty$ perturbation applied to the input features $x_1$, $x_2$, $x_3$, $x_4$, $x_5$, $x_8$, and $x_9$) will reuse these subproblems (constraints) as a starting point for branch-and-bound, instead of starting from scratch with the robustness query (with an empty set of constraints).

\begin{algorithm}[t]
\caption{Branch-and-Bound (with Timeout and) Branching Save/Reuse}
\label{alg:bab-reuse}
\begin{algorithmic}[1]
\Function{BaB}{$a, (\mathbf{x}, \mathcal{A}, \epsilon, f), T, \mathbb{C}$} 
\Comment{$a\colon \hat{\mathcal{Q}} \rightarrow \mathbb{R}, (\mathbf{x}, \mathcal{A}, \epsilon, f)  \in \mathcal{Q}, T \in \mathbb{N}, \mathbb{C} \in \mathcal{P}(\mathcal{P}(\mathcal{C}))$}
	\State  $t_1 \gets \Call{time}$ 
	\BeginBox
	\For{$C \in \mathbb{C}$}  unresolved $\gets \{ (\mathbf{x}, \mathcal{A}, \epsilon, f, C)  \}$ \Comment{Reuse} \label{reuse:reuse}
    \EndBox
    \EndFor
    \State $\mathbb{C} \gets \emptyset$ 
    \For{$Q \in$ unresolved}
            \If{$\Call{cex}{\mathbf{x}, \mathcal{A}, \epsilon, f}$} \label{reuse:cex}
            \BeginBox
                    \For{$Q' \in$ unresolved} 
        $\mathbb{C} \gets \mathbb{C} \cup \{ Q'_C \} $ \Comment{Save \textnormal{unresolved}} \label{reuse:-1}
        \EndBox
        \EndFor 
	\State \Return $-1$, $\mathbb{C}$ \label{reuse:ret-1}
	\EndIf
    	\State unresolved $\gets$ unresolved $\setminus\, \{ Q \}$
        \BeginBox
        \If{a(Q) $> 0$} $\mathbb{C} \gets \mathbb{C} \cup \{ Q_C \} $ \Comment{Save $Q$} \label{reuse:1}
        \EndBox
        \Else 
	\State $t_2 \gets \Call{time}$
	\If{$t_2 - t_1 < T$}
		\State $Q_1, Q_2 \gets \Call{split}{Q}$
		\State unresolved $\gets$ unresolved $\cup\, \{Q_1, Q_2 \}$
	\Else \label{reuse:timeout}
	            \BeginBox
                    \For{$Q' \in$ unresolved} 
        $\mathbb{C} \gets \mathbb{C} \cup \{ Q'_C \} $ \Comment{Save \textnormal{unresolved}} \label{reuse:0}
        \EndBox
        \EndFor 
	\State \Return $0$, $\mathbb{C}$ \label{reuse:ret0}
	\EndIf
	\EndIf
    \EndFor
    \State \Return $1$, $\mathbb{C}$ \label{reuse:ret1}
\EndFunction
\end{algorithmic}
\end{algorithm}

Algorithm~\ref{alg:bab-reuse} shows how we augment branch-and-bound with save and reuse capabilities. When a subproblem $Q$ is verified, we cache the set of constraints $Q_C$ associated with $Q$ (cf. Line~\ref{reuse:1}). Otherwise, if a counterfactual is found (cf. Line~\ref{reuse:cex}) or the timeout is hit (cf. Line~\ref{reuse:timeout}, we cache the constraints of all still unresolved subproblems (cf. Lines~\ref{reuse:-1} and~\ref{reuse:0}). The set of cached constraints is returned together with the verification result (cf. Lines~\ref{reuse:ret-1},~\ref{reuse:ret0},~\ref{reuse:ret1}) to be reused by subsequent calls to branch-and-bound (cf. Line~\ref{reuse:reuse}).

In practice, we do not apply this save-and-reuse strategy indiscriminately. While in general it substantially accelerates verification, in some cases, it can be detrimental, e.g., when the number of cached constraints sets becomes too large. To avoid these pitfalls, we enable reuse only under specific conditions 
and impose (a configurable) limit on the size of the cache. Our practical design choices and heuristics governing when to save and when to reuse constraints of branch-and-bound leaf subproblems are detailed in Section~\ref{sec:impl-favex}.

\subsection{Restricted-Space Counterfactual Search}\label{subsec:attack}

The search for counterfactuals during branch-and-bound verification of a robustness query $(\mathbf{x}, \mathcal{A}, \epsilon, f)$ (cf., e.g., Line~\ref{bab:cex} in Algorithm~\ref{alg:bab}) typically considers the entire set of allowed perturbations $B^\epsilon_{\mathcal{A}}(\mathbf{x})$. That is, all input features indexed by $\mathcal{A}$ are treated as freely perturbable, and thus the entire set $\mathcal{A}$ constitutes the counterfactuals \emph{search space}.

However, we can again leverage the dependency between subsequent verification queries in \textsc{FaVeX} (as in Section~\ref{subsec:reuse}), in this case to narrow the counterfactual search space, thereby reducing the search cost.
When \textsc{FaVeX} operates in sequential query processing, in particular, input feature indexes are individually and iteratively added to set of feature indexes to which perturbations are allowed. If we have access to a counterfactual $\mathbf{x}' \in B^\epsilon_{{\mathcal{A}}}(\mathbf{x})$ for a robustness query $(\mathbf{x}, \mathcal{A}, \epsilon, f)$, it is extremely likely that $x'$ will be very close to a counterfactual for the next robustness query $(\mathbf{x}, \mathcal{A} \cup B, \epsilon, f)$, in the larger set of allowed perturbations $B^\epsilon_{{\mathcal{A} \cup B}}(\mathbf{x})$. 

\begin{figure}[t]
\centering
\begin{tikzpicture}[scale=1.0]

\begin{scope}[shift={(-3.5,0)}]
    \filldraw[blue!30, fill opacity=0.15, draw=blue!50, thick]
        (-0.4,-0.4) rectangle (1.6,1.6);

    \filldraw[blue!30, fill opacity=0.12, draw=blue!50, thick]
        (-1,-1) -- (-0.4,-0.4) -- (-0.4,1.6) -- (-1,1) -- cycle;

    \filldraw[blue!30, fill opacity=0.12, draw=blue!50, thick]
        (-1,-1) -- (1,-1) -- (1.6,-0.4) -- (-0.4,-0.4) -- cycle;

    \filldraw[blue!30, fill opacity=0.12, draw=blue!50, thick]
        (1,-1) -- (1.6,-0.4) -- (1.6,1.6) -- (1,1) -- cycle;

    \filldraw[blue!30, fill opacity=0.12, draw=blue!50, thick]
        (-1,1) -- (1,1) -- (1.6,1.6) -- (-0.4,1.6) -- cycle;

    \node at (0,2.35) {\small \textbf{Full-Space Search}};
    \node at (0,2) {\scriptsize $B^\epsilon_{\mathcal{A} \cup B}(\mathbf{x})$};

    \draw[->, gray!70] (0.3,0.3) -- (1.25,0.55);
    \draw[->, gray!40] (0.3,0.3) -- (-0.9,0.5);
    \draw[->, gray!40] (0.3,0.3) -- (0.8,-0.7);

    \filldraw[black] (0.3,0.3) circle (2pt);
    \node[below] at (0.3,0.3) {\scriptsize $\mathbf{x}$};

    \filldraw[red] (1.25,0.55) circle (2pt);
    \node[right] at (1.25,0.55) {\scriptsize $\mathbf{x}'$};

    \filldraw[blue!30, fill opacity=0.20, draw=blue!50, thick]
        (-1,-1) rectangle (1,1);

\end{scope}

\begin{scope}[shift={(3.5,0)}]

    \draw[draw=blue!20] (-0.4,-0.4) rectangle (1.6,1.6);
    \draw[blue!20] (-1,-1) -- (-0.4,-0.4);
    \draw[blue!20] (1,-1) -- (1.6,-0.4);
    \draw[blue!20] (-1,1) -- (-0.4,1.6);

    \node at (0,2.35) {\small \textbf{Restricted-Space Search}};
    \node at (0,2) {\scriptsize subset of $B^\epsilon_{\mathcal{A}\cup B}(\mathbf{x})$};

    \coordinate (z) at (0.8,0.9);

    \draw[fill=green!10, draw=green!85!black, thick]
        ($(z) + (-0.8,-0.4)$) rectangle ($(z) + (0.8,0.4)$);
            
        \draw[draw=blue!20] (-1,-1) rectangle (1,1);
            \draw[blue!20] (1,1) -- (1.6,1.6);

    \filldraw[black] (0.3,0.3) circle (2pt);
    \node[below] at (0.3,0.3) {\scriptsize $\mathbf{x}$};

    \filldraw[orange!80!black] (z) circle (2pt);
    \node[left] at (z) {\scriptsize $\mathbf{z}$};

    \draw[->, gray!70] (z) -- (1.155,0.655);

    \filldraw[red] (1.2,0.65) circle (2pt);
    \node[right] at (1.2,0.65) {\scriptsize $\mathbf{x}'$};
\end{scope}

\end{tikzpicture}
\caption{Illustration of full-space vs.\ restricted-space counterfactual search. 
Full-space search (left) explores the entire perturbation region $B^\epsilon_{\mathcal{A} \cup B}(\mathbf{x})$. 
Restricted-space search (right) begins from the output $\mathbf{z}$ of the previous search 
(which may not be a counterfactual) and varies only the newly added feature(s), 
thus exploring a much smaller subset of $B^\epsilon_{\mathcal{A}\cup B}(\mathbf{x})$.}
\label{fig:attack}
\end{figure}

This motivates our \emph{restricted-space counterfactual search}. Concretely, when solving the next robustness query $(\mathbf{x}, \mathcal{A} \cup B, \epsilon, f)$, instead of conducting the counterfactual search over the entire search space $\mathcal{A} \cup B$ (left of Figure~\ref{fig:attack}), we restrict it to only the newly-added input feature indexes in $B$ (right of Figure~\ref{fig:attack}). At the same time, we center the search space around the result $\mathbf{z} \in B^\epsilon_{{\mathcal{A}}}(\mathbf{x})$ of the previous counterfactual search, which is a point that was heuristically selected to be as close as possible to misclassification. That is, we fix the value of each input feature $i$ indexed by $\mathcal{A}$ to $\mathbf{z}_i$. This way, the counterfactual search explores only a thin slice of $\mathcal{A} \cup B$, but one that is highly promising. In practice, 
we found this heuristic to be extremely effective for finding valid counterfactuals quickly (cf. Tables~\ref{table:time-cnn3} and~\ref{table:time-cnn7} in Section~\ref{sec:evaluation}). Implementation details are provided in Section~\ref{sec:impl-favex}.


\section{Implementation}\label{sec:impl}

Our implementation\footnote{\url{https://github.com/alessandrodepalma/favex}} is based on the popular deep learning library \textsc{PyTorch}~\cite{Paszke2019}.

\subsection{Verifier} \label{sec:impl-verifier}
We adopt the \textsc{OVAL} branch-and-bound framework~\cite{Bunel2018,Bunel2020,DePalma21,sparsealgosDePalma2024} as the backbone for our verifier, using a fixed timeout of either $300$ seconds per query, or $60$ seconds, depending on the benchmark.
Note that while branch-and-bound is complete in principle, the timeout makes it incomplete in practice (cf. end of Section~\ref{sec:verification}), as it is unlikely that the exponential worst-case runtime will be hit during the allotted timeout on non-trivial networks~\cite{Katz17}.

\subparagraph*{Bounding Phase.} 

The bounding phase of branch-and-bound requires two components: one tries to prove robustness by operating on a network over-approximation. If the over-approximation is robust, then no counterexampe can exist. This corresponds to computing a lower bound on the worst-case logit difference (cf. Section~\ref{sec:verification}).
We mainly rely on a popular dual-based algorithm, named $\alpha$-$\beta$-CROWN~\cite{Wang21}, that solves a dual instance of a network relaxation that replaces the ReLU activations by over-approximating triangles~\cite{Ehlers2017}.  
This algorithm is both employed to bound the logit differences themselves, and to build each network pre-activation (a necessary preliminary step). In line with previous work, we only compute pre-activation bounds once at the branch-and-bound root (i.e., before any splitting)~\cite{DePalma21}.
On smaller networks, we employ a framework configuration that may resort to a tighter network over-approximation for the bounds to the logit differences if needed~\cite{sparsealgosDePalma2024}.
We solve up to $2000$ branch-and-bound sub-problems in parallel at any time, leveraging GPU acceleration through \textsc{PyTorch}.
The other component, based on the evaluation of concrete points, seeks to find a counterexampe violating robustness, upper bounding the worst-case logit difference. In practice, we first run a relatively inexpensive adversarial attack based on a local optimizer, a 10-step Projected Gradient Descent~\cite{Madry2018} (PGD-10) using the minimal logit difference as loss function, from a single input sampled uniformly within the perturbation region, before entering the branch-and-bound loop.
Within branch-and-bound, concrete points to evaluate are collected as a by-product of the over-approximations, and a more expensive local optimizer~\cite{dong2018boosting} (500-step MI-FGSM) is repeatedly run.

\subparagraph*{Branching Phase.} 
In all cases, we use the ReLU splitting partitioning strategy, which operates by splitting ambiguous ReLUs (i.e., their pre-activation can take both negative and positive values for the considered input perturbations) 
into their two linear phases~\cite{Bunel20}.
In particular, we use a recent strategy~\cite{DePalma24} that heuristically branches on the ReLU which is deemed to impact the most the tightness of the over-approximation employed during the bounding phase.

\subparagraph*{Selective MILP Calls.} 
In order to speed up verification on smaller networks, we modified the \textsc{OVAL} framework to allow for the use of a Mixed Integer Linear Programming (MILP) solver whenever a subproblem has fewer than a given number of ambiguous ReLUs. 
In practice, we use this functionality in two different settings: on very small networks, we run the MILP solver at the root of branch-and-bound (i.e., before any splitting), if the quicker dual-based bounds fail to verify robustness.
On harder networks, where this will not pay off, we call the MILP solver when no ambiguous neurons are left, which is in practice faster than ensuring convergence of dual-based bounding on the underlying linear program.
MILPs are solved using Gurobi, a commercial black-box solver~\cite{gurobi}.

\subsection{FaVeX} \label{sec:impl-favex}

\subparagraph*{Restricted-Space Counterfactual Search.}
We run the novel restricted-space counterfactual search described in Section~\ref{subsec:attack} only
if the preliminary PGD-10 attack (see Section~\ref{sec:impl-verifier}) fails to find counterexamples.
We execute several  (i.e., $128$) searches for counterexamples in parallel over the GPU, starting from different points in the restricted search space.
Specifically, we include both the extrema of the search-space (the interval lower and upper bounds), and a series of starting points sampled uniformly within the interval.
PGD-10~\cite{Madry2018} is then run for each of these starting points, using the minimal logit difference as loss function: we call this implementation \textbf{Restricted-Space Attack (RSA)}.

\subparagraph*{Traversal.} 
Our traversal strategy (cf. Section~\ref{subsec:favex}) requires as many calls to the analyzer $a$ as the number of input features. In practice, we execute a number of these in parallel on a GPU.
We consider two analyzers: $\alpha$-CROWN~\cite{Xu21}, and the less expensive IBP~\cite{Gowal19,Gehr18}.

\subparagraph*{Leaf Reuse.} 

We implemented the leaf reuse described in Section~\ref{subsec:reuse} through a direct modification of the \textsc{OVAL} branch-and-bound code. 
We store leaves as a list of branching decisions, which are then enforced after the computation of the pre-activation bounds at the root.
In practice, we do not necessarily inherit the leaves from the previous branch-and-bound call.
For efficiency reasons, we may discard leaves (i.e., start branch-and-bound from scratch) when too many of them are stored (more than a tunable threshold).
Analogously, we discard leaves in case the previous call is a timeout and \textsc{FaVeX} is processing robustness queries sequentially (cf. Section~\ref{sec:favex}), as inheriting them would likely cause further timeouts.
Finally, in case of timeouts during the binary search batch processing mode, we inherit the leaves from the branch-and-bound call preceding the timeout.

\section{Experimental Evaluation}\label{sec:evaluation}


This section presents an experimental evaluation of verifier-optimal robust explanations (Section~\ref{sec:exp-definitions}), of \textsc{FaVeX} (Section~\ref{sec:exp-methods}), and of the associated traversal strategies (Section~\ref{sec:exp-traversals}), preceded by a description of the experimental setting (Section~\ref{sec:exp-setting}).

\subsection{Experimental Setting} \label{sec:exp-setting}

We carry out experiments on computer vision datasets popular in the formal explanations literature~\cite{Wu23,Wu24}. 
In particular, we consider MNIST~\cite{LeCun2010}, a black-and-white handwritten digit classification task; a $10$-class version of GTSRB~\cite{Stallkamp2012} taken from previous work~\cite{Wu23}, a traffic sign recognition task; and CIFAR-10~\cite{Krizhevsky2009}, a popular $10$-class image classification benchmarks (examples of classes including ``airplane'', ``automobile'', ``cat'', ``dog'').

\medskip 

\noindent
\textbf{Network Architectures.}
In order to assess the scalability of \textsc{FaVeX} and of the presented definition, we consider $4$ different feedforward ReLU networks of varying sizes, detailed in Table~\ref{table:architectures}.
The hardness of neural network verification mainly depends on the number of activations acting non-linearly.
\texttt{FC-10x2} is a fully connected network with two hidden layers of width equal to $10$. \texttt{FC-50x2} displays the same structure, but with hidden layers of width $50$.
Both networks have relatively few activations, and were trained using standard (SGD-based) network training and $\ell_1$ regularization.
\texttt{CNN-7} is a $7$-layer convolutional network that is commonly employed in the certified training literature (networks trained for verified robustness), where it reaches state-of-the-art performance~\cite{DePalma24}: we use this to showcase scalability on networks relevant to the broader robust machine learning literature, and take the trained networks directly from previous work~\cite{DePalma24}.
As \texttt{CNN-7} has $2.3 \times 10^{5}$ activations, we also train (using the same training scheme) a narrower $3$-layer version, named \texttt{CNN-3}, which has $2.46 \times 10^{4}$ activations.
\texttt{CNN-7} and \texttt{CNN-3} were trained to be robust against perturbations of $\epsilon=0.2$ and $\epsilon=\frac{2}{255}$ on MNIST, and CIFAR-10, respectively.

\begin{table}[tb!]
	\caption{Size of the network architectures employed in the experimental evaluation. For convolutional networks, the number of activations is computed for GTSRB and CIFAR-10 inputs.
		\label{table:architectures}}
	\sisetup{
		detect-all = true,
		exponent-mode = scientific,
		round-mode = figures,
		round-precision = 3
	}
	\centering
	
	\begin{adjustbox}{max width=\textwidth, center}
		\begin{tabular}{
				l
				c
				S[table-format=1.3e1]
				S[table-format=1.3e1]
			}
			\toprule
			
			\multicolumn{1}{ c }{Network} &
			\multicolumn{1}{ c }{N. layers} &
			\multicolumn{1}{ c }{N. activations} &
			\multicolumn{1}{ c }{N. parameters}\\
			
			\cmidrule(lr){1-1} \cmidrule(lr){2-4}
			
			\multicolumn{1}{ c }{\texttt{FC-10x2}}
			& 3 & 20 & 30950 \\
			
			\multicolumn{1}{ c }{\texttt{FC-50x2}}
			& 3 & 100 & 156710 \\
			
			\multicolumn{1}{ c }{\texttt{CNN-3}}
			& 3 & 24576 & 218378 \\
			
			\multicolumn{1}{ c }{\texttt{CNN-7}}
			& 7 & 229888 & 17192650 \\

			\bottomrule
			
		\end{tabular}
	\end{adjustbox}
\end{table}

\medskip 

\noindent
\textbf{Computational Setup.}
We run the experimental evaluation
on a single GPU, using 6 CPU cores and allocating 20GB of RAM.
In order to run multiple experiments at once, we use different machines, but allocate each combination of network and dataset to a specific machine for consistency.
We employed the following GPU models: Nvidia RTX 6000, Nvidia RTX 8000. All machines are equipped with the following CPU: AMD EPIC 7302.

\medskip 

\noindent
\textbf{Tuning.}
Before launching the experiments, we tested different configurations of the \textsc{OVAL} framework in order to minimize the overall runtime.
We found it beneficial to avoid the custom branch-and-bound, and to call the MILP solver before any branching (see Section~\ref{sec:impl-verifier}) on all the \texttt{FC-10x2} experiments, and on the \texttt{FC-50x2} MNIST experiments.
The only hyper-parameter associated to \textsc{FaVeX} is $k$, the maximal number of leaves to be stored for reuse before discarding them.
We tune $k \in \{0, 25, 100, 500, 5000, 50000\}$ on test images not employed for the evaluation. We leave the number of leaves unbounded if the tuning points to $k=50000$: this is the case on the fully-connected network instances for which we do not directly resort to the MILP solver (\texttt{FC-50x2} on GTSRB and CIFAR-10).

\begin{table}[b!]
	\caption{On \texttt{CNN-3}, \texttt{OVAL}-optimal robust explanations are significantly more scalable than standard robust explanations computed using the same verifier. Results are averaged over the first $10$ images of the test set. Entries highlighted in bold are the smallest average runtime, and the largest average in number of found counterfactuals. \label{table:definition-cnn3}}
	\sisetup{detect-weight=true,detect-inline-weight=math,detect-mode=true}
	\centering
	
	\renewcommand{\bfseries}{\fontseries{b}\selectfont}
	\newrobustcmd{\B}{\bfseries}
	\newcommand{\boldentry}[2]{%
		\multicolumn{1}{S[table-format=#1,
			mode=text,
			text-rm=\fontseries{b}\selectfont
			]}{#2}}
	\begin{adjustbox}{max width=\textwidth, center}
		\begin{tabular}{
				l
								S[table-format=3.2]
				S[table-format=3.2]
				S[table-format=3.2]
				S[table-format=3.2]
								S[table-format=3.2]
				S[table-format=3.2]
				S[table-format=3.2]
				S[table-format=3.2]
			}
			& \multicolumn{4}{ c }{MNIST, $\epsilon=0.25$} & \multicolumn{4}{ c }{CIFAR-10, $\epsilon=\frac{8}{255}$}  \\
			\toprule
			
			\multicolumn{1}{ c }{Robust Explanation} &
			\multicolumn{1}{ c }{$|\mathcal{C}_{\mathbf{x}} \cup \mathcal{U}_{\mathbf{x}}|$} &
			\multicolumn{1}{ c }{$|\mathcal{C}_{\mathbf{x}}|$} &
			\multicolumn{1}{ c }{$|\mathcal{U}_{\mathbf{x}}|$} &
			\multicolumn{1}{ c }{time [s]} &
			\multicolumn{1}{ c }{$|\mathcal{C}_{\mathbf{x}} \cup \mathcal{U}_{\mathbf{x}}|$} &
			\multicolumn{1}{ c }{$|\mathcal{C}_{\mathbf{x}}|$} &
			\multicolumn{1}{ c }{$|\mathcal{U}_{\mathbf{x}}|$} &
			\multicolumn{1}{ c }{time [s]} \\
			
			\cmidrule(lr){1-1} \cmidrule(lr){2-5} \cmidrule(lr){6-9} 
			
			\multicolumn{1}{ c }{\textsc{Standard}}
			&247.80 & 0.00 & 247.80 & 2689.77 & 461.00 & 0.00 & 461.00 & 8984.10  \\
			
			\multicolumn{1}{ c }{\textsc{V-optimal}}
			& 254.70 & \B 160.30 & 94.40 & \B 612.75 & 462.10 & \B 210.40 & 251.70 & \B 1150.82 \\
			
			\bottomrule
			
		\end{tabular}
	\end{adjustbox}
\end{table}

\subsection{Scalable Explanations} \label{sec:exp-definitions}

We here evaluate the scalability improvements associated to the definition of verifier-optimal robust explanations, presented in Section~\ref{sec:v-optimal-x}.
Specifically, we use \textsc{FaVeX} and a fixed traversal order (\textsc{$\alpha$-FaVeX} for CIFAR-10 and \textsc{FaVeX-IBP} for MNIST) to compute both (standard) robust explanations and \texttt{OVAL}-optimal robust explanations, using a per-query timeout of $60$ seconds on \texttt{CNN-3} and \texttt{CNN-7}.
Tables~\ref{table:definition-cnn3}-\ref{table:definition-cnn7} and Figure~\ref{fig:definition} show that \texttt{OVAL}-optimal robust explanations are marginally (less than $3\%$) larger than standard robust explanations. On the other hand,
standard robust explanations are extremely expensive to compute, ranging from roughly $45$ minutes per image on MNIST for \texttt{CNN-3}, to above $7$ hours per image on CIFAR-10 for \texttt{CNN-7}.
Note that, despite this computational effort, no counterexamples can be found (cf. end of Section~\ref{sec:verix}).
On the other hand, \texttt{OVAL}-optimal robust explanations are significantly faster to compute, and result in a significant number of counterfactuals, clearly showing the superior scalability of the proposed definition.

\begin{figure*}[t!]
	\hspace{15pt}
	\begin{subfigure}{0.22\textwidth}
		\centering
		\includegraphics[width=\textwidth]{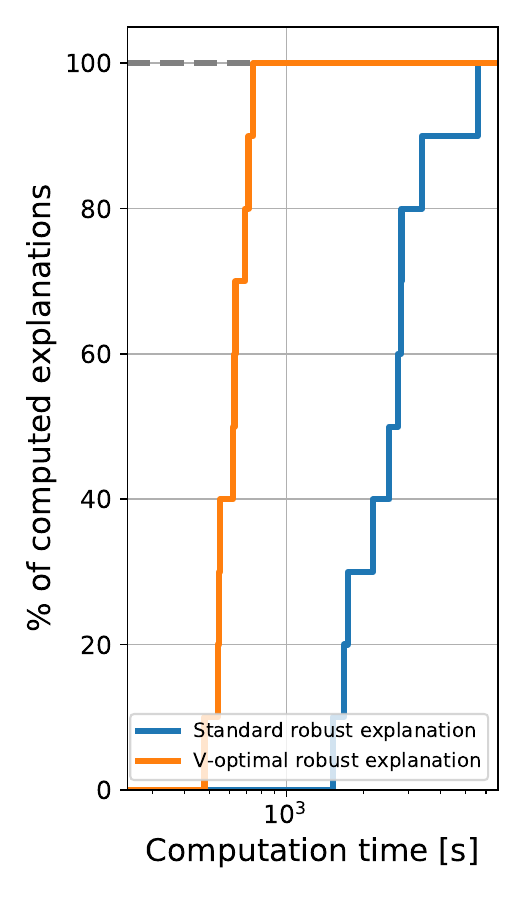}
		\captionsetup{justification=centering, labelsep=space, margin=6pt}
		\caption{\label{fig:definition-cnn3-MNIST} \texttt{CNN-3}, MNIST \\ \phantom{a}}
	\end{subfigure}\hfill
	\begin{subfigure}{0.22\textwidth}
		\centering
		\includegraphics[width=\textwidth]{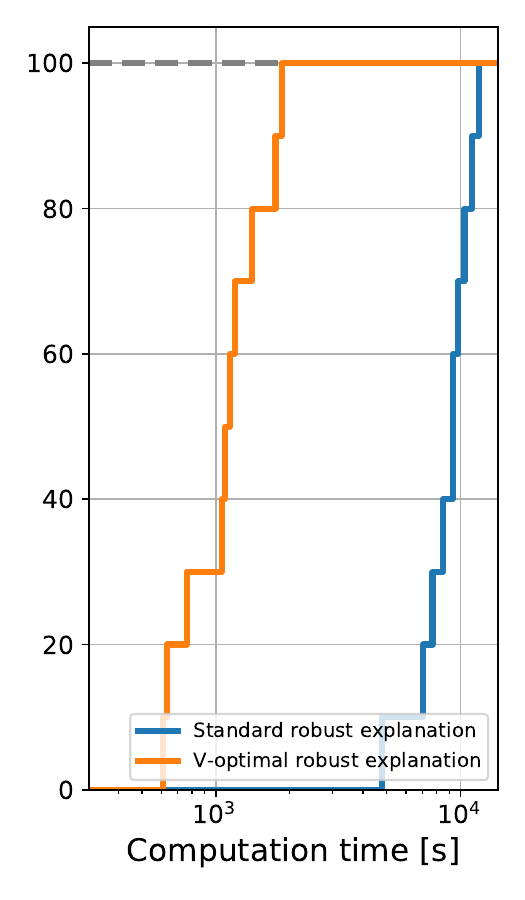}
		\captionsetup{justification=centering, labelsep=space, margin=6pt}
		\caption{\label{fig:definition-cnn3-CIFAR10} \texttt{CNN-3}, CIFAR-10}
	\end{subfigure}\hfill
	\begin{subfigure}{0.22\textwidth}
		\centering
		\includegraphics[width=\textwidth]{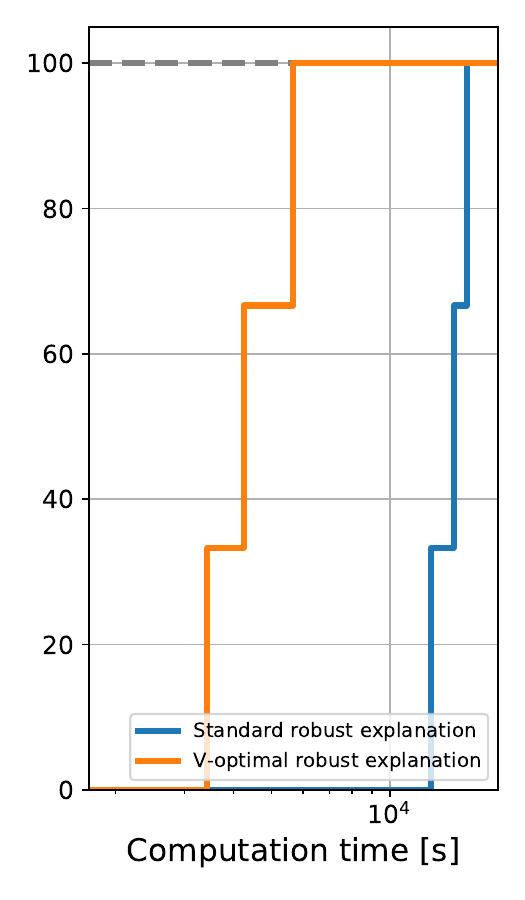}
		\captionsetup{justification=centering, labelsep=space, margin=6pt}
		\caption{\label{fig:definition-cnn7-MNIST} \texttt{CNN-7}, MNIST \\  \phantom{a}}
	\end{subfigure}\hfill
	\begin{subfigure}{0.22\textwidth}
		\centering
		\includegraphics[width=\textwidth]{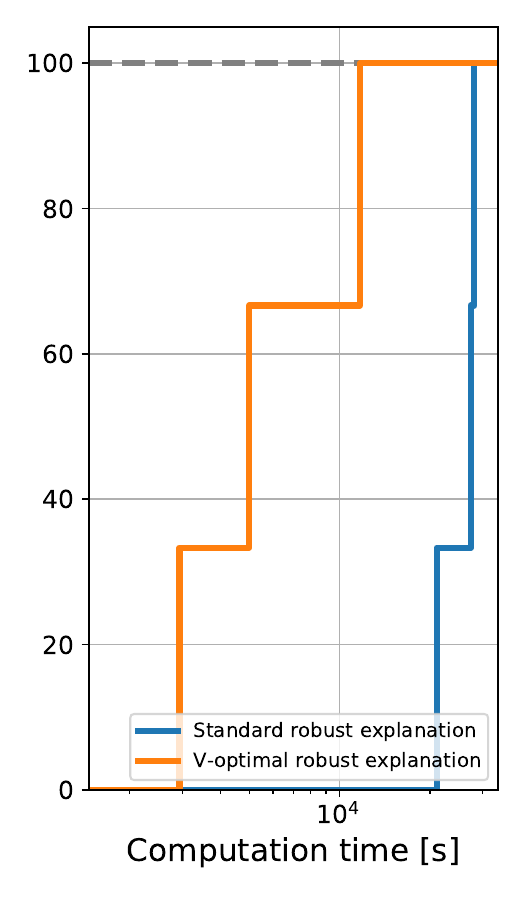}
		\captionsetup{justification=centering, labelsep=space, margin=6pt}
		\caption{\label{fig:definition-cnn7-CIFAR10} \texttt{CNN-7}, CIFAR-10}
	\end{subfigure}\hspace{15pt}
	\caption{\label{fig:definition} 
		Percentage of computed explanations as a function of runtime for Tables~\ref{table:definition-cnn3}~and~\ref{table:definition-cnn7}.  
	}
\end{figure*}

\begin{table}[b!]
	\caption{\texttt{OVAL}-optimal robust explanations are significantly more scalable than standard robust explanations computed using the same verifier on \texttt{CNN-7}. We average results over the first $3$ images of the test set, and highlight in bold the smallest average runtime the largest average in number of found counterfactuals. \label{table:definition-cnn7}}
	\sisetup{detect-weight=true,detect-inline-weight=math,detect-mode=true}
	\centering
	
	\renewcommand{\bfseries}{\fontseries{b}\selectfont}
	\newrobustcmd{\B}{\bfseries}
	\newcommand{\boldentry}[2]{%
		\multicolumn{1}{S[table-format=#1,
			mode=text,
			text-rm=\fontseries{b}\selectfont
			]}{#2}}
	\begin{adjustbox}{max width=\textwidth, center}
		\begin{tabular}{
				l
				S[table-format=3.2]
				S[table-format=3.2]
				S[table-format=3.2]
				S[table-format=5.2]
				S[table-format=3.2]
				S[table-format=3.2]
				S[table-format=5.2]
				S[table-format=5.2]
			}
			& \multicolumn{4}{ c }{MNIST, $\epsilon=0.25$} & \multicolumn{4}{ c }{CIFAR-10, $\epsilon=\frac{16}{255}$}  \\
			\toprule
			
			\multicolumn{1}{ c }{Robust Explanation} &
			\multicolumn{1}{ c }{$|\mathcal{C}_{\mathbf{x}} \cup \mathcal{U}_{\mathbf{x}}|$} &
			\multicolumn{1}{ c }{$|\mathcal{C}_{\mathbf{x}}|$} &
			\multicolumn{1}{ c }{$|\mathcal{U}_{\mathbf{x}}|$} &
			\multicolumn{1}{ c }{time [s]} &
			\multicolumn{1}{ c }{$|\mathcal{C}_{\mathbf{x}} \cup \mathcal{U}_{\mathbf{x}}|$} &
			\multicolumn{1}{ c }{$|\mathcal{C}_{\mathbf{x}}|$} &
			\multicolumn{1}{ c }{$|\mathcal{U}_{\mathbf{x}}|$} &
			\multicolumn{1}{ c }{time [s]} \\
			
			\cmidrule(lr){1-1} \cmidrule(lr){2-5} \cmidrule(lr){6-9} 
			
			\multicolumn{1}{ c }{\textsc{Standard}}
			& 452.00 & 0.00 & 452.00 & 14317.57 & 730.67 & 0.00 & 730.67 & 25487.29  \\
			
			\multicolumn{1}{ c }{\textsc{V-optimal}}
			& 456.66 & \B 207.33 & 249.33 & \B 4446.53 & 733.33 & \B 467.00 & 266.33 & \B 6532.98  \\
			
			\bottomrule
			
		\end{tabular}
	\end{adjustbox}
\end{table}

\subsection{Computing Explanations Faster} \label{sec:exp-methods}

We now turn to evaluating the efficacy of \textsc{FaVeX} when computing both standard robust explanations, and verifier-optimal robust explanations, again using a fixed traversal strategy (\textsc{$\alpha$-FaVeX} for CIFAR-10 and GTSRB, \textsc{FaVeX-IBP} for MNIST).
All considered algorithms to compute the explanations rely on the \textsc{OVAL} framework as verifier, using the same configuration (see Section~\ref{sec:impl-verifier}).
We compare against the following baselines:
\begin{itemize}
	\item \textsc{Sequential}, which processes all the input features sequentially (cf. Section~\ref{sec:favex}), mirroring what done in \textsc{VeriX}~\cite{Wu23}. 
	\item \textsc{Binary Search}, which exclusively uses the binary search-based batch processing, mirroring \textsc{VeriX+}~\cite{Wu24}.
\end{itemize}
Differently from the original works~\cite{Wu23,Wu24}, we use the \textsc{OVAL} verifier in both cases for consistency.
In order to assess the effect of each of the components of \textsc{FaVeX} in isolation, we evaluate the following methods, which are all novel in the context of verified explanations:
\begin{itemize}
	\item \textsc{BinS + Incr}, which uses incremental branch-and-bound (leaf reuse) as described in Section~\ref{subsec:reuse} on top of \textsc{Binary Search}.
	\item \textsc{BinS + Incr + RSA}, which employs the Restricted Space Attack (RSA) implementation of the counterfactual search presented in Section~\ref{subsec:attack} on top of \textsc{BinS + Incr}.
	\item \textsc{FaVeX}, which features the three following additions compared to \textsc{Binary Search}: (i) leaf reuse, (ii) RSA, and (iii) the fallback to sequential processing presented in Section~\ref{subsec:favex}.
\end{itemize}

\medskip

\noindent
\textbf{Standard Robust Explanations.}
Tables~\ref{table:time-fc-10x2}-\ref{table:time-fc-50x2}~and~Figure~\ref{fig:fc-10x2-fc-50x2} show that \textsc{FaVeX} is significantly faster than both baselines on all considered settings on both \texttt{FC-10x2} and \texttt{FC-50x2}.
\textsc{Binary Search} is consistently better than sequentially considering the input features on all benchmarks, cutting runtime almost by a third on MNIST for \texttt{FC-10x2}.
\textsc{FaVeX}, however, reduces runtime even further, attaining speedup factors of up to $13\times$ on \texttt{FC-10x2} and up to $16\times$ on \texttt{FC-50x2}.
Note that on \texttt{FC-10x2} and for MNIST on \texttt{FC-50x2}, where the MILP solver is called at the branch-and-bound root, no explicit \textsc{OVAL} branching is carried out (see Section~\ref{sec:impl-verifier}), so leaf reuse is not employed (hence the $/$ markers for \textsc{BinS + Incr} in Table~\ref{table:time-fc-50x2} and its absence in Table~\ref{table:time-fc-10x2}).
As visible from the performance of \textsc{BinS + Incr + RSA}, most of the remarkable improvements in these settings are to be attributed to the reduced-space counterfactual search. The fallback to sequential processing provides an additional speedup on most of the benchmarks, for instance cutting average runtime by more than $2\times$ on GTSRB for \texttt{FC-10x2}.
Incremental branch-and-bound (\textsc{BinS + Incr}) also achieves significant speed-ups on \texttt{FC-50x2}: more than a factor $1.7\times$ on both GTSRB and CIFAR10, pointing to its efficacy on computing standard robust explanations for small networks.

\begin{figure*}[b!]
	\hspace{15pt}
	\begin{subfigure}{0.19\textwidth}
		\centering
		\includegraphics[width=\textwidth]{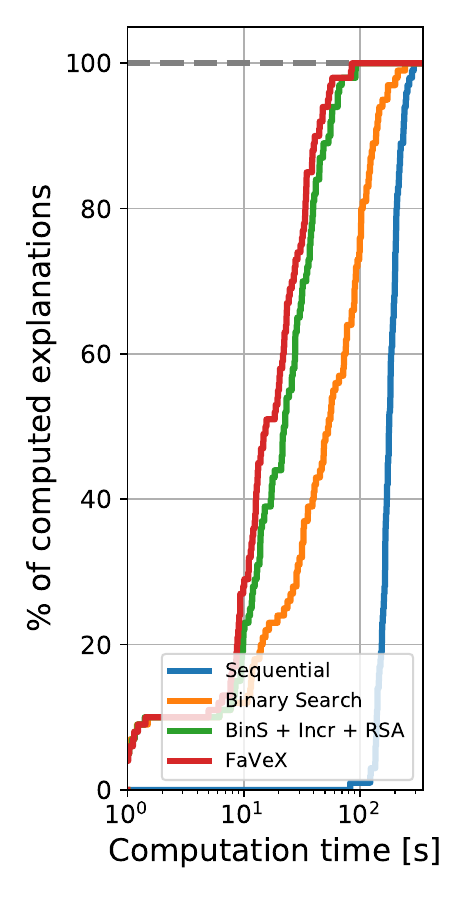}
		\captionsetup{justification=centering, labelsep=space, margin=6pt}
		\caption{\label{fig:fc-10x2-MNIST} \texttt{FC-10x2}, MNIST }
	\end{subfigure}\hfill
	\begin{subfigure}{0.19\textwidth}
		\centering
		\includegraphics[width=\textwidth]{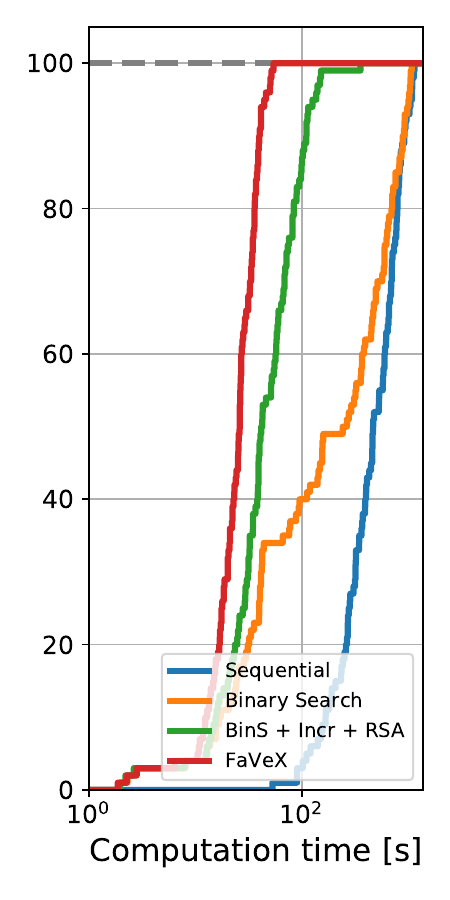}
		\captionsetup{justification=centering, labelsep=space, margin=6pt}
		\caption{\label{fig:fc-10x2-GTSRB}  \texttt{FC-10x2}, GTSRB}
	\end{subfigure}\hfill
	\begin{subfigure}{0.19\textwidth}
		\centering
		\includegraphics[width=\textwidth]{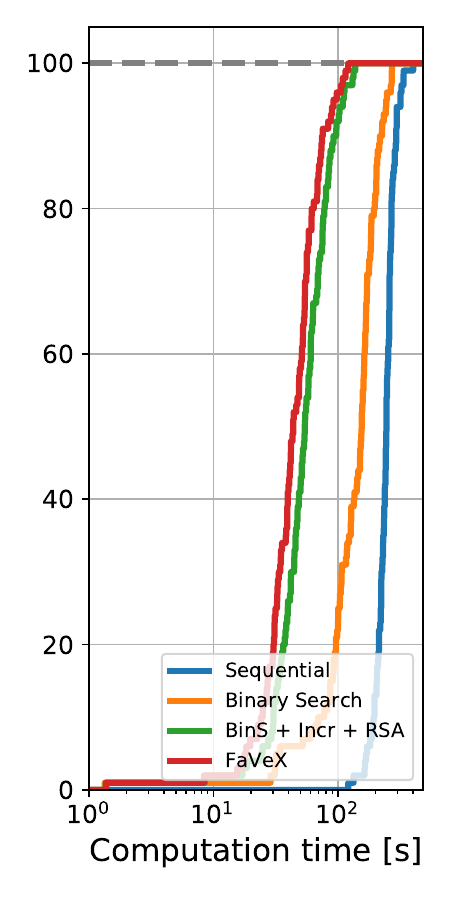}
		\captionsetup{justification=centering, labelsep=space, margin=6pt}
		\caption{\label{fig:fc-50x2-MNIST} \texttt{FC-50x2}, MNIST }
	\end{subfigure}\hfill
	\begin{subfigure}{0.19\textwidth}
		\centering
		\includegraphics[width=\textwidth]{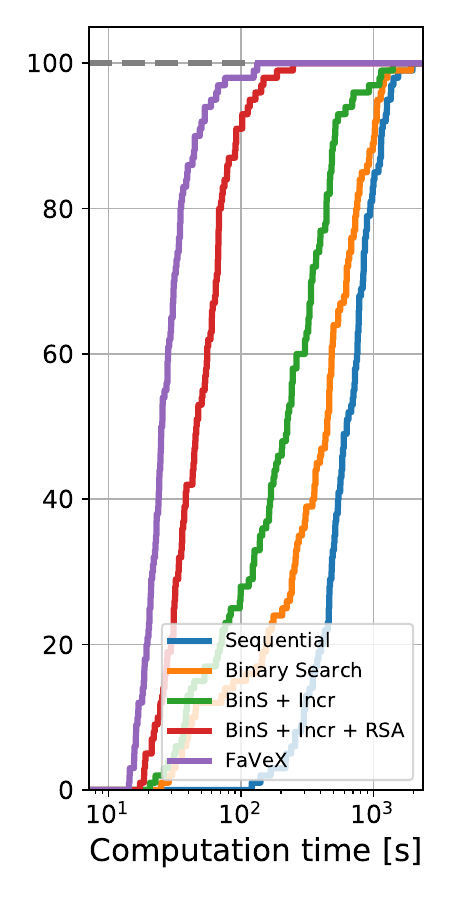}
		\captionsetup{justification=centering, labelsep=space, margin=6pt}
		\caption{\label{fig:fc-50x2-GTSRB}  \texttt{FC-50x2}, GTSRB}
	\end{subfigure}\hfill
	\begin{subfigure}{0.19\textwidth}
		\centering
		\includegraphics[width=\textwidth]{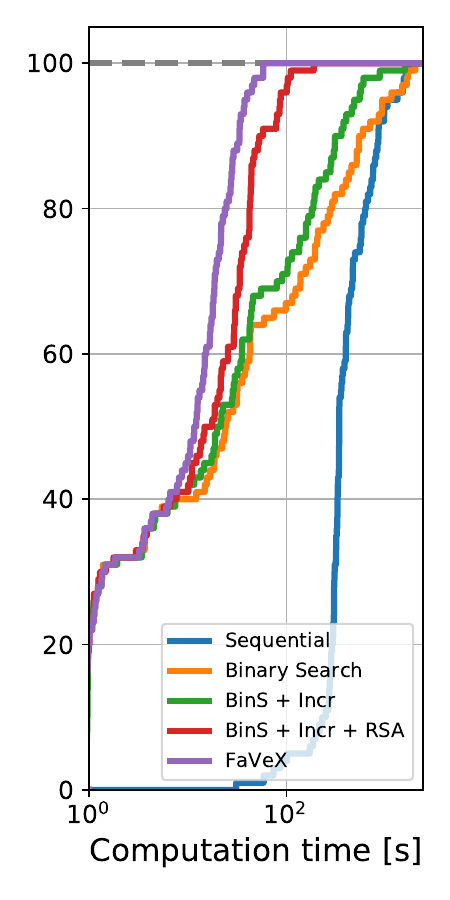}
		\captionsetup{justification=centering, labelsep=space, margin=6pt}
		\caption{\label{fig:fc-50x2-CIFAR10}  \texttt{FC-50x2}, CIFAR10}
	\end{subfigure}\hfill
	\caption{\label{fig:fc-10x2-fc-50x2} 
		Percentage of computed explanations as a function of runtime for Tables~\ref{table:time-fc-10x2}~and~\ref{table:time-fc-50x2}.  
	}
\end{figure*}

\begin{table}[tb!]
	\caption{\textsc{FaVeX} attains significant speed-ups against previous algorithms to compute standard robust explanations~\cite{Wu23,Wu24} on \texttt{FC-10x2}. Results are averaged over the first $100$ images of the test set. Bold entries correspond to the smallest average runtime. \label{table:time-fc-10x2}}
	\sisetup{detect-weight=true,detect-inline-weight=math,detect-mode=true}
	\centering
	
	\renewcommand{\bfseries}{\fontseries{b}\selectfont}
	\newrobustcmd{\B}{\bfseries}
	\newcommand{\boldentry}[2]{%
		\multicolumn{1}{S[table-format=#1,
			mode=text,
			text-rm=\fontseries{b}\selectfont
			]}{#2}}
	\begin{adjustbox}{max width=\textwidth, center}
		\begin{tabular}{
				l
				S[table-format=3.2]
				S[table-format=3.2]
				S[table-format=3.2]
				S[table-format=3.2]
			}
			& \multicolumn{2}{ c }{MNIST, $\epsilon=0.1$} & \multicolumn{2}{ c }{GTSRB, $\epsilon=0.05$}  \\
			\toprule
			
			\multicolumn{1}{ c }{Method} &
			\multicolumn{1}{ c }{$|\mathcal{E}_{\mathbf{x}}|$} &
			\multicolumn{1}{ c }{time [s]} &
			\multicolumn{1}{ c }{$|\mathcal{E}_{\mathbf{x}}|$} &
			\multicolumn{1}{ c }{time [s]} \\
			
			\cmidrule(lr){1-1} \cmidrule(lr){2-3} \cmidrule(lr){4-5} 
			
			\multicolumn{1}{ c }{\textsc{Sequential}}
			& 70.90 & 182.62 &    352.26 & 518.34   \\
			
			\multicolumn{1}{ c }{\textsc{Binary Search}}
			& 70.90 & 65.21 & 352.26 & 345.48 \\
			
			\cmidrule(lr){1-5} 
			
			\multicolumn{1}{ c }{\textsc{BinS + Incr + RSA}}
			& 70.90 & 25.98 & 352.26 & 56.26 \\
			
			\multicolumn{1}{ c }{\textsc{FaVeX}}
			&  70.90 & \B 21.21 & 352.26 & \B 26.08 \\

			\bottomrule
			
		\end{tabular}
	\end{adjustbox}
\end{table}

\begin{table}[tb!]
	\caption{\textsc{FaVeX} is significantly faster than previous algorithms to compute standard robust explanations~\cite{Wu23,Wu24} on \texttt{FC-50x2}. Results are averaged over the first $100$ images of the test set. Bold entries correspond to the smallest average runtime. \label{table:time-fc-50x2}}
	\sisetup{detect-weight=true,detect-inline-weight=math,detect-mode=true}
		\centering

		\renewcommand{\bfseries}{\fontseries{b}\selectfont}
		\newrobustcmd{\B}{\bfseries}
		\newcommand{\boldentry}[2]{%
			\multicolumn{1}{S[table-format=#1,
				mode=text,
				text-rm=\fontseries{b}\selectfont
				]}{#2}}
		\begin{adjustbox}{max width=\textwidth, center}
			\begin{tabular}{
					l
					S[table-format=3.2]
					S[table-format=3.2]
					S[table-format=3.2]
					S[table-format=3.2]
					S[table-format=3.2]
					S[table-format=3.2]
				}
				& \multicolumn{2}{ c }{MNIST, $\epsilon=0.2$} & \multicolumn{2}{ c }{GTSRB, $\epsilon=0.1$} & \multicolumn{2}{ c }{CIFAR-10, $\epsilon=\frac{2}{255}$} \\
				\toprule
				
				\multicolumn{1}{ c }{Method} &
				\multicolumn{1}{ c }{$|\mathcal{E}_{\mathbf{x}}|$} &
				\multicolumn{1}{ c }{time [s]} &
				\multicolumn{1}{ c }{$|\mathcal{E}_{\mathbf{x}}|$} &
				\multicolumn{1}{ c }{time [s]} &
				\multicolumn{1}{ c }{$|\mathcal{E}_{\mathbf{x}}|$} &
				\multicolumn{1}{ c }{time [s]} \\
				
				\cmidrule(lr){1-1} \cmidrule(lr){2-3} \cmidrule(lr){4-5} \cmidrule(lr){6-7}
				
				\multicolumn{1}{ c }{\textsc{Sequential}}
				& 82.74 & 243.38 &    287.02 & 688.24 & 204.23 & 468.23  \\
				
				\multicolumn{1}{ c }{\textsc{Binary Search}}
				& 82.74 & 146.45 & 287.02 & 476.55 & 204.23 & 197.55  \\

				\cmidrule(lr){1-7}
				
				\multicolumn{1}{ c }{\textsc{BinS + Incr}}
				& / & / & 286.99 & 275.47 & 204.23 & 110.91 \\
				
				\multicolumn{1}{ c }{\textsc{BinS + Incr + RSA}}
				& 82.74 & 58.33 & 287.01 & 55.40 & 204.23 & 25.13 \\
				
				\multicolumn{1}{ c }{\textsc{FaVeX}}
				& 82.74 & \B 48.16 & 287.02 & \B 30.48 & 204.23 & \B 13.92  \\

				\bottomrule
				
			\end{tabular}
		\end{adjustbox}
\end{table}

\medskip

\noindent
\textbf{Verifier-Optimal Robust Explanations.}
\begin{table}[tb!]
	\caption{On \texttt{CNN-3}, \textsc{FaVeX} computes \texttt{OVAL}-optimal robust explanations faster than previous computation strategies~\cite{Wu23,Wu24} while at the same time finding more counter-factuals. Results are averaged over the first $10$ images of the test set. Bold entries correspond to the smallest average runtime and the largest average in number of provided counterfactuals. \label{table:time-cnn3}}
	\sisetup{detect-weight=true,detect-inline-weight=math,detect-mode=true}
	\centering
	
	\renewcommand{\bfseries}{\fontseries{b}\selectfont}
	\newrobustcmd{\B}{\bfseries}
	\newcommand{\boldentry}[2]{%
		\multicolumn{1}{S[table-format=#1,
			mode=text,
			text-rm=\fontseries{b}\selectfont
			]}{#2}}
	\begin{adjustbox}{max width=\textwidth, center}
		\begin{tabular}{
				l
				S[table-format=3.2]
				S[table-format=3.2]
				S[table-format=4.2]
				S[table-format=3.2]
				S[table-format=3.2]
				S[table-format=4.2]
			}
			& \multicolumn{3}{ c }{MNIST, $\epsilon=0.25$} & \multicolumn{3}{ c }{CIFAR-10, $\epsilon=\frac{8}{255}$}  \\
			\toprule
			
			\multicolumn{1}{ c }{Method} &
			\multicolumn{1}{ c }{$|\mathcal{C}_{\mathbf{x}}|$} &
			\multicolumn{1}{ c }{$|\mathcal{U}_{\mathbf{x}}|$} &
			\multicolumn{1}{ c }{time [s]} &
			\multicolumn{1}{ c }{$|\mathcal{C}_{\mathbf{x}}|$} &
			\multicolumn{1}{ c }{$|\mathcal{U}_{\mathbf{x}}|$} &
			\multicolumn{1}{ c }{time [s]} \\
			
			\cmidrule(lr){1-1} \cmidrule(lr){2-4} \cmidrule(lr){5-7} 
			
			\multicolumn{1}{ c }{\textsc{Sequential}}
			& 129.10 & 125.50 & 1164.89 & 209.30 & 253.10 & 1287.96  \\
			
			\multicolumn{1}{ c }{\textsc{Binary Search}}
			& 134.20 & 120.40 & 1026.67 & 209.60 & 252.70 & 1367.25 \\
			
			\cmidrule(lr){1-7} 
			
			\multicolumn{1}{ c }{\textsc{BinS + Incr}}
			& 135.20 & 119.50 & 1040.82 & 209.60 & 252.50 & 1374.40 \\
			
			\multicolumn{1}{ c }{\textsc{BinS + Incr + RSA}}
			& 157.70 & 96.60 & 727.94 & \B  211.90 & 250.20 & 1334.00 \\
			
			\multicolumn{1}{ c }{\textsc{FaVeX}}
			&  \B 160.30 & 94.40 & \B 612.75 & 210.40 & 251.70 & \B 1150.82 \\

			\bottomrule
			
		\end{tabular}
	\end{adjustbox}
\end{table}
\begin{table}[tb!]
	\caption{On \texttt{CNN-7}, \textsc{FaVeX} computes \texttt{OVAL}-optimal robust explanations faster than previous computation strategies~\cite{Wu23,Wu24} while at the same time finding more counter-factuals. Results are averaged over the first $3$ images of the test set. Bold entries correspond to the smallest average runtime and the largest average in number of provided counterfactuals. \label{table:time-cnn7}}
	\sisetup{detect-weight=true,detect-inline-weight=math,detect-mode=true}
	\centering
	
	\renewcommand{\bfseries}{\fontseries{b}\selectfont}
	\newrobustcmd{\B}{\bfseries}
	\newcommand{\boldentry}[2]{%
		\multicolumn{1}{S[table-format=#1,
			mode=text,
			text-rm=\fontseries{b}\selectfont
			]}{#2}}
	\begin{adjustbox}{max width=\textwidth, center}
		\begin{tabular}{
				l
				S[table-format=3.2]
				S[table-format=3.2]
				S[table-format=4.2]
				S[table-format=3.2]
				S[table-format=3.2]
				S[table-format=4.2]
			}
			& \multicolumn{3}{ c }{MNIST, $\epsilon=0.25$} & \multicolumn{3}{ c }{CIFAR-10, $\epsilon=\frac{16}{255}$}  \\
			\toprule
			
			\multicolumn{1}{ c }{Method} &
			\multicolumn{1}{ c }{$|\mathcal{C}_{\mathbf{x}}|$} &
			\multicolumn{1}{ c }{$|\mathcal{U}_{\mathbf{x}}|$} &
			\multicolumn{1}{ c }{time [s]} &
			\multicolumn{1}{ c }{$|\mathcal{C}_{\mathbf{x}}|$} &
			\multicolumn{1}{ c }{$|\mathcal{U}_{\mathbf{x}}|$} &
			\multicolumn{1}{ c }{time [s]} \\
			
			\cmidrule(lr){1-1} \cmidrule(lr){2-4} \cmidrule(lr){5-7} 
			
			\multicolumn{1}{ c }{\textsc{Sequential}}
			& 142.33 & 315.00 & 7979.61 & 464.00 & 269.33 & 6868.53  \\
			
			\multicolumn{1}{ c }{\textsc{Binary Search}}
			& 148.33 & 308.67 & 9165.06 & 464.33 & 269.00 & 8284.12\\
			
			\cmidrule(lr){1-7} 
			
			\multicolumn{1}{ c }{\textsc{BinS + Incr}}
			& 151.33 & 305.00 & 9116.13 & 467.00 & 266.33 & 8231.01 \\
			
			\multicolumn{1}{ c }{\textsc{BinS + Incr + RSA}}
			& 196.33 & 260.33 & 5406.70 & \B 468.00 & 265.33 & 8304.52 \\
			
			\multicolumn{1}{ c }{\textsc{FaVeX}}
			& \B 207.33 & 249.33 & \B 4446.53 & 467.00 & 266.33 & \B 6532.98 \\

			\bottomrule
			
		\end{tabular}
	\end{adjustbox}
\end{table}
As visible from Tables~\ref{table:time-cnn3}-\ref{table:time-cnn7} and Figure~\ref{fig:cnn3-cnn7}, \textsc{FaVeX} is also beneficial when computing \texttt{OVAL}-optimal robust explanations on larger networks, with speed-ups up to $67\%$ and $79\%$ on MNIST for \texttt{CNN-3} and for \texttt{CNN-7}, respectively.
Remarkably, on MNIST, the use of \textsc{FaVeX} also results in a larger number of counterfactuals being found, pointing to its superior efficacy in terms of counterfactual search: as visible by comparing \textsc{BinS + Incr + RSA} with \textsc{BinS + Incr}, this is due to the efficacy of the reduced-space attack.
Interestingly, \textsc{Sequential} is almost always preferable to \textsc{Binary Search} on these networks: this is due to the relatively large size of the explanations. In fact, for larger explanations, \textsc{Binary Search} will go relatively deep in most branches of its recursion tree, hence resulting in a larger number of calls to the verifier.
This is reflected in the relative benefit of the fallback to sequential processing (visible by comparing \textsc{FaVeX} with \textsc{BinS + Incr + RSA}), which is the only \textsc{FaVeX} component yielding consistent speed-ups on all considered \texttt{CNN-3} and \texttt{CNN-7} benchmarks.
Incremental branch-and-bound does not appear to be beneficial on these larger networks: we ascribe this to the significantly larger activation space, which makes it unlikely for a branching decision to be effective across queries to the verifier.
Finally, it is worth pointing out that the time to compute the explanations sharply increases with the size of the network. Focusing on MNIST, this takes \textsc{FaVeX} $21.21$ seconds on average for \texttt{FC-10x2}, $48.16$ on \texttt{FC-50x2}, $612.75$ seconds on \texttt{CNN-3}, and $4446.53$ seconds on \texttt{CNN-3}.

\begin{figure*}[b!]
	\hspace{15pt}
	\begin{subfigure}{0.22\textwidth}
		\centering
		\includegraphics[width=\textwidth]{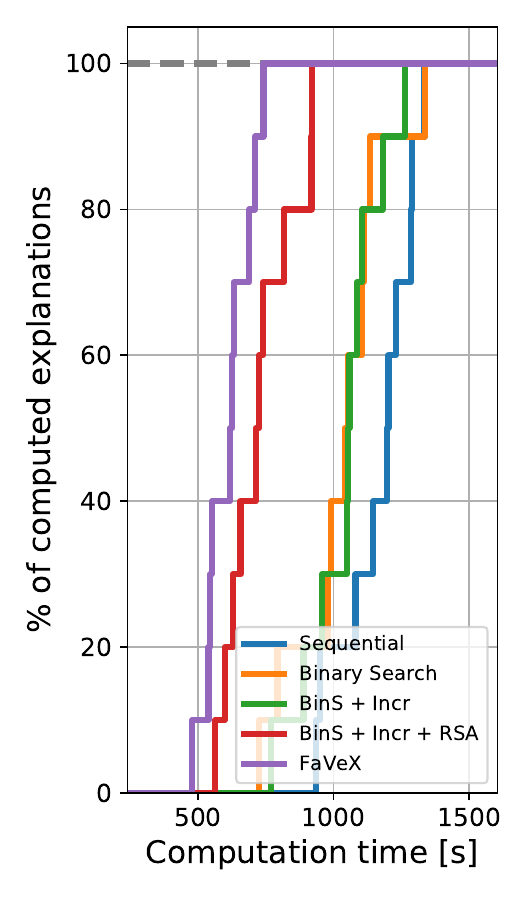}
		\captionsetup{justification=centering, labelsep=space, margin=6pt}
		\caption{\label{fig:cnn3-MNIST} \texttt{CNN-3}, MNIST \\ \phantom{a}}
	\end{subfigure}\hfill
	\begin{subfigure}{0.22\textwidth}
		\centering
		\includegraphics[width=\textwidth]{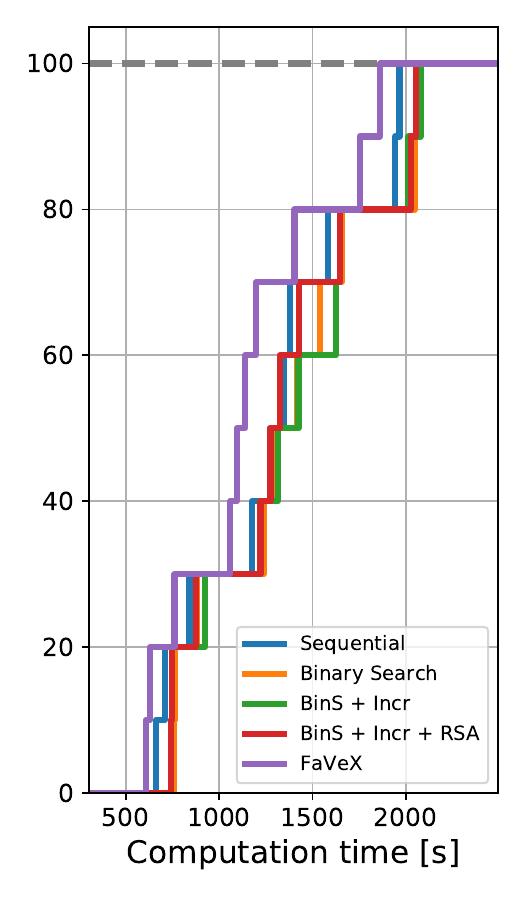}
		\captionsetup{justification=centering, labelsep=space, margin=6pt}
		\caption{\label{fig:cnn3-CIFAR10} \texttt{CNN-3}, CIFAR-10}
	\end{subfigure}\hfill
	\begin{subfigure}{0.22\textwidth}
		\centering
		\includegraphics[width=\textwidth]{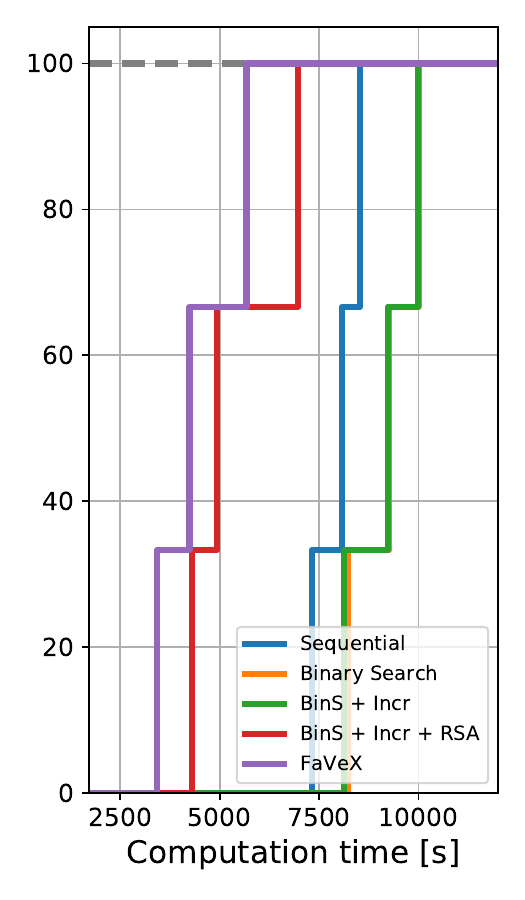}
		\captionsetup{justification=centering, labelsep=space, margin=6pt}
		\caption{\label{fig:cnn7-MNIST} \texttt{CNN-7}, MNIST \\  \phantom{a}}
	\end{subfigure}\hfill
	\begin{subfigure}{0.22\textwidth}
		\centering
		\includegraphics[width=\textwidth]{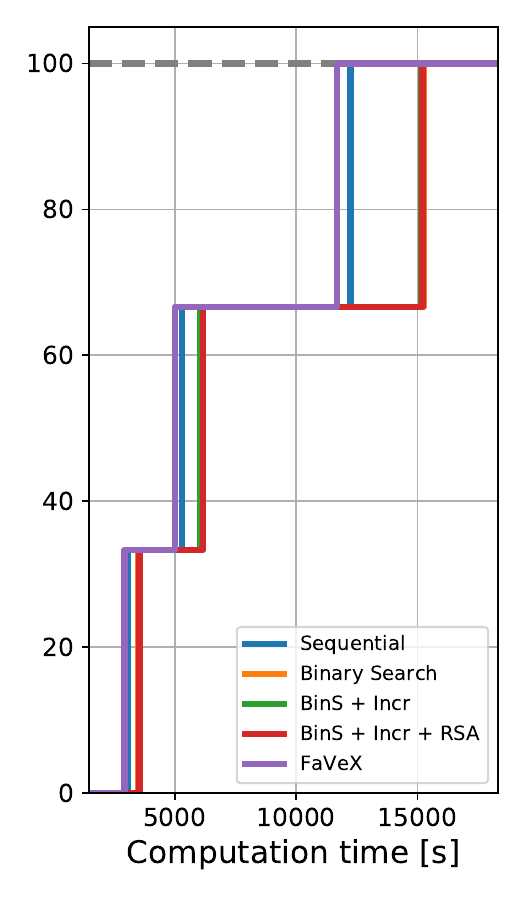}
		\captionsetup{justification=centering, labelsep=space, margin=6pt}
		\caption{\label{fig:cnn7-CIFAR10} \texttt{CNN-7}, CIFAR-10}
	\end{subfigure}\hspace{15pt}
	\caption{\label{fig:cnn3-cnn7} 
		Percentage of computed explanations as a function of runtime for Tables~\ref{table:time-cnn3}~and~\ref{table:time-cnn7}.  
	}
\end{figure*}

\subsection{Traversal Strategies Analysis} \label{sec:exp-traversals}

We now turn our attention to studying the effect of the employed traversal strategies, comparing \textsc{$\alpha$-FaVeX} and \textsc{FaVeX-IBP} with the \textsc{VeriX}~\cite{Wu23} and \textsc{VeriX+}~\cite{Wu24} traversals.
We use \textsc{FaVeX} 
to compute the explanations 
owing to its superior speed.

\medskip

\noindent
\textbf{Standard Robust Explanations.}
Tables~\ref{table:traversal-fc-10x2}~and~\ref{table:traversal-fc-50x2} show the results of the analysis for standard robust explanations on smaller networks (\texttt{FC-10x2} and \texttt{FC-50x2}).
While \textsc{$\alpha$-FaVeX} consistently yields the smallest explanations on GTSRB and CIFAR-10, it underperforms on MNIST, where strategies based on IBP attain superior performance.
Nevertheless, on GTSRB and CIFAR-10, except for the \textsc{VeriX} strategy, which appears to produce markedly larger explanations, the performance of the considered strategies is not particularly dissimilar, with the best size reduction of \textsc{$\alpha$-FaVeX} (attained on CIFAR-10 for \texttt{FC-50x2}) compared to \textsc{VeriX+} being $9\%$.
Finally, we remark that \texttt{FC-50x2} appears to be more robust than the smaller \texttt{FC-10x2}: on GTSRB, smaller explanations are produced by all traversals except \textsc{VeriX}, in spite of the twice-larger considered perturbation radius.

\begin{table}[tb!]
	\caption{Comparison of different traversal strategies for standard robust explanations on \texttt{FC-10x2}. Results are averaged over the first $100$ images of the test set. Bold entries correspond to the smallest average explanation size.
		\label{table:traversal-fc-10x2}}
	\sisetup{detect-weight=true,detect-inline-weight=math,detect-mode=true}
	\centering
	
	\renewcommand{\bfseries}{\fontseries{b}\selectfont}
	\newrobustcmd{\B}{\bfseries}
	\newcommand{\boldentry}[2]{%
		\multicolumn{1}{S[table-format=#1,
			mode=text,
			text-rm=\fontseries{b}\selectfont
			]}{#2}}
	\begin{adjustbox}{max width=\textwidth, center}
		\begin{tabular}{
				l
				S[table-format=3.2]
				S[table-format=3.2]
				S[table-format=3.2]
				S[table-format=3.2]
			}
			& \multicolumn{2}{ c }{MNIST, $\epsilon=0.1$} & \multicolumn{2}{ c }{GTSRB, $\epsilon=0.05$} \\
			\toprule
			
			\multicolumn{1}{ c }{Traversal} &
			\multicolumn{1}{ c }{$|\mathcal{E}_{\mathbf{x}}|$} &
			\multicolumn{1}{ c }{time [s]} &
			\multicolumn{1}{ c }{$|\mathcal{E}_{\mathbf{x}}|$} &
			\multicolumn{1}{ c }{time [s]} \\
			
			\cmidrule(lr){1-1} \cmidrule(lr){2-3} \cmidrule(lr){4-5}
			
			\multicolumn{1}{ c }{\textsc{VeriX}}
			& 91.26 & 38.10 & 737.02 & 37.49 \\
			
			\multicolumn{1}{ c }{\textsc{VeriX+}}
			& 72.08 & 23.05 & 357.26 & 24.06 \\
			
			\cmidrule(lr){1-5} 
			
			\multicolumn{1}{ c }{\textsc{$\alpha$-FaVeX}}
			& 79.85 & 30.24 & \B 352.26 & 26.08  \\
			
			\multicolumn{1}{ c }{\textsc{FaVeX-IBP}}
			& \B 70.90 & 21.21 & 355.52 & 24.32 \\

			\bottomrule
			
		\end{tabular}
	\end{adjustbox}
\end{table}

\begin{table}[tb!]
	\caption{Comparison of different traversal strategies for standard robust explanations on \texttt{FC-50x2}. Results are averaged over the first $100$ images of the test set. Bold entries correspond to the smallest average explanation size.
		\label{table:traversal-fc-50x2}}
	\sisetup{detect-weight=true,detect-inline-weight=math,detect-mode=true}
	\centering
	
	\renewcommand{\bfseries}{\fontseries{b}\selectfont}
	\newrobustcmd{\B}{\bfseries}
	\newcommand{\boldentry}[2]{%
		\multicolumn{1}{S[table-format=#1,
			mode=text,
			text-rm=\fontseries{b}\selectfont
			]}{#2}}
	\begin{adjustbox}{max width=\textwidth, center}
		\begin{tabular}{
				l
				S[table-format=3.2]
				S[table-format=3.2]
				S[table-format=3.2]
				S[table-format=3.2]
				S[table-format=3.2]
				S[table-format=3.2]
			}
			& \multicolumn{2}{ c }{MNIST, $\epsilon=0.2$} & \multicolumn{2}{ c }{GTSRB, $\epsilon=0.1$} & \multicolumn{2}{ c }{CIFAR-10, $\epsilon=\frac{2}{255}$} \\
			\toprule
			
			\multicolumn{1}{ c }{Traversal} &
			\multicolumn{1}{ c }{$|\mathcal{E}_{\mathbf{x}}|$} &
			\multicolumn{1}{ c }{time [s]} &
			\multicolumn{1}{ c }{$|\mathcal{E}_{\mathbf{x}}|$} &
			\multicolumn{1}{ c }{time [s]} &
			\multicolumn{1}{ c }{$|\mathcal{E}_{\mathbf{x}}|$} &
			\multicolumn{1}{ c }{time [s]} \\
			
			\cmidrule(lr){1-1} \cmidrule(lr){2-3} \cmidrule(lr){4-5} \cmidrule(lr){6-7}
			
			\multicolumn{1}{ c }{\textsc{VeriX}}
			& 94.02 & 85.00 & 851.64 & 117.20 & 310.56 & 18.40  \\
			
			\multicolumn{1}{ c }{\textsc{VeriX+}}
			& \B 77.46 & 48.17 & 289.56 & 25.74 & 222.66 & 13.85  \\
			
			\cmidrule(lr){1-7} 
			
			\multicolumn{1}{ c }{\textsc{$\alpha$-FaVeX}}
			& 102.35 & 67.01 & \B 287.02 &  30.48 & \B 204.23 & 13.92  \\
			
			\multicolumn{1}{ c }{\textsc{FaVeX-IBP}}
			& 82.74 & 48.16 & 287.25 & 25.31 & 220.96 & 14.02  \\

			\bottomrule
			
		\end{tabular}
	\end{adjustbox}
\end{table}

\medskip

\noindent
\textbf{Verifier-Optimal Robust Explanations.}
Tables~\ref{table:traversal-cnn3}~and~\ref{table:traversal-cnn7} study traversals for \textsc{OVAL}-optimal robust explanations on the larger \texttt{CNN-3} and \texttt{CNN-7}.
In order to allow for a coherent comparison, we here consider the explanation as the union of $\mathcal{C}_{\mathbf{x}}$ and $\mathcal{U}_{\mathbf{x}}$.
As for the smaller networks, \textsc{VeriX+} produces smaller explanations on MNIST.
On the other hand, the margins of improvement of \textsc{$\alpha$-FaVeX} are reduced, with \textsc{FaVeX-IBP} producing marginally smaller explanations on CIFAR-10 for \texttt{CNN-7}.
At the same time, \textsc{$\alpha$-FaVeX} non-negligibly reduces the number of counterfactuals on CIFAR-10 for both networks, highlighting that the allocation of pixels between $\mathcal{C}_{\mathbf{x}}$ and $\mathcal{U}_{\mathbf{x}}$ depends on the traversal strategy.

\begin{table}[h!]
	\caption{Comparison of different traversal strategies for \texttt{OVAL}-optimal robust explanations~on~\texttt{CNN-3}. Results are averaged over the first $10$ images of the test set. Bold entries correspond to the smallest $\mathcal{C}_{\mathbf{x}} \cup \mathcal{U}_{\mathbf{x}}$. \label{table:traversal-cnn3}}
	\sisetup{detect-weight=true,detect-inline-weight=math,detect-mode=true}
	\centering
	
	\renewcommand{\bfseries}{\fontseries{b}\selectfont}
	\newrobustcmd{\B}{\bfseries}
	\newcommand{\boldentry}[2]{%
		\multicolumn{1}{S[table-format=#1,
			mode=text,
			text-rm=\fontseries{b}\selectfont
			]}{#2}}
	\begin{adjustbox}{max width=\textwidth, center}
		\begin{tabular}{
				l
				S[table-format=3.2]
				S[table-format=3.2]
				S[table-format=3.2]
				S[table-format=4.2]
				S[table-format=3.2]
				S[table-format=3.2]
				S[table-format=3.2]
				S[table-format=4.2]
			}
			& \multicolumn{4}{ c }{MNIST, $\epsilon=0.25$} & \multicolumn{4}{ c }{CIFAR-10, $\epsilon=\frac{8}{255}$}  \\
			\toprule
			
			\multicolumn{1}{ c }{Method} &
			\multicolumn{1}{ c }{$|\mathcal{C}_{\mathbf{x}} \cup \mathcal{U}_{\mathbf{x}}|$} &
			\multicolumn{1}{ c }{$|\mathcal{C}_{\mathbf{x}}|$} &
			\multicolumn{1}{ c }{$|\mathcal{U}_{\mathbf{x}}|$} &
			\multicolumn{1}{ c }{time [s]} &
			\multicolumn{1}{ c }{$|\mathcal{C}_{\mathbf{x}} \cup \mathcal{U}_{\mathbf{x}}|$} &
			\multicolumn{1}{ c }{$|\mathcal{C}_{\mathbf{x}}|$} &
			\multicolumn{1}{ c }{$|\mathcal{U}_{\mathbf{x}}|$} &
			\multicolumn{1}{ c }{time [s]} \\
			
			\cmidrule(lr){1-1} \cmidrule(lr){2-5} \cmidrule(lr){6-9} 
			
			\multicolumn{1}{ c }{\textsc{VeriX}}
			& 282.80 &160.50 & 122.30 & 974.35 & 765.50 & 437.20 & 328.30 & 1925.09 \\
			
			\multicolumn{1}{ c }{\textsc{VeriX+}}
			& \B 247.40 & 155.20 & 92.20 & 576.67 & 468.90 & 262.00 & 206.90 & 915.45  \\
			
			\cmidrule(lr){1-9} 
			
			\multicolumn{1}{ c }{\textsc{$\alpha$-FaVeX}}
			& 289.90 & 181.20 & 108.70 & 696.31 & \B 462.10 & 210.40 & 251.70 & 1150.82 \\
			
			\multicolumn{1}{ c }{\textsc{FaVeX-IBP}}
			& 254.70 & 160.30 & 94.40 & 612.75 & 466.40 & 250.90 & 215.50 & 941.95 \\

			\bottomrule
			
		\end{tabular}
	\end{adjustbox}
\end{table}

\begin{table}[h!]
	\caption{Comparison of different traversal strategies for \texttt{OVAL}-optimal robust explanations~on~\texttt{CNN-7}. Results are averaged over the first $3$ images of the test set. Bold entries correspond to the smallest $\mathcal{C}_{\mathbf{x}} \cup \mathcal{U}_{\mathbf{x}}$. \label{table:traversal-cnn7}}
	\sisetup{detect-weight=true,detect-inline-weight=math,detect-mode=true}
	\centering
	
	\renewcommand{\bfseries}{\fontseries{b}\selectfont}
	\newrobustcmd{\B}{\bfseries}
	\newcommand{\boldentry}[2]{%
		\multicolumn{1}{S[table-format=#1,
			mode=text,
			text-rm=\fontseries{b}\selectfont
			]}{#2}}
	\begin{adjustbox}{max width=\textwidth, center}
		\begin{tabular}{
				l
				S[table-format=3.2]
				S[table-format=3.2]
				S[table-format=3.2]
				S[table-format=4.2]
				S[table-format=3.2]
				S[table-format=3.2]
				S[table-format=3.2]
				S[table-format=4.2]
			}
			& \multicolumn{3}{ c }{MNIST, $\epsilon=0.25$} & \multicolumn{3}{ c }{CIFAR-10, $\epsilon=\frac{16}{255}$}  \\
			\toprule
			
			\multicolumn{1}{ c }{Method} &
			\multicolumn{1}{ c }{$|\mathcal{C}_{\mathbf{x}} \cup \mathcal{U}_{\mathbf{x}}|$} &
			\multicolumn{1}{ c }{$|\mathcal{C}_{\mathbf{x}}|$} &
			\multicolumn{1}{ c }{$|\mathcal{U}_{\mathbf{x}}|$} &
			\multicolumn{1}{ c }{time [s]} &
			\multicolumn{1}{ c }{$|\mathcal{C}_{\mathbf{x}} \cup \mathcal{U}_{\mathbf{x}}|$} &
			\multicolumn{1}{ c }{$|\mathcal{C}_{\mathbf{x}}|$} &
			\multicolumn{1}{ c }{$|\mathcal{U}_{\mathbf{x}}|$} &
			\multicolumn{1}{ c }{time [s]} \\
			
			\cmidrule(lr){1-1} \cmidrule(lr){2-5} \cmidrule(lr){6-9} 
			
			\multicolumn{1}{ c }{\textsc{VeriX}}
			& 547.00 & 123.33 & 423.67 & 7763.91 & 945.00 & 728.67 & 216.33 & 4974.69 \\
			
			\multicolumn{1}{ c }{\textsc{VeriX+}}
			&  \B 429.33 & 196.67 & 232.67 & 4394.67 & 735.67 & 522.00 & 213.67 & 5105.86  \\
			
			\cmidrule(lr){1-9} 
			
			\multicolumn{1}{ c }{\textsc{$\alpha$-FaVeX}}
			&  552.33 & 234.67 & 317.67 & 5344.59 & 733.33 & 467.00 & 266.33 & 6532.98 \\
			
			\multicolumn{1}{ c }{\textsc{FaVeX-IBP}}
			& 456.67 & 207.33 & 249.33 & 4446.53 & \B 729.33 & 512.67 & 216.67 & 5099.98 \\

			\bottomrule
			
		\end{tabular}
	\end{adjustbox}
\end{table}

\subsection{Ablation: Cap on Stored Leaves} \label{sec:exp-leaves}

\textsc{FaVeX} features only a single hyper-parameter: $k$, the maximum number of leaves to be stored for reuse, tuned as described in Section~\ref{sec:exp-setting}. As outlined in Section~\ref{sec:impl-favex}, when the number of branch-and-bound leaves from a call to the verifier is greater than $k$, branch-and-bound for the next verifier call will start from scratch, discarding any leaf information.
In order to report the effect of $k$ on final performance, we here show results for three values: one relatively small ($k=25$), one intermediate ($k=500$), and one associated to uncapped storage ($k=\infty$). 
Consistently with our focus on the efficiency of \textsc{FaVeX}, we report results for the setups from Section~\ref{sec:exp-methods}, excluding benchmarks where the MILP solver is called at the branch-and-bound root.
Tables~\ref{table:leaves-fc-50x2},~\ref{table:leaves-cnn3}~and~\ref{table:leaves-cnn7} show that \textsc{FaVeX} is relatively insensitive to its hyper-parameter, and that the tuning process consistently leads to good performance. 
Consistently with the fact that leaf reuse was found to be ineffective on \texttt{CNN-3} and \texttt{CNN-7} in Section~\ref{sec:exp-methods}, allowing for the reuse of a large number of leaves leads has a negative impact on performance in Tables~\ref{table:leaves-cnn3}~and~\ref{table:leaves-cnn7}.

\begin{table}[tb!]
	\caption{ 
		Ablation on $k$, the maximum number of leaves that can be stored for reuse within \textsc{FaVeX}, for the experiments from Table~\ref{table:time-fc-50x2}. Results are averaged over the first $100$ images of the test set. Bold entries correspond to the tuned $k$ value employed in the original experiment. \label{table:leaves-fc-50x2}}
	\sisetup{detect-weight=true,detect-inline-weight=math,detect-mode=true}
	\centering
	
	\renewcommand{\bfseries}{\fontseries{b}\selectfont}
	\newrobustcmd{\B}{\bfseries}
	\newcommand{\boldentry}[2]{%
		\multicolumn{1}{S[table-format=#1,
			mode=text,
			text-rm=\fontseries{b}\selectfont
			]}{#2}}
	\begin{adjustbox}{max width=\textwidth, center}
		\begin{tabular}{
				l
				S[table-format=3.2]
				S[table-format=3.2]
				S[table-format=3.2]
				S[table-format=3.2]
			}
			& \multicolumn{2}{ c }{GTSRB, $\epsilon=0.1$} & \multicolumn{2}{ c }{CIFAR-10, $\epsilon=\frac{2}{255}$} \\
			\toprule
			
			\multicolumn{1}{ c }{$k$} &
			\multicolumn{1}{ c }{$|\mathcal{E}_{\mathbf{x}}|$} &
			\multicolumn{1}{ c }{time [s]} &
			\multicolumn{1}{ c }{$|\mathcal{E}_{\mathbf{x}}|$} &
			\multicolumn{1}{ c }{time [s]} \\
			
			\cmidrule(lr){1-1} \cmidrule(lr){2-3} \cmidrule(lr){4-5}
			
			\multicolumn{1}{ c }{$25$}
			& 287.02 & 40.95 &   204.23 & 14.44  \\
			
			\multicolumn{1}{ c }{$500$}
			& 287.02 & 29.42 & 204.23 & 13.90  \\
			
			\multicolumn{1}{ c }{$\infty$}
			& \B 287.02 & \B 30.48 & \B 204.23 & \B 13.92  \\

			\bottomrule
			
		\end{tabular}
	\end{adjustbox}
\end{table}

\begin{table}[tb!]
	\caption{Ablation for the maximum number of leaves that can be stored for reuse within \textsc{FaVeX} (denoted $k$) for the experiments from Table~\ref{table:time-cnn3}. Results are averaged over the first $10$ images of the test set. Bold entries correspond to the tuned $k$ value employed in the original experiment. \label{table:leaves-cnn3}}
	\sisetup{detect-weight=true,detect-inline-weight=math,detect-mode=true}
	\centering
	
	\renewcommand{\bfseries}{\fontseries{b}\selectfont}
	\newrobustcmd{\B}{\bfseries}
	\newcommand{\boldentry}[2]{%
		\multicolumn{1}{S[table-format=#1,
			mode=text,
			text-rm=\fontseries{b}\selectfont
			]}{#2}}
	\begin{adjustbox}{max width=\textwidth, center}
		\begin{tabular}{
				l
				S[table-format=3.2]
				S[table-format=3.2]
				S[table-format=4.2]
				S[table-format=3.2]
				S[table-format=3.2]
				S[table-format=4.2]
			}
			& \multicolumn{3}{ c }{MNIST, $\epsilon=0.25$} & \multicolumn{3}{ c }{CIFAR-10, $\epsilon=\frac{8}{255}$}  \\
			\toprule
			
			\multicolumn{1}{ c }{$k$} &
			\multicolumn{1}{ c }{$|\mathcal{C}_{\mathbf{x}}|$} &
			\multicolumn{1}{ c }{$|\mathcal{U}_{\mathbf{x}}|$} &
			\multicolumn{1}{ c }{time [s]} &
			\multicolumn{1}{ c }{$|\mathcal{C}_{\mathbf{x}}|$} &
			\multicolumn{1}{ c }{$|\mathcal{U}_{\mathbf{x}}|$} &
			\multicolumn{1}{ c }{time [s]} \\
			
			\cmidrule(lr){1-1} \cmidrule(lr){2-4} \cmidrule(lr){5-7} 
			
			\multicolumn{1}{ c }{$25$}
			& \B 160.30 & \B 94.40 & \B 612.75 & 210.00 & 252.20 & 1139.99  \\
			
			\multicolumn{1}{ c }{$500$}
			& 160.20 & 94.20 & 629.48 & \B 210.40 & \B 251.70 & \B 1150.82\\
			
			\multicolumn{1}{ c }{$\infty$}
			&  159.00 & 96.00 & 785.96 & 209.80 & 253.10 & 1204.79 \\

			\bottomrule
			
		\end{tabular}
	\end{adjustbox}
\end{table}
\begin{table}[tb!]
	\caption{Ablation on $k$, the maximum number of leaves that can be stored for reuse within \textsc{FaVeX}, for the experiments from Table~\ref{table:time-cnn7}. Results are averaged over the first $3$ images of the test set. Bold entries correspond to the tuned $k$ value employed in the original experiment. \label{table:leaves-cnn7}}
	\sisetup{detect-weight=true,detect-inline-weight=math,detect-mode=true}
	\centering
	
	\renewcommand{\bfseries}{\fontseries{b}\selectfont}
	\newrobustcmd{\B}{\bfseries}
	\newcommand{\boldentry}[2]{%
		\multicolumn{1}{S[table-format=#1,
			mode=text,
			text-rm=\fontseries{b}\selectfont
			]}{#2}}
	\begin{adjustbox}{max width=\textwidth, center}
		\begin{tabular}{
				l
				S[table-format=3.2]
				S[table-format=3.2]
				S[table-format=4.2]
				S[table-format=3.2]
				S[table-format=3.2]
				S[table-format=4.2]
			}
			& \multicolumn{3}{ c }{MNIST, $\epsilon=0.25$} & \multicolumn{3}{ c }{CIFAR-10, $\epsilon=\frac{16}{255}$}  \\
			\toprule
			
			\multicolumn{1}{ c }{$k$} &
			\multicolumn{1}{ c }{$|\mathcal{C}_{\mathbf{x}}|$} &
			\multicolumn{1}{ c }{$|\mathcal{U}_{\mathbf{x}}|$} &
			\multicolumn{1}{ c }{time [s]} &
			\multicolumn{1}{ c }{$|\mathcal{C}_{\mathbf{x}}|$} &
			\multicolumn{1}{ c }{$|\mathcal{U}_{\mathbf{x}}|$} &
			\multicolumn{1}{ c }{time [s]} \\
			
			\cmidrule(lr){1-1} \cmidrule(lr){2-4} \cmidrule(lr){5-7} 
			
			\multicolumn{1}{ c }{$25$}
			& \B 207.33 & \B  249.33 & \B  4446.53 & 466.00 & 267.33 & 6542.63 \\
			
			\multicolumn{1}{ c }{$500$}
			& 186.33 & 270.00 & 4889.95 & \B 467.00 & \B 266.33 & \B 6532.98 \\
			
			\multicolumn{1}{ c }{$\infty$}
			& 207.33 & 252.67 & 4658.70 & 463.67 & 272.67 & 6759.26 \\
			
			\bottomrule
			
		\end{tabular}
	\end{adjustbox}
\end{table}

\section{Related Work}\label{sec:relw}

A large body of work from the machine learning community has sought to provide explanations on the decisions of machine learning models.
Earlier works from the machine learning community, for instance, propose to build interpretable approximations of the original model~\cite{lundberg2017unified,ribeiro2018anchors,ribeiro2016should}.
Another line of work tries to find points close to the input such that the decision would change (counterfactuals), which are robust in either the input space~\cite{Leofante23}, or the parameter space~\cite{pmlr-v202-hamman23a, pmlr-v222-jiang24a}.
For a survey on robust counterfactual explanation in machine learning, we direct the reader to~\cite{abs-2402-01928}.
At the same time, multiple approaches to heuristically select a subset of features deemed responsible for the predictions exist, typically based on model gradients~\cite{simonyan2013deep,selvaraju2017grad}.
However, these heuristics lack any type of formal guarantees, which are however crucial for safety-critical systems.
Earlier work on formal explainability was carried out by the logic community, focusing on abductive explanations~\cite{MarquesSilva2021,MarquesSilva22b,MarquesSilva22a} (cf. Section~\ref{sec:verix}).
This was followed by work that adds locality constraints, necessary to obtain meaningful explanations on complex models such as neural networks, first in natural language processing~\cite{LaMalfa21}, and then in computer vision, where the computation of formal explanations has been defined through a series of robustness queries to a neural network verifier~\cite{Wu23, Wu24}.

\medskip

Earlier approaches to neural network verification typically relied on the use of black-box solvers, either from the model checking community~\cite{Katz17}, or from mathematical optimization~\cite{TjengXT19,lomuscio2017approach}.
All these complete approaches can be seen as specific instances of branch-and-bound~\cite{Bunel2018}.
Incomplete approaches, instead, rely on network over-approximations, and were derived in parallel by both the formal methods~\cite{Li19,Gehr18,Singh19a,Singh18,Wang18,Muller22}, and the machine learning~\cite{Gowal19,Zhang18} communities, often relying on tools from optimization~\cite{Wong2018,Dvijotham2018}.
A series of incomplete verifiers has been designed and integrated within branch-and-bound (as bounding algorithm), yielding verifiers that scale significantly better than earlier approaches based on black-box solvers~\cite{Bunel2020,Xu21,DePalma24,Henriksen21}.
The current state-of-the-art, as highlighted by yearly competitions in the area~\cite{brix2023first}, relies on fast incomplete algorithms based on linear bound propagation techniques~\cite{Wang21}, possibly with the addition of optimizations such as custom cutting planes~\cite{zhang2022general,zhou2024scalable}.
A recent work~\cite{Ugare23}, \textsc{IVAN}, proposes to reuse information from previous branch-and-bound calls when verifying slightly modified versions of the same network. Our incremental approach, presented in Section~\ref{subsec:reuse} is a simplified version of this idea yet applied to formal explainability, where the network stays the same, but the verifier is called on a series of similar input domains that differ in the perturbed features of the same input point.

\section{Conclusion and Future Work}\label{sec:concl}

While verified explanations offer a theoretically-principled approach for explainable machine learning, their applicability is severely limited by their scalability.
This work makes a step towards improving the scalability of verified explanations for neural networks classifiers by both introducing a novel algorithm for their computation, and presenting a novel and more realistic definition of verified explanations.


We proposed \textsc{FaVeX} (Section~\ref{sec:favex}), a new algorithm that substantially accelerates the computation of verified explanations. 
Compared to previous work, \textsc{FaVeX} presents three key improvements: 
\begin{enumerate}[(i)]
\item it dynamically 
combines both batch processing of robustness queries in a binary-search fashion and their sequential processing, getting the best of both worlds depending on the query (Section~\ref{subsec:favex}); 
\item it reuses information from previous executions of branch-and-bound, thus accelerating the overall verification process (Section~\ref{subsec:reuse}); 
\item it leverages a novel reduced-space counterfactual search that prevents calls to branch-and-bound by effectively re-using information from previous queries, which are used to quickly find counterfactuals (Section~\ref{subsec:attack}).
\end{enumerate}
Furthermore, we introduced \emph{verifier-optimal robust explanations} (Definition~\ref{def:v-optimal-x}), which explicitly account for the practical incompleteness of neural network verifiers on medium-sized networks by allowing features whose robustness cannot be practically ascertained to be perturbed along the invariants, and produce a hierarchical 
explanation. 
In the case of a perfect verifier, the presented definition reverts to the standard definition of an optimal robust explanation from previous work (Lemma~\ref{lem}).


Our experiments demonstrate that \textsc{FaVeX} yields speed-ups of up to an order of magnitude compared to previous algorithms when computing standard robust explanations on small fully-connected networks, and that it cuts explanation times while discovering more counterfactuals when computing verifier-optimal robust explanations on larger convolutional networks.
At the same time we show that, using \textsc{FaVeX}, verifier-optimal robust explanations can be computed for a fraction of the cost of the standard definition of robust explanations, while at the same time allowing for the production of several counterfactuals, which cannot be otherwise provided.
In particular, \textsc{FaVeX} can find hundreds of counterfactuals on state-of-the-art networks trained for provable robustness to small adversarial examples, which feature hundreds of thousands of neurons, hence producing formal explanations at scale.

\subparagraph*{Threats to Validity.}
While our experimental evaluation demonstrates improvements over prior approaches, several factors may affect the generalizability of our findings:

\begin{description}
\item[Network Architectures]
Our experiments inherit the current limitations of state-of-the-art neural network verification. We focus on piecewise-linear feedforward neural networks with ReLU activations, which are more amenable to verification in practice. Further advances in verification techniques are a prerequisite for scalably extending verified explanations to other architectures.
\item[Computational Resources]
All experiments were conducted with fixed computational resources (6 CPU cores, 20GB RAM, single GPU). While different hardware configurations or resource allocations may affect absolute runtimes, we do not expect them to significantly alter the observed trends.
\item[Domain Specificity]
Consistently with previous work, our evaluation focuses on computer vision tasks (MNIST, GTSRB, CIFAR-10). The effectiveness of \textsc{FaVeX} and verifier-optimal robust explanations on other domains remains to be investigated. 
As the proposed algorithm is not inherently tied to image data, we do not expect fundamentally different behavior on other datasets (the example in Section~\ref{sec:overview} originates from a tabular dataset). On the other hand, domain-specific challenges (for instance, the discreteness of natural language processing data) may require substantial adaptations to verification and formal explainability techniques.
\item[Verifier Choice]
Our results are specific to the \textsc{OVAL} branch-and-bound framework with $\alpha$-$\beta$-CROWN bounding. While different verifiers or verification algorithms may exhibit variations in absolute performance, we expect the qualitative trends observed for \textsc{FaVeX} to generalize across verification backends.
\end{description}

\subparagraph*{Future Work.}
While this work substantially improves on the scalability of formal explanations, their computation still requires in the order of an hour per image on networks with tens of millions of parameters and trained for verifiability. These are small compared to real-world applications in many domains.
Further progress on scalability is hence fundamental to ensure their widest applicability.
Promising directions in this sense include: exploring more complex ways to re-use information from the various queries to the verifier, systematically exploring the interaction between training methods and formal explanability, and devising ad-hoc algorithms to train networks that combine good performance and formal explanability.
Finally, while our experiments focus on vision models, testing and potentially adapting verifier-optimal robust explanations and \textsc{FaVeX} to other data domains is an exciting avenue for future work.


 \bibliography{ecoop2026}

\begin{thebibliography}{10}

\bibitem{Albarghouthi21}
Aws Albarghouthi.
\newblock Introduction to neural network verification.
\newblock {\em CoRR}, abs/2109.10317, 2021.
\newblock URL: \url{https://arxiv.org/abs/2109.10317}.

\bibitem{Anderson20}
Ross Anderson, Joey Huchette, Will Ma, Christian Tjandraatmadja, and Juan~Pablo
  Vielma.
\newblock Strong mixed-integer programming formulations for trained neural
  networks.
\newblock {\em Math. Program.}, 183(1):3--39, 2020.

\bibitem{brix2023first}
Christopher Brix, Mark~Niklas M{\"u}ller, Stanley Bak, Taylor~T Johnson, and
  Changliu Liu.
\newblock First three years of the international verification of neural
  networks competition (vnn-comp).
\newblock {\em International Journal on Software Tools for Technology
  Transfer}, 25(3):329--339, 2023.

\bibitem{Bunel2020}
Rudy Bunel, Alessandro De~Palma, Alban Desmaison, Krishnamurthy Dvijotham,
  Pushmeet Kohli, Philip~HS Torr, and M~Pawan Kumar.
\newblock Lagrangian decomposition for neural network verification.
\newblock In {\em Conference on Uncertainty in Artificial Intelligence}, 2020.

\bibitem{Bunel20}
Rudy Bunel, Jingyue Lu, Ilker Turkaslan, Philip H.~S. Torr, Pushmeet Kohli, and
  M.~Pawan Kumar.
\newblock Branch and bound for piecewise linear neural network verification.
\newblock {\em J. Mach. Learn. Res.}, 21:42:1--42:39, 2020.
\newblock URL: \url{https://jmlr.org/papers/v21/19-468.html}.

\bibitem{Bunel2018}
Rudy Bunel, Ilker Turkaslan, Philip~HS Torr, Pushmeet Kohli, and M~Pawan Kumar.
\newblock A unified view of piecewise linear neural network verification.
\newblock In {\em Neural Information Processing Systems}, 2018.

\bibitem{Chinneck91}
John~W. Chinneck and Erik~W. Dravnieks.
\newblock Locating minimal infeasible constraint sets in linear programs.
\newblock {\em {INFORMS} J. Comput.}, 3(2):157--168, 1991.

\bibitem{Darwiche20}
Adnan Darwiche and Auguste Hirth.
\newblock On the reasons behind decisions.
\newblock In Giuseppe~De Giacomo, Alejandro Catal{\'{a}}, Bistra Dilkina,
  Michela Milano, Sen{\'{e}}n Barro, Alberto Bugar{\'{\i}}n, and
  J{\'{e}}r{\^{o}}me Lang, editors, {\em {ECAI} 2020 - 24th European Conference
  on Artificial Intelligence, 29 August-8 September 2020, Santiago de
  Compostela, Spain, August 29 - September 8, 2020 - Including 10th Conference
  on Prestigious Applications of Artificial Intelligence {(PAIS} 2020)}, volume
  325 of {\em Frontiers in Artificial Intelligence and Applications}, pages
  712--720. {IOS} Press, 2020.

\bibitem{Darwiche23}
Adnan Darwiche and Auguste Hirth.
\newblock On the (complete) reasons behind decisions.
\newblock {\em J. Log. Lang. Inf.}, 32(1):63--88, 2023.

\bibitem{sparsealgosDePalma2024}
Alessandro De~Palma, Harkirat~Singh Behl, Rudy Bunel, Philip H.~S. Torr, and
  M.~Pawan Kumar.
\newblock Scaling the convex barrier with sparse dual algorithms.
\newblock {\em Journal of Machine Learning Research}, 2024.

\bibitem{DePalma21}
Alessandro De~Palma, Rudy Bunel, Alban Desmaison, Krishnamurthy Dvijotham,
  Pushmeet Kohli, Philip H.~S. Torr, and M.~Pawan Kumar.
\newblock Improved branch and bound for neural network verification via
  lagrangian decomposition.
\newblock {\em CoRR}, abs/2104.06718, 2021.
\newblock URL: \url{https://arxiv.org/abs/2104.06718}, \href
  {https://arxiv.org/abs/2104.06718} {\path{arXiv:2104.06718}}.

\bibitem{DePalma24}
Alessandro {De Palma}, Rudy Bunel, Krishnamurthy~(Dj) Dvijotham, M.~Pawan
  Kumar, Robert Stanforth, and Alessio Lomuscio.
\newblock Expressive losses for verified robustness via convex combinations.
\newblock In {\em The Twelfth International Conference on Learning
  Representations, {ICLR} 2024, Vienna, Austria, May 7-11, 2024}.
  OpenReview.net, 2024.

\bibitem{dong2018boosting}
Yinpeng Dong, Fangzhou Liao, Tianyu Pang, Hang Su, Jun Zhu, Xiaolin Hu, and
  Jianguo Li.
\newblock Boosting adversarial attacks with momentum.
\newblock In {\em Proceedings of the IEEE conference on computer vision and
  pattern recognition}, pages 9185--9193, 2018.

\bibitem{doshi2017towards}
Finale Doshi-Velez and Been Kim.
\newblock Towards a rigorous science of interpretable machine learning.
\newblock {\em arXiv preprint arXiv:1702.08608}, 2017.

\bibitem{Dvijotham2018}
Krishnamurthy Dvijotham, Robert Stanforth, Sven Gowal, Timothy Mann, and
  Pushmeet Kohli.
\newblock A dual approach to scalable verification of deep networks.
\newblock In {\em Conference on Uncertainty in Artificial Intelligence}, 2018.

\bibitem{Ehlers2017}
Ruediger Ehlers.
\newblock Formal verification of piece-wise linear feed-forward neural
  networks.
\newblock {\em Automated Technology for Verification and Analysis}, 2017.

\bibitem{Gehr18}
Timon Gehr, Matthew Mirman, Dana Drachsler{-}Cohen, Petar Tsankov, Swarat
  Chaudhuri, and Martin~T. Vechev.
\newblock {AI2:} safety and robustness certification of neural networks with
  abstract interpretation.
\newblock In {\em 2018 {IEEE} Symposium on Security and Privacy, {SP} 2018,
  Proceedings, 21-23 May 2018, San Francisco, California, {USA}}, pages 3--18.
  {IEEE} Computer Society, 2018.

\bibitem{Gowal19}
Sven Gowal, Krishnamurthy Dvijotham, Robert Stanforth, Rudy Bunel, Chongli Qin,
  Jonathan Uesato, Relja Arandjelovic, Timothy~Arthur Mann, and Pushmeet Kohli.
\newblock Scalable verified training for provably robust image classification.
\newblock In {\em 2019 {IEEE/CVF} International Conference on Computer Vision,
  {ICCV} 2019, Seoul, Korea (South), October 27 - November 2, 2019}, pages
  4841--4850. {IEEE}, 2019.

\bibitem{GuidottiMRTGP19}
Riccardo Guidotti, Anna Monreale, Salvatore Ruggieri, Franco Turini, Fosca
  Giannotti, and Dino Pedreschi.
\newblock A survey of methods for explaining black box models.
\newblock {\em {ACM} Comput. Surv.}, 51(5):93:1--93:42, 2019.
\newblock \href {https://doi.org/10.1145/3236009} {\path{doi:10.1145/3236009}}.

\bibitem{gurobi}
{Gurobi Optimization, LLC}.
\newblock {Gurobi Optimizer Reference Manual}.
\newblock URL: \url{https://www.gurobi.com}.

\bibitem{pmlr-v202-hamman23a}
Faisal Hamman, Erfaun Noorani, Saumitra Mishra, Daniele Magazzeni, and
  Sanghamitra Dutta.
\newblock Robust counterfactual explanations for neural networks with
  probabilistic guarantees.
\newblock In Andreas Krause, Emma Brunskill, Kyunghyun Cho, Barbara Engelhardt,
  Sivan Sabato, and Jonathan Scarlett, editors, {\em Proceedings of the 40th
  International Conference on Machine Learning}, volume 202 of {\em Proceedings
  of Machine Learning Research}, pages 12351--12367. PMLR, 23--29 Jul 2023.
\newblock URL: \url{https://proceedings.mlr.press/v202/hamman23a.html}.

\bibitem{Hechtlinger16}
Yotam Hechtlinger.
\newblock {Interpretation of Prediction Models Using the Input Gradient}.
\newblock {\em CoRR arXiv}, 2016.
\newblock URL: \url{http://arxiv.org/abs/1611.07634}.

\bibitem{Hemery06}
Fred Hemery, Christophe Lecoutre, Lakhdar Sais, and Fr{\'{e}}d{\'{e}}ric
  Boussemart.
\newblock Extracting mucs from constraint networks.
\newblock In Gerhard Brewka, Silvia Coradeschi, Anna Perini, and Paolo
  Traverso, editors, {\em {ECAI} 2006, 17th European Conference on Artificial
  Intelligence, August 29 - September 1, 2006, Riva del Garda, Italy, Including
  Prestigious Applications of Intelligent Systems {(PAIS} 2006), Proceedings},
  volume 141 of {\em Frontiers in Artificial Intelligence and Applications},
  pages 113--117. {IOS} Press, 2006.

\bibitem{Henriksen21}
Patrick Henriksen and Alessio Lomuscio.
\newblock {DEEPSPLIT:} an efficient splitting method for neural network
  verification via indirect effect analysis.
\newblock In Zhi{-}Hua Zhou, editor, {\em Proceedings of the Thirtieth
  International Joint Conference on Artificial Intelligence, {IJCAI} 2021,
  Virtual Event / Montreal, Canada, 19-27 August 2021}, pages 2549--2555.
  ijcai.org, 2021.

\bibitem{Huang24}
Xuanxiang Huang and Jo{\~{a}}o Marques{-}Silva.
\newblock On the failings of shapley values for explainability.
\newblock {\em Int. J. Approx. Reason.}, 171:109112, 2024.

\bibitem{Ignatiev19a}
Alexey Ignatiev, Nina Narodytska, and Jo{\~{a}}o Marques{-}Silva.
\newblock Abduction-based explanations for machine learning models.
\newblock In {\em The Thirty-Third {AAAI} Conference on Artificial
  Intelligence, {AAAI} 2019, The Thirty-First Innovative Applications of
  Artificial Intelligence Conference, {IAAI} 2019, The Ninth {AAAI} Symposium
  on Educational Advances in Artificial Intelligence, {EAAI} 2019, Honolulu,
  Hawaii, USA, January 27 - February 1, 2019}, pages 1511--1519. {AAAI} Press,
  2019.

\bibitem{Ignatiev19b}
Alexey Ignatiev, Nina Narodytska, and Jo{\~{a}}o Marques{-}Silva.
\newblock On relating explanations and adversarial examples.
\newblock In Hanna~M. Wallach, Hugo Larochelle, Alina Beygelzimer, Florence
  d'Alch{\'{e}}{-}Buc, Emily~B. Fox, and Roman Garnett, editors, {\em Advances
  in Neural Information Processing Systems 32: Annual Conference on Neural
  Information Processing Systems 2019, NeurIPS 2019, December 8-14, 2019,
  Vancouver, BC, Canada}, pages 15857--15867, 2019.

\bibitem{Izza24}
Yacine Izza, Xuanxiang Huang, Ant{\'{o}}nio Morgado, Jordi Planes, Alexey
  Ignatiev, and Jo{\~{a}}o Marques{-}Silva.
\newblock Distance-restricted explanations: Theoretical underpinnings {\&}
  efficient implementation.
\newblock In Pierre Marquis, Magdalena Ortiz, and Maurice Pagnucco, editors,
  {\em Proceedings of the 21st International Conference on Principles of
  Knowledge Representation and Reasoning, {KR} 2024, Hanoi, Vietnam. November
  2-8, 2024}, 2024.

\bibitem{pmlr-v222-jiang24a}
Junqi Jiang, Jianglin Lan, Francesco Leofante, Antonio Rago, and Francesca
  Toni.
\newblock Provably robust and plausible counterfactual explanations for neural
  networks via robust optimisation.
\newblock In Berrin Yanıkoğlu and Wray Buntine, editors, {\em Proceedings of
  the 15th Asian Conference on Machine Learning}, volume 222 of {\em
  Proceedings of Machine Learning Research}, pages 582--597. PMLR, 11--14 Nov
  2024.
\newblock URL: \url{https://proceedings.mlr.press/v222/jiang24a.html}.

\bibitem{abs-2402-01928}
Junqi Jiang, Francesco Leofante, Antonio Rago, and Francesca Toni.
\newblock Robust counterfactual explanations in machine learning: {A} survey.
\newblock {\em CoRR}, abs/2402.01928, 2024.
\newblock URL: \url{https://doi.org/10.48550/arXiv.2402.01928}, \href
  {https://doi.org/10.48550/ARXIV.2402.01928}
  {\path{doi:10.48550/ARXIV.2402.01928}}.

\bibitem{Junker04}
Ulrich Junker.
\newblock {QUICKXPLAIN:} preferred explanations and relaxations for
  over-constrained problems.
\newblock In Deborah~L. McGuinness and George Ferguson, editors, {\em
  Proceedings of the Nineteenth National Conference on Artificial Intelligence,
  Sixteenth Conference on Innovative Applications of Artificial Intelligence,
  July 25-29, 2004, San Jose, California, {USA}}, pages 167--172. {AAAI} Press
  / The {MIT} Press, 2004.

\bibitem{Katz17}
Guy Katz, Clark~W. Barrett, David~L. Dill, Kyle Julian, and Mykel~J.
  Kochenderfer.
\newblock Reluplex: An efficient {SMT} solver for verifying deep neural
  networks.
\newblock In Rupak Majumdar and Viktor Kuncak, editors, {\em Computer Aided
  Verification - 29th International Conference, {CAV} 2017, Heidelberg,
  Germany, July 24-28, 2017, Proceedings, Part {I}}, volume 10426 of {\em
  Lecture Notes in Computer Science}, pages 97--117. Springer, 2017.
\newblock \href {https://doi.org/10.1007/978-3-319-63387-9\_5}
  {\path{doi:10.1007/978-3-319-63387-9\_5}}.

\bibitem{Krizhevsky2009}
A.~Krizhevsky and G.~Hinton.
\newblock Learning multiple layers of features from tiny images.
\newblock {\em Master's thesis, Department of Computer Science, University of
  Toronto}, 2009.

\bibitem{LaMalfa21}
Emanuele {La Malfa}, Rhiannon Michelmore, Agnieszka~M. Zbrzezny, Nicola
  Paoletti, and Marta Kwiatkowska.
\newblock On guaranteed optimal robust explanations for {NLP} models.
\newblock In Zhi{-}Hua Zhou, editor, {\em Proceedings of the Thirtieth
  International Joint Conference on Artificial Intelligence, {IJCAI} 2021,
  Virtual Event / Montreal, Canada, 19-27 August 2021}, pages 2658--2665.
  ijcai.org, 2021.

\bibitem{LeCun2010}
Yann LeCun, Corinna Cortes, and CJ~Burges.
\newblock Mnist handwritten digit database.
\newblock {\em ATT Labs [Online]. Available: http://yann.lecun.com/exdb/mnist},
  2, 2010.

\bibitem{Leofante23}
Francesco Leofante and Nico Potyka.
\newblock Promoting counterfactual robustness through diversity.
\newblock In {\em Proceedings of the Thirty-Eighth AAAI Conference on
  Artificial Intelligence and Thirty-Sixth Conference on Innovative
  Applications of Artificial Intelligence and Fourteenth Symposium on
  Educational Advances in Artificial Intelligence}, AAAI'24/IAAI'24/EAAI'24.
  AAAI Press, 2024.
\newblock \href {https://doi.org/10.1609/aaai.v38i19.30127}
  {\path{doi:10.1609/aaai.v38i19.30127}}.

\bibitem{Li19}
Jianlin Li, Jiangchao Liu, Pengfei Yang, Liqian Chen, Xiaowei Huang, and Lijun
  Zhang.
\newblock Analyzing deep neural networks with symbolic propagation: Towards
  higher precision and faster verification.
\newblock In Bor{-}Yuh~Evan Chang, editor, {\em Static Analysis - 26th
  International Symposium, {SAS} 2019, Porto, Portugal, October 8-11, 2019,
  Proceedings}, volume 11822 of {\em Lecture Notes in Computer Science}, pages
  296--319. Springer, 2019.

\bibitem{lomuscio2017approach}
Alessio Lomuscio and Lalit Maganti.
\newblock An approach to reachability analysis for feed-forward relu neural
  networks.
\newblock {\em arXiv preprint arXiv:1706.07351}, 2017.

\bibitem{Lundberg17}
Scott~M. Lundberg and Su{-}In Lee.
\newblock A unified approach to interpreting model predictions.
\newblock In {\em Advances in Neural Information Processing Systems 30: Annual
  Conference on Neural Information Processing Systems 2017}, pages 4765--4774,
  2017.

\bibitem{lundberg2017unified}
Scott~M Lundberg and Su-In Lee.
\newblock A unified approach to interpreting model predictions.
\newblock {\em Neural Information Processing Systems}, 2017.

\bibitem{Madry2018}
Aleksander Madry, Aleksandar Makelov, Ludwig Schmidt, Dimitris Tsipras, and
  Adrian Vladu.
\newblock Towards deep learning models resistant to adversarial attacks.
\newblock In {\em International Conference on Learning Representations}, 2018.

\bibitem{MarquesSilva22b}
Jo{\~{a}}o Marques{-}Silva.
\newblock Logic-based explainability in machine learning.
\newblock In Leopoldo~E. Bertossi and Guohui Xiao, editors, {\em Reasoning Web.
  Causality, Explanations and Declarative Knowledge - 18th International Summer
  School 2022, Berlin, Germany, September 27-30, 2022, Tutorial Lectures},
  volume 13759 of {\em Lecture Notes in Computer Science}, pages 24--104.
  Springer, 2022.

\bibitem{MarquesSilva24}
Jo{\~{a}}o Marques{-}Silva.
\newblock Logic-based explainability: Past, present and future.
\newblock In Tiziana Margaria and Bernhard Steffen, editors, {\em Leveraging
  Applications of Formal Methods, Verification and Validation. Software
  Engineering Methodologies - 12th International Symposium, ISoLA 2024, Crete,
  Greece, October 27-31, 2024, Proceedings, Part {IV}}, volume 15222 of {\em
  Lecture Notes in Computer Science}, pages 181--204. Springer, 2024.

\bibitem{MarquesSilva2021}
Joao Marques-Silva, Thomas Gerspacher, Martin~C Cooper, Alexey Ignatiev, and
  Nina Narodytska.
\newblock Explanations for monotonic classifiers.
\newblock In Marina Meila and Tong Zhang, editors, {\em International
  Conference on Machine Learning}. PMLR, 2021.

\bibitem{MarquesSilva22a}
Jo{\~{a}}o Marques{-}Silva and Alexey Ignatiev.
\newblock Delivering trustworthy {AI} through formal {XAI}.
\newblock In {\em Thirty-Sixth {AAAI} Conference on Artificial Intelligence,
  {AAAI} 2022, Thirty-Fourth Conference on Innovative Applications of
  Artificial Intelligence, {IAAI} 2022, The Twelveth Symposium on Educational
  Advances in Artificial Intelligence, {EAAI} 2022 Virtual Event, February 22 -
  March 1, 2022}, pages 12342--12350. {AAAI} Press, 2022.

\bibitem{Muller22}
Mark~Niklas M{\"{u}}ller, Gleb Makarchuk, Gagandeep Singh, Markus
  P{\"{u}}schel, and Martin~T. Vechev.
\newblock {PRIMA:} general and precise neural network certification via
  scalable convex hull approximations.
\newblock {\em Proc. {ACM} Program. Lang.}, 6({POPL}):1--33, 2022.

\bibitem{Paszke2019}
Adam Paszke, Sam Gross, Francisco Massa, Adam Lerer, James Bradbury, Gregory
  Chanan, Trevor Killeen, Zeming Lin, Natalia Gimelshein, Luca Antiga, Alban
  Desmaison, Andreas Kopf, Edward Yang, Zachary DeVito, Martin Raison, Alykhan
  Tejani, Sasank Chilamkurthy, Benoit Steiner, Lu~Fang, Junjie Bai, and Soumith
  Chintala.
\newblock Pytorch: An imperative style, high-performance deep learning library.
\newblock In {\em Neural Information Processing Systems}, 2019.

\bibitem{ribeiro2016should}
Marco~Tulio Ribeiro, Sameer Singh, and Carlos Guestrin.
\newblock " why should i trust you?" explaining the predictions of any
  classifier.
\newblock In {\em Proceedings of the 22nd ACM SIGKDD international conference
  on knowledge discovery and data mining}, pages 1135--1144, 2016.

\bibitem{Ribeiro16}
Marco~T{\'{u}}lio Ribeiro, Sameer Singh, and Carlos Guestrin.
\newblock ``{W}hy should {I} trust you?'': {E}xplaining the predictions of any
  classifier.
\newblock In {\em Proceedings of the 22nd {ACM} {SIGKDD} International
  Conference on Knowledge Discovery and Data Mining, 2016}, pages 1135--1144.
  {ACM}, 2016.

\bibitem{ribeiro2018anchors}
Marco~T{\'{u}}lio Ribeiro, Sameer Singh, and Carlos Guestrin.
\newblock Anchors: High-precision model-agnostic explanations.
\newblock In Sheila~A. McIlraith and Kilian~Q. Weinberger, editors, {\em
  Proceedings of the Thirty-Second {AAAI} Conference on Artificial
  Intelligence, (AAAI-18), the 30th innovative Applications of Artificial
  Intelligence (IAAI-18), and the 8th {AAAI} Symposium on Educational Advances
  in Artificial Intelligence (EAAI-18), New Orleans, Louisiana, USA, February
  2-7, 2018}, pages 1527--1535. {AAAI} Press, 2018.
\newblock URL: \url{https://doi.org/10.1609/aaai.v32i1.11491}, \href
  {https://doi.org/10.1609/AAAI.V32I1.11491}
  {\path{doi:10.1609/AAAI.V32I1.11491}}.

\bibitem{selvaraju2017grad}
Ramprasaath~R Selvaraju, Michael Cogswell, Abhishek Das, Ramakrishna Vedantam,
  Devi Parikh, and Dhruv Batra.
\newblock {Grad-CAM}: Visual explanations from deep networks via gradient-based
  localization.
\newblock In {\em IEEE International Conference on Computer Vision}, 2017.

\bibitem{Shih18}
Andy Shih, Arthur Choi, and Adnan Darwiche.
\newblock A symbolic approach to explaining bayesian network classifiers.
\newblock In J{\'{e}}r{\^{o}}me Lang, editor, {\em Proceedings of the
  Twenty-Seventh International Joint Conference on Artificial Intelligence,
  {IJCAI} 2018, July 13-19, 2018, Stockholm, Sweden}, pages 5103--5111.
  ijcai.org, 2018.
\newblock URL: \url{https://doi.org/10.24963/ijcai.2018/708}, \href
  {https://doi.org/10.24963/IJCAI.2018/708}
  {\path{doi:10.24963/IJCAI.2018/708}}.

\bibitem{simonyan2013deep}
Karen Simonyan, Andrea Vedaldi, and Andrew Zisserman.
\newblock Deep inside convolutional networks: Visualising image classification
  models and saliency maps.
\newblock {\em arXiv preprint arXiv:1312.6034}, 2013.

\bibitem{Singh19b}
Gagandeep Singh, Rupanshu Ganvir, Markus P{\"{u}}schel, and Martin~T. Vechev.
\newblock Beyond the single neuron convex barrier for neural network
  certification.
\newblock In Hanna~M. Wallach, Hugo Larochelle, Alina Beygelzimer, Florence
  d'Alch{\'{e}}{-}Buc, Emily~B. Fox, and Roman Garnett, editors, {\em Advances
  in Neural Information Processing Systems 32: Annual Conference on Neural
  Information Processing Systems 2019, NeurIPS 2019, December 8-14, 2019,
  Vancouver, BC, Canada}, pages 15072--15083, 2019.

\bibitem{Singh18}
Gagandeep Singh, Timon Gehr, Matthew Mirman, Markus P{\"{u}}schel, and
  Martin~T. Vechev.
\newblock Fast and effective robustness certification.
\newblock In Samy Bengio, Hanna~M. Wallach, Hugo Larochelle, Kristen Grauman,
  Nicol{\`{o}} Cesa{-}Bianchi, and Roman Garnett, editors, {\em Advances in
  Neural Information Processing Systems 31: Annual Conference on Neural
  Information Processing Systems 2018, NeurIPS 2018, December 3-8, 2018,
  Montr{\'{e}}al, Canada}, pages 10825--10836, 2018.

\bibitem{Singh19a}
Gagandeep Singh, Timon Gehr, Markus P{\"{u}}schel, and Martin~T. Vechev.
\newblock An abstract domain for certifying neural networks.
\newblock {\em Proc. {ACM} Program. Lang.}, 3({POPL}):41:1--41:30, 2019.

\bibitem{Stallkamp2012}
Johannes Stallkamp, Marc Schlipsing, Jan Salmen, and Christian Igel.
\newblock Man vs. computer: Benchmarking machine learning algorithms for
  traffic sign recognition.
\newblock {\em Neural networks}, 32:323--332, 2012.

\bibitem{TjengXT19}
Vincent Tjeng, Kai~Yuanqing Xiao, and Russ Tedrake.
\newblock Evaluating robustness of neural networks with mixed integer
  programming.
\newblock In {\em 7th International Conference on Learning Representations,
  {ICLR} 2019, New Orleans, LA, USA, May 6-9, 2019}. OpenReview.net, 2019.
\newblock URL: \url{https://openreview.net/forum?id=HyGIdiRqtm}.

\bibitem{Ugare23}
Shubham Ugare, Debangshu Banerjee, Sasa Misailovic, and Gagandeep Singh.
\newblock Incremental verification of neural networks.
\newblock {\em Proc. {ACM} Program. Lang.}, 7({PLDI}):1920--1945, 2023.
\newblock \href {https://doi.org/10.1145/3591299} {\path{doi:10.1145/3591299}}.

\bibitem{Urban21}
Caterina Urban and Antoine Min{\'{e}}.
\newblock A review of formal methods applied to machine learning.
\newblock {\em CoRR}, abs/2104.02466, 2021.
\newblock URL: \url{https://arxiv.org/abs/2104.02466}.

\bibitem{Wang18}
Shiqi Wang, Kexin Pei, Justin Whitehouse, Junfeng Yang, and Suman Jana.
\newblock Efficient formal safety analysis of neural networks.
\newblock In Samy Bengio, Hanna~M. Wallach, Hugo Larochelle, Kristen Grauman,
  Nicol{\`{o}} Cesa{-}Bianchi, and Roman Garnett, editors, {\em Advances in
  Neural Information Processing Systems 31: Annual Conference on Neural
  Information Processing Systems 2018, NeurIPS 2018, December 3-8, 2018,
  Montr{\'{e}}al, Canada}, pages 6369--6379, 2018.

\bibitem{Wang21}
Shiqi Wang, Huan Zhang, Kaidi Xu, Xue Lin, Suman Jana, Cho{-}Jui Hsieh, and
  J.~Zico Kolter.
\newblock Beta-crown: Efficient bound propagation with per-neuron split
  constraints for neural network robustness verification.
\newblock In Marc'Aurelio Ranzato, Alina Beygelzimer, Yann~N. Dauphin, Percy
  Liang, and Jennifer~Wortman Vaughan, editors, {\em Advances in Neural
  Information Processing Systems 34: Annual Conference on Neural Information
  Processing Systems 2021, NeurIPS 2021, December 6-14, 2021, virtual}, pages
  29909--29921, 2021.

\bibitem{Wong2018}
Eric Wong and Zico Kolter.
\newblock Provable defenses against adversarial examples via the convex outer
  adversarial polytope.
\newblock In {\em International Conference on Machine Learning}, 2018.

\bibitem{Wu24}
Min Wu, Xiaofu Li, Haoze Wu, and Clark~W. Barrett.
\newblock Better verified explanations with applications to incorrectness and
  out-of-distribution detection.
\newblock {\em CoRR}, abs/2409.03060, 2024.
\newblock \href {https://arxiv.org/abs/2409.03060} {\path{arXiv:2409.03060}},
  \href {https://doi.org/10.48550/ARXIV.2409.03060}
  {\path{doi:10.48550/ARXIV.2409.03060}}.

\bibitem{Wu23}
Min Wu, Haoze Wu, and Clark~W. Barrett.
\newblock Verix: Towards verified explainability of deep neural networks.
\newblock In Alice Oh, Tristan Naumann, Amir Globerson, Kate Saenko, Moritz
  Hardt, and Sergey Levine, editors, {\em Advances in Neural Information
  Processing Systems 36: Annual Conference on Neural Information Processing
  Systems 2023, NeurIPS 2023, New Orleans, LA, USA, December 10 - 16, 2023},
  2023.

\bibitem{Xu21}
Kaidi Xu, Huan Zhang, Shiqi Wang, Yihan Wang, Suman Jana, Xue Lin, and
  Cho{-}Jui Hsieh.
\newblock Fast and complete: Enabling complete neural network verification with
  rapid and massively parallel incomplete verifiers.
\newblock In {\em 9th International Conference on Learning Representations,
  {ICLR} 2021, Virtual Event, Austria, May 3-7, 2021}. OpenReview.net, 2021.

\bibitem{zhang2022general}
Huan Zhang, Shiqi Wang, Kaidi Xu, Linyi Li, Bo~Li, Suman Jana, Cho-Jui Hsieh,
  and J~Zico Kolter.
\newblock General cutting planes for bound-propagation-based neural network
  verification.
\newblock In {\em Neural Information Processing Systems}, volume~35, 2022.

\bibitem{Zhang18}
Huan Zhang, Tsui{-}Wei Weng, Pin{-}Yu Chen, Cho{-}Jui Hsieh, and Luca Daniel.
\newblock Efficient neural network robustness certification with general
  activation functions.
\newblock In Samy Bengio, Hanna~M. Wallach, Hugo Larochelle, Kristen Grauman,
  Nicol{\`{o}} Cesa{-}Bianchi, and Roman Garnett, editors, {\em Advances in
  Neural Information Processing Systems 31: Annual Conference on Neural
  Information Processing Systems 2018, NeurIPS 2018, December 3-8, 2018,
  Montr{\'{e}}al, Canada}, pages 4944--4953, 2018.

\bibitem{Zhou24}
Duo Zhou, Christopher Brix, Grani~A. Hanasusanto, and Huan Zhang.
\newblock Scalable neural network verification with branch-and-bound inferred
  cutting planes.
\newblock In Amir Globersons, Lester Mackey, Danielle Belgrave, Angela Fan,
  Ulrich Paquet, Jakub~M. Tomczak, and Cheng Zhang, editors, {\em Advances in
  Neural Information Processing Systems 38: Annual Conference on Neural
  Information Processing Systems 2024, NeurIPS 2024, Vancouver, BC, Canada,
  December 10 - 15, 2024}, 2024.

\bibitem{zhou2024scalable}
Duo Zhou, Christopher Brix, Grani~A Hanasusanto, and Huan Zhang.
\newblock Scalable neural network verification with branch-and-bound inferred
  cutting planes.
\newblock In {\em Neural Information Processing Systems}, pages 29324--29353,
  2024.

\end{thebibliography}

\end{document}